\documentclass{article}

\usepackage{arxiv}

\usepackage[utf8]{inputenc} 
\usepackage[T1]{fontenc}    
\usepackage{hyperref}       
\usepackage{url}            
\usepackage{nicefrac}       
\usepackage{microtype}      
\usepackage{lipsum}
\usepackage{graphicx}
\usepackage{color}
\usepackage{amsmath,amsfonts,amssymb}
\usepackage{mathtools}
\usepackage{enumerate, enumitem}
\usepackage{verbatim}
\usepackage{commath}
\usepackage{accents}
\usepackage{dashrule}
\usepackage{booktabs, tabularx, ragged2e, subcaption}
\usepackage[round,sort]{natbib}
\usepackage{cleveref}
\usepackage{algorithm}
\usepackage{algpseudocode}

\author{Tanya Veeravalli$^{1,2}$\thanks{Equal contribution} \quad David M. Bossens$^{1,2*}$ \quad Atsushi Nitanda$^{1,2,3}$ \\
\normalsize{\textit{$^1$Centre for Frontier AI Research, Agency for Science, Technology and Research (A*STAR)}} \\
\normalsize{\textit{$^2$Institute of High Performance Computing, Agency for Science, Technology and Research (A*STAR)}} \\
\normalsize{\textit{$^3$College of Computing and Data Science, Nanyang Technological University}} \\
\small{Email: \{Veeravalli\_Tanya\,,\,\,David\_Bossens\,,\,\,Atsushi\_Nitanda\}@a-star.edu.sg} }

\def\ddefloop#1{\ifx\ddefloop#1\else\ddef{#1}\expandafter\ddefloop\fi}

\def\ddef#1{\expandafter\def\csname c#1\endcsname{\ensuremath{\mathcal{#1}}}}
\ddefloop ABCDEFGHIJKLMNOPQRSTUVWXYZ\ddefloop
 
\def\ddef#1{\expandafter\def\csname s#1\endcsname{\ensuremath{\mathsf{#1}}}}
\ddefloop ABCDEFGHIJKLMNOPQRSTUVWXYZ\ddefloop

\def\:{\coloneqq}
\def\Reals{{\mathbb R}}

\def\Ex{{\mathbf E}} 
\def\Pr{{\mathbf P}} 
\def\bP{{\mathbb P}} 
\def\tr{{\mathsf T}} 
\def\eps{\varepsilon}
\DeclareMathOperator*{\argmax}{arg\,max}
\DeclareMathOperator*{\argmin}{arg\,min}
\DeclareMathOperator{\Tr}{Tr}

\makeatletter
\newsavebox{\@brx}
\newcommand{\pmt}[1][]{\tiny{ $\pm #1$} }
\newcommand{\llangle}[1][]{\savebox{\@brx}{\(\m@th{#1\langle}\)}%
  \mathopen{\copy\@brx\kern-0.5\wd\@brx\usebox{\@brx}}}
\newcommand{\rrangle}[1][]{\savebox{\@brx}{\(\m@th{#1\rangle}\)}%
  \mathclose{\copy\@brx\kern-0.5\wd\@brx\usebox{\@brx}}}
\makeatother
\newenvironment{proof}{\paragraph*{Proof:}}{\hfill$\square$}
\newtheorem{theorem}{Theorem}
\newtheorem{definition}{Definition}
\newtheorem{assumption}{Assumption}
\crefname{assumption}{assumption}{Assumptions}
\creflabelformat{assumption}{#2#1#3}
\newtheorem{example}{Example}

\newtheorem{lemma}{Lemma}
\newtheorem{corollary}{Corollary}
\newtheorem{remark}{Remark}

\newcolumntype{Y}{>{\RaggedRight\arraybackslash}X}

\begin{document}

\title{Policy Gradient for Continuous-Time Robust Markov Decision Processes}
\date{}
\maketitle

\begin{abstract}
The framework of robust Markov decision processes (RMDPs) allows the design of reinforcement learning agents that satisfy performance guarantees under worst-case transition dynamics. Traditional RMDPs consider discrete-time dynamics and recently, sample-efficient policy gradient algorithms have been considered in this context. This paper investigates policy gradient algorithms within a continuous-time RMDP framework. Policy gradients and adversarial gradients are derived using pathwise and adjoint-based formulas for stochastic and ordinary differential equations. We propose double-loop optimisers to obtain linear convergence in the oracle-based setting and  $\tilde{\cO}(\frac{1}{\eps})$ and $\tilde{\cO}(\frac{1}{\eps^2})$ sample complexities in the sample-based setting under the last-iterate and average regret, respectively, under assumptions of gradient dominance and convergent transition kernels. Our analysis also derives novel tools for the framework of undiscounted total cost MDPs. Additionally, we propose mean-field optimisers as distributional optimisers with an $\tilde{\cO}(\frac{1}{K})$ oracle-based convergence rate and an $\tilde{\cO}(\frac{N^2}{\eps})$ sample complexity under $N$-particle approximation, conditional on an additional contractivity/dissipativity assumption. The effectiveness of continuous-time policy gradient algorithms is confirmed for both optimisers on continuous-time RMDPs with neural ordinary differential equation dynamics.
\end{abstract}
\section{Introduction}
Reinforcement learning agents require data trajectories of states, actions, and rewards to optimise the long-term cumulative reward. Such data trajectories are often obtained from poorly specified transition models, and therefore when an agent transfers to the true transition model, its performance will be far from optimal. The framework of robust Markov decision processes  (RMDPs;  \cite{Wiesemann2013,Nilim2005}) provides a natural approach to deal with such problems, with relevance to sim2real and robust control applications.

RMDPs optimise the policy of an agent under worst-case transition dynamics, based on some pre-defined uncertainty set. While RMDPs are often studied in mostly theoretical settings, they have also found their way to policy gradient algorithms, which are highly scalable and performant. In particular, double-loop policy gradient algorithms \citep{Wang2024,Wang2025,bossens2025mirror,Li2026} perform repeated gradient ascent in the space of transition kernels to guarantee that the policy at each time step is evaluated on its worst-case transition kernel. These works have found strong theoretical and empirical support.

Many problems in various application domains, e.g. in control, physics, chemistry, video generation, etc., are modeled within continuous-time formalisms such as partial differential equations (PDEs) and continuous-time Markov chains, which contrasts with traditional MDP based frameworks. For instance, recent works investigate neural ordinary differential equations (ODEs) \citep{neuralODE_chen2018}, neural stochastic differential equations (SDEs) \citep{Tzen2019}, diffusion models \citep{song2021scorebased}, and gradient flow \cite{Santambrogio2017}. In RL, while these formalisms have been used to a varying degree for varying purposes (exploration, policy definition, etc.), their use for policy optimisation algorithms in continuous-time systems is limited to relatively few works \mbox{\citep{Zhan2023,zhao2023policy,Ainsworth2021,Munos2006}}. More so, the design of continuous-time policy gradient algorithms for RMDPs remains an unexplored challenge. 

With the above in mind, our paper makes the following contributions:
\begin{itemize}
    \item We introduce a novel framework that we will call \textit{Continuous-time robust Markov decision processes}, which considers decision-making with worst-case robustness by extending the RMDP framework to the continuous-time setting.
    \item Pathwise and adjoint-based gradient computations for the policy and the adversary under deterministic and stochastic policies in a continuous-time RMDP problems, including dynamics based on ODE and SDE systems.
    \item A last-iterate convergence proof for two classes of policy-gradient algorithms to solve the continuous-time RMDP problem: i) a double-loop algorithm for gradient-based policy optimisation with linear convergence in the oracle-based setting and a sample complexity of $\tilde{\cO}(\frac{1}{\eps})$  and $\tilde{\cO}(\frac{1}{\eps})$ in the last-iterate and average regret; and ii) a single-loop mean-field optimisation algorithm for distributional gradient-based policy optimisation with oracle-based rate of $\tilde{\cO}(\frac{1}{K})$ and sample-based rate of $\tilde{\cO}(\frac{N^2}{\eps})$. The double-loop analysis in particular also introduces new tools for analysing rectangularity under indirect parametrisation and a performance difference lemma for transition kernels within the total cost RMDP framework.
    \item Empirical experiments comparing the double-loop and mean-field optimisers with different instantiations of the policy and adversary optimisers. Experiments on LQR with neural ODE dynamics and neural network parametrised uncertainty sets show that robust discrete-time policy gradient algorithms have poor robust optimisation performance in continuous-time RMDPs while continuous-time policy gradient algorithms have a clear benefit in the setting.
\end{itemize}

\section{Related work}
\subsection{Continuous-time formulations in RL}
First we consider previous works that consider continuous-time formalisms in the non-robust RL and control. From a technical perspective we distinguish between HJB-based algorithms, pathwise algorithms, and adjoint-based algorithms for policy optimisation as the most directly related works. 

HJB-based value-based algorithms approximate the Q-value from which the optimal control can be determined \cite{Kim2021,Jia2023}. HJB-based actor-critic algorithms use a similar Q-value approximation to approximate the policy gradient \cite{Cohen2025}. Our paper focuses on pure policy gradient algorithms without any function approximation, instead making use of an adjoint state and applying Pontryagin's maximum principle \citep{Pontryagin1962}.

One of our contributions is the formulation of pathwise policy gradients for continuous-time RMDPs. \cite{Munos2006} previously derived the policy gradient for continuous-time (non-robust) undiscounted MDPs with deterministic dynamics based on ODEs. The work showed that likelihood ratio methods result in exponential blow-up of the variance and proposed a pathwise gradient that converges almost surely to the true gradient when the time step approaches zero. Our work takes place in the same setting, and in addition to extending the pathwise policy gradient estimator to time intervals, we also derive a pathwise adversarial gradient for continuous-time RMDPs.

Another contribution is the formulation of adjoint-based policy gradients for continuous-time RMDPs. Along the same lines, \cite{Zhan2023} leverage Pontryagin's maximum principle to connect continuous-time policy optimisation to minimizing the integral of the associated Hamiltonian of a neural ODE over a fixed time horizon. They use first-order optimisation methods and employ a state augmentation technique to convert the total cost into a state component, after which an adjoint method is employed to calculate the policy gradient. Within the discounted MDP setting, \cite{Jia2022} explores policy gradient and actor-critic in continuous time discounted MDPs, where the state dynamics are governed by an SDE. Similarly, \cite{zhao2023policy} proposes continuous-time policy optimisation with discounted MDPs, showing performance-difference and local-approximation formulas, and extending these from basic policy gradient to PPO and TRPO. These approaches do not present convergence rates for the optimisers and do not deal with worst-case robustness. Most closely related to our work is \cite{Ainsworth2021}, which proposes a continuous-time policy gradient with the same total cost objective albeit without any proof. Our adjoint-based policy gradient is similar except to this work. In addition to providing its theoretical proof, we also derive an adversary policy gradient for robust MDPs. Two further differences are that we also include time-dependencies in the policy and the dynamics and that our gradient is not limited to the sample-and-hold model.

Other work employing the maximum principle in continuous-time includes \cite{leeftink2025optimalcontrolprobabilisticdynamics}, where the authors assume imperfect knowledge of the underlying systems' dynamics and proceed to probabilistically minimize the mean Hamiltonian with respect to epistemic uncertainty. However, they do not frame the problem in a robust optimisation framework. The authors also propose a numerical method scalable to large-scale probabilistic dynamical models including ensemble neural ODEs.  

Some works that tackle policy gradients in continuous-time for SDEs and stochastic control problems include \cite{zhou2024stabilizingpolicygradientsstochastic, zhou2025PolicyGradSOC}. The former optimises SDEs (viewed as continuous-time generative models in this context) using policy gradient and a perturbation process argument. The latter analyses the gradient flow of the stochastic control as the continuous-time limit of policy gradient. Lastly, on a more general note, the foundational work \cite{Sirignano2017_sgdCtsTime} proposes a computationally efficient stochastic gradient descent in continuous-time which has found its way to many of the above related works. 

 Beyond policy optimisation in environments with continuous-time dynamics, continuous-time formalisms also find their way to alternative directions in RL. For instance, continuous-time formalisms have found their way in the use of diffusion policies (e.g. \cite{Ren2025}) which can be used for improved stability and exploration. Another example is the use of gradient flow for RL, where they apply continuous-time methods for parameter updates \citep{kerimkulov2025fisherraogradientflowentropyregularised,zorba2025,pfau2025,muller2025fisherraogradientflowslinear}. While gradient flow  techniques are orthogonal to the continuous-time RMDP framework, we also demonstrate how mean-field gradient flow \citep{Liu2025} can be applied to distributional optimisation within the framework. 

Issues arising from discretisation and analogs between discrete-time and continuous-time formalisms are of wide interest. Formalisms such as continuous-time MDPs and semi-MDPs focus on MDPs with discrete state-action spaces where each transition takes a random amount of time \citep{Bradtke}. A discretised approach to the policy gradient was previously explored for SDEs \citep{zhou2024stabilizingpolicygradientsstochastic}. The relation between the Q-function in discrete-time RL and the Hamiltonian has been discussed in earlier works (e.g. \cite{mehta2009qlearningPMP}). Our empirical study investigates the time discretisation sensitivity and the comparison of discrete-time vs continuous-time algorithms in the context of robust PDE dynamics; our regret analyses are based on the discrete-time framework of undiscounted total cost MDPs; and we show that the discrete-time stochastic and deterministic policy gradient \cite{Sutton2012,Silverb} converge to our policy gradient formulations as $\Delta t \to 0$ under a sample-and-hold model. Naturally, in practice, a much more coarse time step should be chosen, which is where the benefit of continuous-time algorithms comes in. Our study shows that robust discrete-time policy gradient algorithms have poor robust optimisation performance in continuous-time RMDPs and we empirically demonstrate the benefit of continuous-time policy gradient algorithms.

\subsection{Robust reinforcement learning and control}
Now we discuss related works in robust reinforcement learning, which include policy gradient algorithms for RMDPs and PDE-based techniques for robust control and RL.

Our first proposed optimiser is an instance of the class of \textit{double-loop policy gradient algorithms} in RMDPs \citep{Wang2024,Wang2025,bossens2025mirror,Li2026}, which solve the inner problem $\epsilon$-optimally at each iteration. In the oracle-based last-iterate regret, we obtain a linear convergence rate, which matches recent results \citep{Wang2024,bossens2025mirror}. In the sample-based last-iterate regret, we obtain an $\tilde{\cO}(\frac{1}{\epsilon})$ sample complexity. Our results for the average-regret based sample complexity obtain an improved bound $\tilde{\cO}(\frac{1}{\epsilon^2})$ over the state-of-the-art in \cite{bossens2025mirror}. In addition to our work considering continuous-time dynamics, there are a few further differences. Our formulation focuses on the undiscounted total cost MDP setting, which is a discretisation of our continuous-time RMDP framework which has equivalent policy gradients as $\Delta t \to 0$. We also do not consider constraints (as in robust constrained MDPs; \cite{bossens2025mirror}). Our regret analysis also holds under a general policy parameterization whereas these previous works considered direct and softmax parameterizations only. Our work requires separate PL conditions for the policy and the adversary but does not require the PL coefficients to be positive at every iteration. As in previous works, the convergence bounds come with further assumptions (e.g. regularised objectives, uniqueness of the worst transition kernel) \citep{kitamura2026}. Our results are based on separate gradient dominance conditions for the policy and adversary objectives  under specific smoothness, differentiability, and Lipschitz assumptions, and a restriction that the transition kernel's objective converges. Together with the upper bound of the iterates' cost deviations to the robust cost, as guaranteed by the double-loop policy algorithm, this leads to a robust cost upper bound. For the adversary, we proposed a modified PL condition suitable for projected gradient descent. Our work is the first to propose a policy gradient solution for continuous-time RMDPs and we consider a range of policy gradients, including stochastic adjoint-based, deterministic adjoint-based, and pathwise. Rectangularity may be required for the PL condition so we also derive a rectangularity result under an indirect parameterization to support the analysis. It may under some conditions be necessary apply a correction for non-rectangular sets as in  \cite{Li2026}, who show that under direct parameterization, a Frank-Wolfe algorithm can obtain convergence guarantees with an additional factor in the convergence rate. 

Our regret analysis in terms of undiscounted total cost RMDPs is novel, to our knowledge, with only one recent work appearing in the robust MDP literature. This work \citep{Su2025} explored entropic risk and entropic value-at-risk measures in this context, and concludes that the total reward framework may be preferable to discounted MDPs in a broad range of domains. Our work connects a continuous-time formalism with the discrete-time total cost RMDP. The framework of average-reward MDPs is explored more widely in the robust MDP literature \citep{Sun2024a,Wang2025} but is conceptually distinct in that it emphasises long-term asymptotic behaviour rather than transient behaviour. Nevertheless, to demonstrate the competitiveness of our algorithm, a linear oracle-based convergence was found for rectangular RMDPs in \cite{Sun2024a} and an $\cO(\frac{1}{\epsilon^4})$ oracle-based iteration complexity was found for non-rectangular RMDPs in \cite{Wang2025}. 

Our second proposed optimiser is a distributional optimiser, which optimises a problem an objective formulated over a distribution of parameters. Specifically, we directly use the analysis of the Mean-field Langevin Stochastic Descent-Ascent (MFL-SDA) algorithm from the mean-field optimisation literature \citep{Liu2025}. In the context of GAN training, previous work has highlighted distributional optimisation as a favourable global optimiser with theoretical guarantees \citep{hsieh2019finding}. Our work extends the approach to RL objectives, where distributional optimisation is not commonly explored; note that that this is distinct from distributional RL (e.g. \cite{Bellemare2017}), where the return distribution is modeled to improve value estimates, distributed RL (e.g. \cite{Espeholt2018}), where distinct actors are trained in parallel for accelerated training, or uncertain MDPs (e.g. \cite{Ahmed2013}), where the objective is to optimise the policy over a fixed distribution over transition kernels. Although here are a few prior works which update the adversary based on single-particle Langevin dynamics in the context of non-rectangular uncertainty sets  \citep{Wang2025,Li2026}, to the best of our knowledge, the use of distributional optimisation, let alone mean-field optimisation, has not yet been explored in the context of RMDPs. Using the MFL-SDA optimiser as opposed to the double loop optimiser allows us to perform distributional optimisation and gives us a particle-based practical algorithm.

While the extension of RMDPs to the continuous-time setting has not been considered, works on robust control in continuous time have proposed solutions related to ours. For instance, Linear Quadratic Regulator (LQR) problems are frequently studied by control theorists, often in a more specialized set of control strategies and algorithms \citep{VINODHKUMAR2013169, kalman1960new}. Similarly, recent work has also explored worst-case robustness in Linear Quadratic Gaussian problems \citep{Taskesen2023}. By contrast, our approach tackles a more general problem structure including neural network parameterized dynamics, of which LQR type problems are one instance (as shown in the experiments). Within the context of a robotics application, previous work has considered uncertainty-aware parameterized SDEs to fit observed dynamics, which allowed a soft-robust RL using traditional PPO on SDE samples \citep{Djeumou2023}. Our paper investigates continuous-time policy gradient algorithms and considers hard-robustness, obtaining the worst-case environment from the uncertainty set. Previous work has also considered risk-averse control with the maximum principle but this has focused on alternative risk measures and outside the context of policy gradient algorithms \citep{Bonalli2023}.

\section{Preliminaries}
Our goal is to use policy gradient in the context of RMDPs with continuous-time dynamics. We briefly describe the setting, the main objective, and the key concepts and general assumptions underlying our algorithmic solutions and proof strategies.
\subsection{General notations}
A few general notations we use throughout the paper are: $\bP(x)$ to denote the probability of an event $x$; $\bP[\cY]$ to denote the space of probability distributions over a set $\cY$; $\partial_y$ to denote the partial derivative w.r.t. a uni-dimensional variable $y$ and $\nabla_y$ to denote the gradient w.r.t. a multi-dimensional variable $y$; and $\Ex[x]$ to denote the expectation of a random variable $x$. Further notations are specific to the RL setting explained below or are introduced in the rest of the main body.
\subsection{Continuous-time setting}
The setting includes a continuous state space $\cX \subset \Reals^{d_x}$ and a continuous control space $\cU \subset \Reals^{d_u}$ where the dynamics evolve according to a controlled differential equation of the form
\begin{equation}
\label{eq: dynamics}
\frac{\dif x_t}{\dif t} = f_{\xi}(t, x_t, u_t)  \,,
\end{equation}
where $\xi \in \Reals^{d_\xi}$ parametrizes the dynamics, which will be considered as the inner problem of the robust optimisation.

The agent at time $t \geq 0$ observes the dynamics as $x_t$ and takes actions as the control $u_t \sim \pi_\theta(t,x_t)$, where $\pi_{\theta}$ is a stochastic policy in $\Pi \: \{\pi_\theta: [0,T] \times \Reals^d \rightarrow  \bP[\cU] \}$. When the policy is deterministic, we denote the policy as \(\mu_\theta(t, x)\) for a state \(x \in \cX\). Given a particular dynamics model $f_{\xi}$, the agent seeks to optimise the expected cumulative cost of the trajectory according to 
\begin{equation}
\label{eq: policy objective}
\min_{\theta} J(\theta,\xi) \: \Ex_{\pi_\theta}\left[\int_{0}^{T} r(\tau,x_{\tau},u_{\tau})  \dif \tau  + R(x_T) \Big \vert d_0, f_{\xi}\right] \,,
\end{equation}
where $d_0$ is the starting distribution, $r: [0,T] \times \cX \times \cU \to \Reals^{+}$ is the running cost, and $R: \cX \to \Reals^+$ is the terminal cost. For simplicity of exposition, we further assume that the cost function is non-negative.
 
To facilitate the computation of the policy gradient and the expected cumulative cost, we make use of two related value functions. First, the action-value, which is the cost-to-go from a particular time and state-action pair, is formulated for continuous-time algorithms as
\begin{equation}
\label{eq: state-action value}
Q(t,x,u) = \Ex_{\pi_\theta}\left[\int_{t}^{T} r(\tau,x_{\tau},u_{\tau})  \dif \tau + R(x_T) \Big\vert x_t=x, u_t=u, f_{\xi}\right] \,,
\end{equation}
where $\pi_{\theta}$ and $f_{\xi}$ are the current policy and transition dynamics.
Similarly, the value can be formulated for a particular state and time as 
\begin{equation}
\label{eq: state value}
V(t,x) = \min_{u}\Ex \left[Q(t,x,u) \Big\vert f_{\xi}, \pi_\theta \right] \,.
\end{equation}
While the above is the value of the current policy on the optimal action, the value of the current policy with the next action selected by the policy is denoted as
\begin{equation}
\label{eq: V-pi}
V^{\pi_\theta}(t,x) = \Ex_{u \sim \pi_\theta(\cdot|t,x)} Q(t,x,u) \,.
\end{equation}

\subsection{Continuous-time RMDP formalism}
Designing a formalism for continuous-time RMDP (CT-RMDPs), we are interested in obtaining the worst-case transition dynamics chosen by nature. The general objective matches that of the traditional RMDP,
\begin{equation}
\label{eq: RMDP}
\min_{\theta} \left\{ \Phi(\theta) \:\sup_{\xi \in \Xi}  J(\theta,\xi) \right\}  \,,
\end{equation}
where obtaining the worst-case dynamics is also referred to as the ``inner problem'' of the RMDP. 

In the CT-RMDP setting, we consider underlying parametrised dynamics of the type \eqref{eq: dynamics} with
\begin{equation}
\label{eq: uncertainty set parametrisation}
f_{\xi}(t,x,u)  =   \bar{f}(t,x,u) + g_{\xi}(t,x,u) \,,
\end{equation}
where $\bar{f}$ is the nominal dynamics, $g_{\xi}$ is a Lipschitz-continuous function, and the parameters \(\xi\) are restricted in the compact space $\Xi$.

\subsection{Maximum principle}
The maximum principle (MP) is a general framework using the calculus of variations that presents optimality conditions of a continuous, controlled trajectory subject to a cost functional \(J\) \citep{Pontryagin1962}. The MP is employed for systems with an open-loop control (i.e. \(u(t)\) rather than \(u(t, x_t)\)). Alternatively, the Hamilton-Jacobi-Bellman partial differential equation (HJB PDE) provides optimality conditions for a \textit{closed-loop} policy/control using a dynamic programming formalism \citep{bellman1957dynamic}. These two approaches are related, as we will see in the following.

Consider the controlled ODE dynamics \eqref{eq: dynamics}, subject to the cost functional \eqref{eq: policy objective}. We construct the corresponding Hamiltonian:
\begin{equation}\label{eq:hamiltonian_ode}
    H(t, x_t, u_t, p_t) \: r(t, x_t, u_t) + p_t^\tr f_\xi(t, x_t, u_t)
\end{equation}
where \(p_t \in \Reals^{d_x}\) is the time-varying costate/adjoint term. In essence, the Hamiltonian is a Lagrangian that considers the ODE dynamics as a constraint while the $p_t$ are interpreted as the Lagrange multipliers. The MP states that the optimal state \(x_t^*\), optimal control \(u_t^*\), and corresponding costate \(p_t^*\) must minimize the Hamiltonian: \(H(t, x_t^*, u_t^*, p_t^*) \leq  H(t, x_t, u_t, p_t)\) for all \(t, x_t, u_t, p_t\).
The trajectory of the costate \(p_t\) is determined by the partial derivative of the Hamiltonian w.r.t.~ the state variables and is solved backwards in time:
\begin{equation}
    - \dot p_t = \partial_x H(t, x_t, u_t, p_t), \quad p_T = \nabla_x R(x_T).
\end{equation}
Note that the state evolution is prescribed by
\begin{equation}
    \dot x_t = \partial_p H(t, x_t, u_t, p_t), \quad x_0 = x \sim d_0
\end{equation}
which reduces to \eqref{eq: dynamics}.
On the other hand, if we have a value function \(V(t, x_t)\) that denotes the minimum cost-to-go from time $t$, the HJB equation describes its evolution:
\begin{equation}
    - \partial_t V(t,x_t) = \min_{u} \{ r(t, x_t, u_t) + \nabla_x V(t,x_t) ^\tr f_\xi(t,x_t, u_t)\}.
\end{equation}
The costate \(p_t\) is physically and mathematically identified with the spatial gradient of the value function: \(p_t = \nabla_x V(t, x_t)\). Along an optimal characteristic, \(p^*_t = \nabla_x V^*(t, X^*_t)\), which would be the quantity used in HJB-based optimal control. Additionally, in the RL setting, it is known that the \(Q\)-function \eqref{eq: state value} is closely related to the Hamiltonian \eqref{eq:hamiltonian_ode}, in that the optimal Hamiltonian is equivalent to the optimal \(Q\) function obtained from Q-learning \citep{mehta2009qlearningPMP}.
The MP can be used to learn global feedback policies \citep{Ainsworth2021}, as we will see, and is generally a far more tractable way to solve for the optimal control than solving the HJB PDE (which suffers in high dimensions).

\subsection{General assumptions}\label{sec:assumptions}
The following general assumptions will be made throughout our analyses. These are standard and are physically realisable through bounded features/weights (when applicable). We assume:
\begin{enumerate}
\item We require \(f_\xi(t,x,u)\), \(r(t, x,u)\), and \(R(x)\) to be smooth in their arguments in order to apply the pathwise derivation for the pathwise gradient formula below.
    \item The functions $f_{\xi}, r, R, g_{\xi}$ are Lipschitz (in all their arguments) on a compact (invariant) set.  \label{ga: first}
    \item The derivatives are bounded $\norm{r_u} \leq C_{ru}$,\; $\norm{r_x} \leq C_{rx}$, \; $ \norm{R_x} \leq C_{Rx}$.
    \item $\norm{\nabla_u f} \leq C_{fu}, \norm{f_x} \leq C_{fx}$.
    \item $\norm{\nabla_x(g_\xi(t,x,u))} \leq C_{g\xi}$.
    \item $\norm{\nabla_\theta \pi_\theta(t,x_t)} \leq C_{\pi \theta}$ for stochastic policy $\pi_\theta$ and $\norm{\nabla_\theta \mu_\theta(t,x_t)} \leq C_{\mu \theta}$ for deterministic policy $\mu_\theta$
    \item $\nabla_\theta \pi_\theta$ is $L_{\pi \theta}$-Lipschitz in $\theta$; analogously, $\nabla_\theta \mu_\theta$ is  $L_{\mu \theta}$-Lipschitz in $\theta$.
    \item $\pi_\theta$ is $L_{\pi}$-,$L_{\pi x}$- and $L_{\pi t}$-Lipschitz in $\theta$, $x$, and $t$; analogously, $\mu_{\theta}$ is $L_{\mu}$-, $L_{\mu x}$-, and $L_{\mu t}$-Lipschitz continuous. It follows that, for instance, $\norm{\pi_{\theta'}(t,x') - \pi_\theta(t,x)} \leq L_{\pi x}\norm{x'-x} + L_\pi \norm{\theta'-\theta}$. Additionally, we assume the policies $\nabla_x \pi_\theta(t,x)$ and $\nabla_x \mu_\theta(t,x)$ are uniformly bounded. 
    \item Likewise, the Jacobians of $r$ and  $f$ are Lipschitz in $x$ and in $u$.
\end{enumerate}
These smoothness, Lipschitz and boundedness assumptions are fairly standard (and not restrictive) in order to obtain guarantees for smoothness of the policy gradient (see Lemma~\ref{lem: l smoothness}), and understand the dependencies of the regret on various facets of the system. Moreover, they are required to guarantee the existence and uniqueness of the solutions of the ODE system (from the Picard-Lindel\"{o}f theorem).

\section{Policy gradient for the CT-RMDP problem}
To solve the robust optimisation problem \eqref{eq: RMDP}, we alternatingly optimise $\theta$ and $\xi$ in double- or single-loop manner based on policy gradient. At each iteration $k \in [K]$, the policy optimiser performs gradient descent based on $\nabla_\theta J(\theta,\xi)$ with fixed inner problem parameters $\xi$, while the adversary optimiser performs gradient \textit{ascent}, where gradients are based on $\nabla_\xi J(\theta,\xi)$ with fixed policy parameters $\theta$. As a crucial part of our algorithm, we perform a forward-backward algorithm based on the Maximum Principle \citep{Li2018, Ainsworth2021, jin2020_pontryaginProgramming}.

\subsection{Pathwise policy gradient}
A first policy gradient we derive extends the pathwise gradients of \citep{Munos2006} for ODEs. While \cite{Munos2006} considers the estimator for a single, terminal cost, we extend the reasoning here to the total cost \eqref{eq: policy objective} and we assume the availability of $\nabla_x f_{\xi}(t,x_t,u_t)$ rather than approximating it, which is not a strong assumption in the robust MDP case where the model $f_{\xi}$ is typically available. 

Let $(z_t \: \nabla_{\theta} x_t)_{t \geq 0}$ be the process that describes how the state changes with respect to the policy parameters, with \(z_t \in \Reals^{d_x \times d_\theta}\). Using the policy sensitivity process given by 
\begin{align}
\label{eq: state sensitivity}
    \frac{\dif z_t}{\dif t} = \nabla_{\theta} f_{\xi}(t,x_t,u_t) = \nabla_u f_\xi(t,x_t,u_t) \cdot \left( \nabla_x \mu_{\theta}(t,x_t) z_t +\nabla_\theta \mu_\theta(t,x_t) \right) + \nabla_x f_{\xi}(t,x_t,u_t) z_t \,,
\end{align}
a pathwise policy gradient is derived below.
\begin{theorem}[Pathwise policy gradient]
\label{th: pathwise policy gradient}
Let $\mu_{\theta}: [0,T] \times \cX \to \cU$ be a deterministic policy parametrised by $\theta \in \Theta$ and let $f_{\xi}:  [0,T] \times \cX \times \cU \to \cX$ be the ODE dynamics parametrised by $\xi \in \Xi$ according to \eqref{eq: dynamics}. The pathwise policy gradient for $\mu_{\theta}$ is given by
\begin{align}
\nabla_\theta J(\theta,\xi) = \Ex \left[ \int_0^T \left(\nabla_u r(t,x_t,u_t) \cdot (\nabla_x \mu_{\theta}(t,x_t) \cdot z_t + \nabla_\theta \mu_{\theta}(t,x_t)) + \nabla_x r(t,x_t,u_t) \cdot z_t \right) \dif t + \nabla_x R(x_T) \cdot z_T\Big\vert d_0, \mu_{\theta}, f_{\xi}\right] \,,
\end{align} 
where  \(z_t\) is the policy sensitivity process \eqref{eq: state sensitivity}.
\end{theorem}
\begin{proof}
 For an instantaneous cost at the terminal state, the gradient is given by $\nabla_x R(x_T)\cdot z_T$. Additionally differentiating the cost in terms of the total derivative and integrating over the time interval $[0,T)$, we obtain 
 \[\nabla_\theta J(\theta, \xi) = \Ex \left[ \int_0^T \left(\nabla_u r(t,x_t,u_t) \cdot (\nabla_x \mu_{\theta}(t,x_t) \cdot z_t + \nabla_\theta \mu_{\theta}(t,x_t)) + \nabla_x r(t,x_t,u_t) \cdot z_t \right) \dif t + \nabla_x R(x_T) \cdot z_T\right].\]    
\end{proof}

For the adversary \(\xi\), we similarly define the adversary sensitivity process  $(y_t \: \nabla_{\xi} x_t)_{t \geq 0}$, where \(y_t \in \Reals^{d_x \times d_\xi}\):
\begin{align}
\label{eq: adversary sensitivity}
    \frac{\dif y_t}{\dif t} = \nabla_x f_\xi \cdot y_t + \nabla_\xi f_\xi(t, x_t, u_t) + \nabla_u f_\xi(t, x_t, u_t) \cdot \nabla_x \mu_\theta(t, x_t) \cdot y_t.
\end{align} 
This leads to the pathwise adversarial policy gradient below.
\begin{theorem}[Pathwise adversarial policy gradient]
\label{th: pathwise adversarial policy gradient}
Let $\mu_{\theta}: [0,T] \times \cX \to \cU$ be a deterministic policy parametrised by $\theta \in \Theta$ and let $f_{\xi}:  [0,T] \times \cX \times \cU \to \cX$ be the ODE dynamics parametrised by $\xi \in \Xi$ according to \eqref{eq: dynamics}.
\begin{align}
\nabla_\xi J(\theta,\xi) = \Ex \left[ \int_0^T (\nabla_x r(t,x_t,u_t) \cdot y_t +  \nabla_u r(t,x_t,u_t) \nabla_x \mu_{\theta}(t,x_t) \cdot y_t \dif t + \nabla_x R(x_T) \cdot y_T\Big\vert d_0, \mu_{\theta}, f_{\xi}\right] \,,
\end{align}
where $(y_t \: \nabla_{\xi} x_t)_{t \geq 0}$ is the adversary sensitivity process \eqref{eq: adversary sensitivity}.
\end{theorem}
\begin{proof}
The gradient is derived analogously to Theorem~\ref{th: pathwise policy gradient} but now with respect to $\xi$ and with the adversary sensitivity process.
\end{proof}

To compute the gradients $\nabla_u r(t,x_t,u_t)$ and $\nabla_x r(t,x_t,u_t)$, the pathwise gradients above require access to the reward model \(r(\cdot)\), and similarly \(R(\cdot)\). They therefore come under the classification of model-based algorithms. The below adjoint-gradient methods do not require such model access.

\subsection{Adjoint-based policy gradient}
In this section, we derive policy gradients based on the maximum principle, which can learn global feedback policies, using a forward-backward approach. For ODE dynamics, we follow \cite{Li2018}, which uses the method of successive approximations (MSA). The algorithm assumes deterministic system and parametrisation. Similar to Section~4.2 of \cite{Li2018}, a time discretisation based on the Euler method is applied and the maximization is based on a single step of gradient descent. The algorithm is summarised in Algorithm~\ref{alg: MSA}. We fix the parameters of the other player during this iterative process.

\begin{algorithm}
\caption{The Discrete-time MSA algorithm.} \label{alg: MSA}
\begin{algorithmic}[1]
\State \textbf{Input:} discretisation $\Delta t$.
\For {$j = 1,\dots, N_{\text{it}}$}
\State \textbf{Forward:} obtain the state process $(x_t)_{t=0,\Delta t,\dots, T}$ starting from $x_0 = x$ 
\State \textbf{Backward:} obtain the costate process $(p_t)_{t=T,T-\Delta t,\dots,0}$ starting from $p_T = -\nabla R_x(x_T)$
\State \textbf{Maximize Hamiltonian:} $\theta_j = \argmax_{\theta \in \Theta} H(t)$ for all $t = 0, \Delta t, \cdots, T$
\EndFor
\end{algorithmic}
\end{algorithm}

In essence, we first compute forward rollouts and the adjoint (backward) rollouts (alternating between holding the policy parameters fixed and the adversary parameters fixed). We use those values to compute the following policy gradients below.

\begin{remark}
    The (robust) HJB equation of this finite-horizon problem is given by
\begin{align}
\label{eq:robust_HJB}
    \partial_t V(x_t,t) + \inf_u \sup_{\xi} \left[ r(t,x_t, u) + \nabla_x V(t,x_t)^\tr f_{\xi}(t,x_t,u_t) \right] &= 0, \quad V(T,x_T) = R(x_T) \,.
\end{align}
Solving this PDE is an alternate way to obtain the value function and the optimum value, but it is prohibitive in high dimensions and we merely present it for completeness.
\end{remark}

To find an approximate solution to this problem, we implement gradient descent-ascent schemes where one of both parameters is updated based on the obtained trajectories. The Hamiltonian is then computed in backward manner from $p_T = V(T,x_T)$ to $p_0 = V(0,x_0)$.

\paragraph{The policy gradient}
We present our formulations for the robust policy and adversary gradients of the objective \eqref{eq: policy objective} in the cases of stochastic and deterministic policies.

First, we derive the adjoint policy gradient in the case of a stochastic policy. The \textit{sample-and-hold} model describes an action \(u_{t_n} \sim \pi_\theta(t_n, x_{t_n})\) being sampled and held for the time interval \([t_n, t_{n+1})\) (or \([t_n, t_{n-1})\) in the backwards case).

The following is an adjoint deterministic policy gradient formula.
\begin{theorem}[Adjoint-based deterministic policy gradient]
\label{th: deterministic adjoint gradient}
Let $\mu_{\theta}: [0,T] \times \cX \to \cU$ be a deterministic policy parametrised by $\theta \in \Theta$ and let $f_{\xi}:  [0,T] \times \cX \times \cU \to \cX$ be the ODE dynamics parametrised by $\xi \in \Xi$ according to \eqref{eq: dynamics}. The adjoint-based deterministic policy gradient is given by
 \begin{align}
 \nabla_\theta J(\theta, \xi) = \Ex \left[ \int_0^T \left(\nabla_\theta \mu_\theta(t, x_t) + \nabla_x \mu_{\theta}(t, x_t) z_t\right)^\tr \nabla_u H(t, x_t, u_t, p_t) \dif t \bigg\vert d_0, \mu_{\theta}, f_{\xi}\right] \,, 
 \end{align}
 where \(H(t,x_t, u_t, p; \xi) = r(t,x_t,u_t) + p^\tr f_\xi(t,x_t,u_t)\) and   \[- \dot p_t = \nabla_x r(t, x_t, u_t) + \nabla_x f_\xi(t, x_t, u_t)^\tr p_t, \quad p_T = \nabla_x R(x_T), \;\; u_t = \mu_\theta(t, x_t).\] 
 Under the sample-and-hold model, as $\Delta t \to 0$, this converges to the formula for the discrete-time deterministic policy gradient \citep{Silverb}, i.e.
\[\nabla_\theta J_h(\theta, \xi) = \Ex \left[ \sum_{n=0} ^{N-1}\nabla_\theta \mu_\theta(t, x_t)^\tr \nabla_u Q_h(t_n, x_{t_n}, u_n) \right],\] \,
where \(Q_h\) is the cost-to-go from the sample-and-hold system.
\end{theorem}
\begin{proof}
We first analyse the first variation of \(x_t\), $u_t$, and $J$ w.r.t. \(\theta\) and then we eliminate \(\delta x\) using an adjoint. 
\paragraph{Variation w.r.t. $\theta$}:
Take a perturbation $\theta^\epsilon = \theta + \epsilon \delta \theta$, which induces perturbed trajectories $x^\epsilon_t$ and controls $u^\epsilon_t = \mu_{\theta^\epsilon}(t, x^\epsilon_t)$.
Define 
\begin{align*}
\delta x_t \: \frac{\dif}{\dif \epsilon}x^\epsilon_t\bigg\vert_{\epsilon = 0}, \quad \delta u_t \: \frac{\dif}{\dif \epsilon}u^\epsilon_t\bigg\vert_{\epsilon = 0}
\end{align*}
Differentiate the feedback control (linearisation): 
\[u^\epsilon_t = \mu_{\theta^\epsilon}(t, x^\epsilon_t), \quad \delta u_t = \nabla_\theta \mu_\theta(t, x_t) \delta \theta + \nabla_x \mu(t, x_t) \delta x_t.\]
Likewise, we linearise the dynamics to get \(\delta \dot x_t = \nabla_x f_\xi(t, x_t, u_t) \delta x_t + \nabla_u f_\xi(t, x_t, u_t) \delta u_t\).
Differentiating the cost $J$ yields
\[\delta J \: \frac{\dif }{\dif \epsilon}J(\theta^\epsilon, \xi)\bigg\vert_{\epsilon = 0} = \Ex \left[  \int_0^T \left(\nabla_x r(t, x_t, u_t) \delta x_t + \nabla_u r(t, x_t, u_t) \delta u_t \right) \dif t + \nabla_x R(x_T) \delta x_T\right]\]
        
\paragraph{Costate to eliminate \(\delta x\):}
Define the costate as \(p_t \in \Reals^{d_x}\) by the following adjoint ODE:
\[- \dot p_t = \nabla_x r(t, x_t, u_t) + \nabla_x f_\xi(t, x_t, u_t)^\tr p_t, \quad p_T = \nabla_x R(x_T)\]
Using \(\delta x_0 = 0\) (the initial state does not depend on $\theta$), 
\[p_T^\tr \delta x_T = \int_0^T \left ( \dot p_t \delta x_t + p_t^\tr \delta \dot x_t\right) \dif t\]
Substituting for \(\dot p_t\) and \(\delta \dot x_t\), we obtain
\[p_T^\tr \delta x_T = \int_0^T \left(\left(-\nabla_x r(t, x_t, u_t) - \nabla_x f_\xi(t, x_t, u_t)^\tr p_t \right)^\tr\delta x_t + p_t^\tr \left( \nabla_x f_\xi(t, x_t, u_t) \delta x_t + \nabla_u f_\xi(t, x_t, u_t) \delta u_t\right) \right) \dif t\]
which simplifies to
\[p_T^\tr \delta x_T  + \int_0^T \nabla_x r(t, x_t, u_t)^\tr \delta x_t \dif t = \int_0^T p_t^\tr  \nabla_u f_\xi(t, x_t, u_t) \delta u_t \dif t\]
This leads to 
\[\delta J = \Ex \left [ \int_0^t \left( \underbrace{\nabla_u r(t, x_t, u_t)  +  \nabla_u f_\xi(t, x_t, u_t)^\tr p_t}_{=: \nabla_u H(t, x_t, u_t, p_t; \xi)}\right)^\tr \delta u_t \dif t\right]\]
Now, we substitute in the linearisation of \(\delta u_t\) to obtain
 \[\delta J = \Ex \left[ \int_0^T \left(\left(\nabla_\theta \mu_\theta(t, x_t)+ \nabla_x \mu(t, x_t) z_t\right)^\tr \nabla_u H(t, x_t, u_t, p_t;\xi)\right) \dif t\right] \cdot \delta \theta\]
noting that \(\delta x_t = z_t \delta \theta\). Since this holds for all directions \(\delta \theta\), the policy gradient is
\[\nabla_\theta J(\theta, \xi) = \Ex \left[ \int_0^T \left(\nabla_\theta \mu_\theta(t, x_t) + \nabla_x \mu(t, x_t) z_t\right)^\tr \nabla_u H(t, x_t, u_t, p_t;\xi) \dif t\right]\]

\paragraph{Equivalence to discrete-time deterministic policy gradient}
Under the sample-and-hold model, the gradient w.r.t. $x$ drops from the policy, and with $h \equiv \Delta t \to 0$, we have
\begin{align*}
\nabla_\theta J(\theta, \xi) &= \Ex \left[ \int_0^T \left(\nabla_\theta \mu_\theta(t, x_t)\right)^\tr \nabla_u H(t, x_t, u_t, p_t;\xi) \dif t\right]  \\
&= \Ex \left[ \int_0^T \left(\nabla_\theta \mu_\theta(t, x_t)\right)^\tr \nabla_u \left(r(t,x_t,u_t) + f_\xi(t, x_t, u_t)^\tr p_t + \partial_t V(t,x_t)\right) \dif t \right]  \\
&= \Ex \left[ \int_0^T \left(\nabla_\theta \mu_\theta(t, x_t)\right)^\tr \nabla_u Q(t,x_t,u_t) \dif t \right]  \\
&\approx \Ex \left[ \sum_{n=0}^{N-1} h \left(\nabla_\theta \mu_\theta(t_n, x_{t_n})\right)^\tr \nabla_u Q(t_{n},x_{t_n},u_{t_n})  \right]   \\
&= \Ex \left[ \sum_{n=0}^{N-1}  \left(\nabla_\theta \mu_\theta(t_n, x_{t_n})\right)^\tr \nabla_u Q_h(t_n,x_{t_n},u_{t_n}) \right]  \,.
\end{align*}
\end{proof}

Under the deterministic policy \(\mu_\theta\), we obtain the following adversarial gradient. 
\begin{theorem}[Adjoint-based deterministic adversarial policy gradient]
Let $f_{\xi}:  [0,T] \times \cX \times \cU \to \cX$ be the ODE dynamics parametrised by $\xi \in \Xi$ according to \eqref{eq: dynamics}. Then the adjoint-based adversarial policy gradient is
\begin{align}
\label{eq: deterministic adjoint adversary}
\nabla_\xi J(\theta,\xi) &= \Ex \left[ \int_0^T p_t^{\tr} \nabla_{\xi} f_{\xi}(t,x_t,u_t) \dif t \Big\vert d_0, f_{\xi}\right] 
\end{align}
where the action $u_t$ is based on a stochastic or deterministic policy, and $p_t = \nabla_x V(t,x_t)$ is the costate.
\end{theorem}
\begin{proof} We compute the directional derivative of the objective with respect to the adversarial parameter $\xi$, while keeping the policy parameter $\theta$ fixed. Importantly, fixing $\theta$ does not fix the realised control trajectory, since the feedback control $u_t=\mu_\theta(t,x_t)$ changes whenever the state trajectory changes. Let \[ \xi^\varepsilon=\xi+\varepsilon\delta\xi, \] for an arbitrary direction $\delta\xi\in\mathbb{R}^{d_\xi}$, and let $x_t^\varepsilon$ and $u_t^\varepsilon$ denote the corresponding state and control trajectories: \[ \dot{x}_t^\varepsilon = f_{\xi^\varepsilon} (t,x_t^\varepsilon,u_t^\varepsilon), \qquad u_t^\varepsilon = \mu_\theta(t,x_t^\varepsilon). \] Define the first variations \[ \delta x_t := \left. \frac{d}{d\varepsilon}x_t^\varepsilon \right|_{\varepsilon=0}, \qquad \delta u_t := \left. \frac{d}{d\varepsilon}u_t^\varepsilon \right|_{\varepsilon=0}. \] Since $\theta$ is fixed but the policy is state feedback, \[ \delta u_t = \nabla_x\mu_\theta(t,x_t)\,\delta x_t. \] Differentiating the state dynamics therefore gives \[ \begin{aligned} \delta\dot{x}_t &= \nabla_x f_\xi(t,x_t,u_t)\,\delta x_t + \nabla_u f_\xi(t,x_t,u_t)\,\delta u_t + \nabla_\xi f_\xi(t,x_t,u_t)\,\delta\xi \\ &= \left( \nabla_x f_\xi(t,x_t,u_t) + \nabla_u f_\xi(t,x_t,u_t) \nabla_x\mu_\theta(t,x_t) \right)\delta x_t + \nabla_\xi f_\xi(t,x_t,u_t)\,\delta\xi, \end{aligned} \] with $\delta x_0=0$, since the initial distribution is independent of $\xi$. For compactness, define \[ A_t := \nabla_x f_\xi(t,x_t,u_t) + \nabla_u f_\xi(t,x_t,u_t) \nabla_x\mu_\theta(t,x_t). \] The state sensitivity equation becomes \[ \delta\dot{x}_t = A_t\delta x_t + \nabla_\xi f_\xi(t,x_t,u_t)\,\delta\xi. \] We next differentiate the objective. Since the running and terminal costs do not depend explicitly on $\xi$, \[ \begin{aligned} \delta J &:= DJ(\theta,\xi)[\delta\xi] \\ &= \Ex\left[ \int_0^T \left( \nabla_x r(t,x_t,u_t)^\top\delta x_t + \nabla_u r(t,x_t,u_t)^\top\delta u_t \right)dt + \nabla_x R(x_T)^\top\delta x_T \right] \\ &= \Ex\left[ \int_0^T \left( \nabla_x r(t,x_t,u_t) + \nabla_x\mu_\theta(t,x_t)^\top \nabla_u r(t,x_t,u_t) \right)^\top \delta x_t\,dt + \nabla_x R(x_T)^\top\delta x_T \right]. \end{aligned} \] Define \[ c_t := \nabla_x r(t,x_t,u_t) + \nabla_x\mu_\theta(t,x_t)^\top \nabla_u r(t,x_t,u_t). \] Thus \[ \delta J = \Ex\left[ \int_0^T c_t^\top\delta x_t\,dt + \nabla_xR(x_T)^\top\delta x_T \right]. \] Introduce the adjoint process satisfying \[ -\dot{p}_t = c_t+A_t^\top p_t, \qquad p_T=\nabla_xR(x_T). \] Using the product rule, \[ \frac{d}{dt} \left(p_t^\top\delta x_t\right) = \dot{p}_t^\top\delta x_t + p_t^\top\delta\dot{x}_t. \] Substituting the adjoint and sensitivity equations gives \[ \begin{aligned} \frac{d}{dt} \left(p_t^\top\delta x_t\right) &= \left(-c_t-A_t^\top p_t\right)^\top\delta x_t + p_t^\top \left( A_t\delta x_t + \nabla_\xi f_\xi(t,x_t,u_t)\delta\xi \right) \\ &= -c_t^\top\delta x_t + p_t^\top \nabla_\xi f_\xi(t,x_t,u_t)\delta\xi. \end{aligned} \] Integrating from $0$ to $T$ yields \[ p_T^\top\delta x_T-p_0^\top\delta x_0 = -\int_0^T c_t^\top\delta x_t\,dt + \int_0^T p_t^\top \nabla_\xi f_\xi(t,x_t,u_t)\delta\xi\,dt. \] Since $\delta x_0=0$ and $p_T=\nabla_xR(x_T)$, \[ \int_0^T c_t^\top\delta x_t\,dt + \nabla_xR(x_T)^\top\delta x_T = \int_0^T p_t^\top \nabla_\xi f_\xi(t,x_t,u_t)\delta\xi\,dt. \] Substitution into the expression for $\delta J$ gives \[ DJ(\theta,\xi)[\delta\xi] = \Ex\left[ \int_0^T p_t^\top \nabla_\xi f_\xi(t,x_t,u_t)\,dt \right]\delta\xi. \] Since the above holds for every direction $\delta\xi$, \[ \nabla_\xi J(\theta,\xi) = \Ex\left[ \int_0^T \nabla_\xi f_\xi(t,x_t,u_t)^\top p_t\,dt \right]. \] Finally, expanding $c_t+A_t^\top p_t$ gives \[ \begin{aligned} -\dot p_t &= \nabla_x r + \nabla_x f_\xi^\top p_t + \nabla_x\mu_\theta^\top \left( \nabla_u r+\nabla_u f_\xi^\top p_t \right) \\ &= \nabla_xH(t,x_t,u_t,p_t;\xi) + \nabla_x\mu_\theta(t,x_t)^\top \nabla_uH(t,x_t,u_t,p_t;\xi), \end{aligned} \] which proves the result. \end{proof}

\begin{theorem}[Stochastic adjoint-based policy gradient and sample-and-hold limit]
\label{th: stochastic adjoint gradient}
Let
\[
    \pi_\theta(\cdot\mid t,x)\in\cP(\cU),
    \qquad \theta\in\Theta,
\]
be a stochastic feedback policy admitting a differentiable density with
respect to a fixed reference measure on $\cU$, and let
$f_\xi:[0,T]\times\cX\times\cU\to\cX$
be the ODE dynamics in \eqref{eq: dynamics}, with $\xi\in\Xi$ fixed.
For a uniform discretisation
\[
    0=t_0<t_1<\cdots<t_N=T,
    \qquad h=t_{n+1}-t_n,
\]
consider the sample-and-hold policy
\[
    U_n\sim\pi_\theta(\cdot\mid t_n,X_n),
\]
where $U_n$ is held fixed on $[t_n,t_{n+1})$.
Let $J_h(\theta,\xi)$ denote the corresponding expected total cost,
and let $Q_h(t_n,x,u)$ and
$V_h^{\pi_\theta}(t_n,x)$ denote its action-value and value functions.
Then
\begin{equation}
\label{eq:discrete_stochastic_pg}
    \nabla_\theta J_h(\theta,\xi)
    =
    \Ex\left[
        \sum_{n=0}^{N-1}
        \nabla_\theta
        \log\pi_\theta(U_n\mid t_n,X_n)
        Q_h(t_n,X_n,U_n)
    \right].
\end{equation}
Equivalently, defining the discrete advantage rate
\begin{equation}
\label{eq:discrete_advantage_rate}
    \cA_h^{\pi_\theta}(t_n,x,u)
    :=
    \frac{
        Q_h(t_n,x,u)
        -
        V_h^{\pi_\theta}(t_n,x)
    }{h},
\end{equation}
we have
\begin{equation}
\label{eq:discrete_stochastic_pg_adv}
    \nabla_\theta J_h(\theta,\xi)
    =
    \Ex\left[
        \sum_{n=0}^{N-1}
        h\,
        \nabla_\theta
        \log\pi_\theta(U_n\mid t_n,X_n)
        \cA_h^{\pi_\theta}(t_n,X_n,U_n)
    \right].
\end{equation}
Assume, in addition, that the sample-and-hold state processes converge as
$h\to0$ to the relaxed continuous-time dynamics
\begin{equation}
\label{eq:relaxed_dynamics}
    \dot X_t
    =
    \bar f_{\theta,\xi}(t,X_t),
    \qquad
    \bar f_{\theta,\xi}(t,x)
    :=
    \int_{\cU}
    f_\xi(t,x,u)\pi_\theta(du\mid t,x),
\end{equation}
and that the corresponding value functions admit $C^1$ interpolations
converging in $C^1$ to $V^{\pi_\theta}$. Assume also the integrability
conditions required to interchange differentiation, integration, and the
limit $h\to0$.
Define the continuous-time advantage rate by
\begin{equation}
\label{eq:continuous_advantage_rate}
    \cA^{\pi_\theta}(t,x,u)
    :=
    r(t,x,u)
    +
    \partial_t V^{\pi_\theta}(t,x)
    +
    \nabla_xV^{\pi_\theta}(t,x)^\top
    f_\xi(t,x,u).
\end{equation}
Then
\begin{equation}
\label{eq:continuous_stochastic_pg}
\boxed{
    \nabla_\theta J(\theta,\xi)
    =
    \Ex\left[
        \int_0^T
        \nabla_\theta
        \log\pi_\theta(U_t\mid t,X_t)
        \cA^{\pi_\theta}(t,X_t,U_t)
        \,dt
    \right],
}
\end{equation}
where, conditionally on $(t,X_t)$,
$U_t\sim\pi_\theta(\cdot\mid t,X_t)$.
Moreover, let
\[
    H(t,x,u,p;\xi)
    :=
    r(t,x,u)+p^\top f_\xi(t,x,u)
\]
and define the policy-averaged running cost
\[
    \bar r_\theta(t,x)
    :=
    \int_{\cU}
    r(t,x,u)\pi_\theta(du\mid t,x).
\]
Let $p_t$ solve the closed-loop adjoint equation
\begin{equation}
\label{eq:stochastic_closed_loop_adjoint}
    -\dot p_t
    =
    \nabla_x\bar r_\theta(t,X_t)
    +
    \nabla_x\bar f_{\theta,\xi}(t,X_t)^\top p_t,
    \qquad
    p_T=\nabla_xR(X_T).
\end{equation}
Then $p_t=\nabla_xV^{\pi_\theta}(t,X_t)$ and
\eqref{eq:continuous_stochastic_pg} is equivalently
\begin{equation}
\label{eq:continuous_stochastic_hamiltonian_pg}
\boxed{
    \nabla_\theta J(\theta,\xi)
    =
    \Ex\left[
        \int_0^T
        \nabla_\theta
        \log\pi_\theta(U_t\mid t,X_t)
        H(t,X_t,U_t,p_t;\xi)
        \,dt
    \right].
}
\end{equation}
The adjoint equation can alternatively be written as
\begin{equation}
\label{eq:stochastic_closed_loop_adjoint_expanded}
    -\dot p_t
    =
    \Ex_{U\sim\pi_\theta(\cdot\mid t,X_t)}
    \left[
        \nabla_xH(t,X_t,U,p_t;\xi)
        +
        H(t,X_t,U,p_t;\xi)
        \nabla_x\log\pi_\theta(U\mid t,X_t)
    \right].
\end{equation}
\end{theorem}

\begin{proof}
We first derive the exact policy-gradient identity for the
sample-and-hold system and then take the continuous-time limit.
\paragraph{Step 1: Exact likelihood-ratio gradient for the
sample-and-hold system.}
Let
\[
    \tau_h=(X_0,U_0,X_1,U_1,\ldots,X_N)
\]
denote a sample-and-hold trajectory. Conditional on $X_n$ and $U_n$,
the transition from $X_n$ to $X_{n+1}$ is determined by the ODE dynamics
and hence does not depend explicitly on $\theta$. Therefore the
$\theta$-dependence of the trajectory law occurs through the policy
factors only. In particular,
\[
    p_\theta(\tau_h)
    =
    d_0(X_0)
    \prod_{n=0}^{N-1}
    \pi_\theta(U_n\mid t_n,X_n)
    P_h^\xi(X_{n+1}\mid X_n,U_n),
\]
where $P_h^\xi$ denotes the transition induced by holding $U_n$
constant over $[t_n,t_{n+1})$.
Consequently,
\[
    \nabla_\theta\log p_\theta(\tau_h)
    =
    \sum_{n=0}^{N-1}
    \nabla_\theta
    \log\pi_\theta(U_n\mid t_n,X_n).
\]
Applying the likelihood-ratio identity gives
\[
    \nabla_\theta J_h(\theta,\xi)
    =
    \Ex\left[
        G_h(\tau_h)
        \sum_{n=0}^{N-1}
        \nabla_\theta
        \log\pi_\theta(U_n\mid t_n,X_n)
    \right],
\]
where $G_h(\tau_h)$ denotes the total sample-and-hold cost.
By causality, costs accumulated before time $t_n$ have zero contribution
after multiplication by
$\nabla_\theta\log\pi_\theta(U_n\mid t_n,X_n)$.
Conditioning the remaining return on $(X_n,U_n)$ therefore yields
\[
    \nabla_\theta J_h(\theta,\xi)
    =
    \Ex\left[
        \sum_{n=0}^{N-1}
        \nabla_\theta
        \log\pi_\theta(U_n\mid t_n,X_n)
        Q_h(t_n,X_n,U_n)
    \right].
\]
This proves \eqref{eq:discrete_stochastic_pg}.
Importantly,
$\nabla_\theta\log\pi_\theta(U_n\mid t_n,X_n)$ is the partial
derivative with respect to $\theta$, with $(t_n,X_n)$ held fixed.

\paragraph{Step 2: Subtraction of a state-dependent baseline.}
By definition,
\[
    V_h^{\pi_\theta}(t_n,x)
    =
    \Ex_{U\sim\pi_\theta(\cdot\mid t_n,x)}
    \left[
        Q_h(t_n,x,U)
    \right].
\]
For any function $B(t_n,X_n)$ that is independent of $U_n$ conditional
on $X_n$,
\[
\begin{aligned}
    &
    \Ex\left[
        \nabla_\theta
        \log\pi_\theta(U_n\mid t_n,X_n)
        B(t_n,X_n)
        \,\middle|\, X_n
    \right]
    \\
    &\qquad=
    B(t_n,X_n)
    \int_{\cU}
    \pi_\theta(du\mid t_n,X_n)
    \nabla_\theta
    \log\pi_\theta(u\mid t_n,X_n)
    \\
    &\qquad=
    B(t_n,X_n)
    \nabla_\theta
    \int_{\cU}
    \pi_\theta(du\mid t_n,X_n)
    =0.
\end{aligned}
\]
Taking
$B(t_n,X_n)=V_h^{\pi_\theta}(t_n,X_n)$ therefore gives
\[
\begin{aligned}
    \nabla_\theta J_h(\theta,\xi)
    =
    \Ex\Bigg[
        \sum_{n=0}^{N-1}
        \nabla_\theta
        \log\pi_\theta(U_n\mid t_n,X_n)
        \Big(
            Q_h(t_n,X_n,U_n)
            -
            V_h^{\pi_\theta}(t_n,X_n)
        \Big)
    \Bigg].
\end{aligned}
\]
Using the definition
\[
    \cA_h^{\pi_\theta}
    :=
    \frac{Q_h-V_h^{\pi_\theta}}{h},
\]
we obtain the exact identity
\[
    \nabla_\theta J_h(\theta,\xi)
    =
    \Ex\left[
        \sum_{n=0}^{N-1}
        h\,
        \nabla_\theta
        \log\pi_\theta(U_n\mid t_n,X_n)
        \cA_h^{\pi_\theta}(t_n,X_n,U_n)
    \right].
\]
\paragraph{Step 3: Continuous-time limit of the sample-and-hold
state process.}
Write $X^h$ for the state process generated by the sample-and-hold
policy with step size $h$. At the grid points $t_n=nh$, smoothness of
$f_\xi$ gives
\[
    X_{n+1}^h
    =
    X_n^h
    +
    h f_\xi(t_n,X_n^h,U_n)
    +
    R_{n+1}^h,
    \qquad
    \|R_{n+1}^h\|\le Ch^2.
\]
Define the policy-averaged drift
\[
    \bar f_{\theta,\xi}(t,x)
    :=
    \Ex_{U\sim\pi_\theta(\cdot\mid t,x)}
    [f_\xi(t,x,U)]
\]
and the centred fluctuation
\[
    M_{n+1}
    :=
    f_\xi(t_n,X_n^h,U_n)
    -
    \bar f_{\theta,\xi}(t_n,X_n^h).
\]
Then
\[
    X_{n+1}^h
    =
    X_n^h
    +
    h\bar f_{\theta,\xi}(t_n,X_n^h)
    +
    hM_{n+1}
    +
    R_{n+1}^h.
\]
If $\cF_n$ denotes the filtration immediately before sampling
$U_n$, then
\[
    \Ex[M_{n+1}\mid\cF_n]=0.
\]
Assuming
\[
    \Ex[\|M_{n+1}\|^2\mid\cF_n]\le C_M,
\]
the martingale
\[
    S_k^h:=\sum_{n=0}^{k-1}hM_{n+1}
\]
satisfies, by Doob's maximal inequality,
\[
\begin{aligned}
    \Ex\left[
        \max_{k\le N}\|S_k^h\|^2
    \right]
    &\le
    4\Ex\|S_N^h\|^2
    \\
    &\le
    4h^2
    \sum_{n=0}^{N-1}
    \Ex\|M_{n+1}\|^2
    \\
    &\le
    4C_MT h.
\end{aligned}
\]
Thus the accumulated fluctuations induced by the sampled actions
vanish in $L^2$ as $h\to0$.
Let $\bar X_n^h$ denote the Euler approximation of the averaged ODE
\[
    \dot{\bar X}_t
    =
    \bar f_{\theta,\xi}(t,\bar X_t),
    \qquad
    \bar X_0=X_0,
\]
namely
\[
    \bar X_{n+1}^h
    =
    \bar X_n^h
    +
    h\bar f_{\theta,\xi}(t_n,\bar X_n^h).
\]
If $\bar f_{\theta,\xi}$ is uniformly Lipschitz in $x$, subtraction
of the two recursions and the discrete Gronwall inequality imply
\[
    \max_{k\le N}
    \|X_k^h-\bar X_k^h\|
    \le
    C_T
    \left(
        \max_{k\le N}\|S_k^h\|
        +
        h
    \right).
\]
Consequently,
\[
    \Ex\left[
        \max_{k\le N}
        \|X_k^h-\bar X_k^h\|^2
    \right]
    \le C_T h.
\]
Finally, standard Euler convergence for the averaged ODE gives
\[
    \max_{k\le N}
    \|\bar X_k^h-\bar X(t_k)\|
    =O(h).
\]
Combining the two estimates, and using the $O(h)$ variation of both
processes within each grid interval, yields
\[
    \Ex\left[
        \sup_{0\le t\le T}
        \|X_t^h-\bar X_t\|^2
    \right]
    \le C_T h.
\]
In particular,
\[
    X^h\longrightarrow\bar X
\]
uniformly in probability (indeed, in $L^2$), where
\[
    \dot{\bar X}_t=\bar f_{\theta,\xi}(t,\bar X_t).
\]
\paragraph{Step 4: Limit of the discrete advantage rate.}
Let $\Phi_h^\xi(t,x,u)$ denote the state reached after evolving the ODE
for time $h$ from state $x$, while holding action $u$ fixed. Smoothness
of the dynamics gives
\[
    \Phi_h^\xi(t,x,u)
    =
    x+h f_\xi(t,x,u)+O(h^2).
\]
Similarly, the running cost accumulated over one interval satisfies
\[
    c_h(t,x,u)
    =
    h r(t,x,u)+O(h^2).
\]
By the dynamic-programming decomposition,
\[
    Q_h(t,x,u)
    =
    c_h(t,x,u)
    +
    V_h^{\pi_\theta}
    \left(
        t+h,\Phi_h^\xi(t,x,u)
    \right).
\]
Hence
\[
\begin{aligned}
    \cA_h^{\pi_\theta}(t,x,u)
    &=
    r(t,x,u)+
    \frac{
        V_h^{\pi_\theta}
        \left(
            t+h,
            x+h f_\xi(t,x,u)+O(h^2)
        \right)
        -
        V_h^{\pi_\theta}(t,x)
    }{h}
    +
    O(h).
\end{aligned}
\]
By the assumed $C^1$ consistency of $V_h^{\pi_\theta}$,
\[
    \cA_h^{\pi_\theta}(t,x,u)
    \longrightarrow
    r(t,x,u)
    +
    \partial_tV^{\pi_\theta}(t,x)
    +
    \nabla_xV^{\pi_\theta}(t,x)^\top
    f_\xi(t,x,u).
\]
Thus
\[
    \cA_h^{\pi_\theta}(t,x,u)
    \longrightarrow
    \cA^{\pi_\theta}(t,x,u).
\]
Substituting this convergence into
\eqref{eq:discrete_stochastic_pg_adv}, and using convergence of the
sample-and-hold occupation measures together with the assumed uniform
integrability, the Riemann sums converge to
\[
    \nabla_\theta J(\theta,\xi)
    =
    \Ex\left[
        \int_0^T
        \nabla_\theta
        \log\pi_\theta(U_t\mid t,X_t)
        \cA^{\pi_\theta}(t,X_t,U_t)
        \,dt
    \right].
\]
This proves \eqref{eq:continuous_stochastic_pg}.
\paragraph{Step 5: Adjoint and Hamiltonian representation.}
For the limiting relaxed dynamics, the value of the current policy
satisfies the policy-evaluation equation
\[
    \partial_tV^{\pi_\theta}(t,x)
    +
    \bar r_\theta(t,x)
    +
    \nabla_xV^{\pi_\theta}(t,x)^\top
    \bar f_{\theta,\xi}(t,x)
    =
    0,
\]
with
\[
    V^{\pi_\theta}(T,x)=R(x).
\]
Along the state trajectory define
\[
    p_t:=\nabla_xV^{\pi_\theta}(t,X_t).
\]
Differentiating the policy-evaluation equation with respect to $x$ and
evaluating along $X_t$ gives
\[
    -\dot p_t
    =
    \nabla_x\bar r_\theta(t,X_t)
    +
    \nabla_x\bar f_{\theta,\xi}(t,X_t)^\top p_t,
    \qquad
    p_T=\nabla_xR(X_T).
\]
Hence $p_t$ is precisely the adjoint of the policy-averaged
continuous-time system.
Now,
\[
    \cA^{\pi_\theta}(t,x,u)
    =
    H(t,x,u,\nabla_xV^{\pi_\theta}(t,x);\xi)
    +
    \partial_tV^{\pi_\theta}(t,x).
\]
Since $\partial_tV^{\pi_\theta}(t,x)$ is independent of $u$,
\[
\begin{aligned}
    &
    \Ex_{U\sim\pi_\theta(\cdot\mid t,x)}
    \left[
        \nabla_\theta\log\pi_\theta(U\mid t,x)
        \partial_tV^{\pi_\theta}(t,x)
    \right]
    \\
    &\qquad=
    \partial_tV^{\pi_\theta}(t,x)
    \nabla_\theta
    \int_{\cU}
    \pi_\theta(du\mid t,x)
    =0.
\end{aligned}
\]
Therefore
\[
    \nabla_\theta J(\theta,\xi)
    =
    \Ex\left[
        \int_0^T
        \nabla_\theta
        \log\pi_\theta(U_t\mid t,X_t)
        H(t,X_t,U_t,p_t;\xi)
        \,dt
    \right].
\]
Finally, differentiating the policy averages with respect to $x$ gives
\[
    \nabla_x\bar r_\theta(t,x)
    =
    \Ex_{U\sim\pi_\theta}
    \left[
        \nabla_xr(t,x,U)
        +
        r(t,x,U)
        \nabla_x\log\pi_\theta(U\mid t,x)
    \right]
\]
and
\[
    \nabla_x\bar f_{\theta,\xi}(t,x)
    =
    \Ex_{U\sim\pi_\theta}
    \left[
        \nabla_xf_\xi(t,x,U)
        +
        f_\xi(t,x,U)
        \nabla_x\log\pi_\theta(U\mid t,x)^\top
    \right].
\]
Substitution into the adjoint equation yields
\[
    -\dot p_t
    =
    \Ex_{U\sim\pi_\theta(\cdot\mid t,X_t)}
    \left[
        \nabla_xH(t,X_t,U,p_t;\xi)
        +
        H(t,X_t,U,p_t;\xi)
        \nabla_x\log\pi_\theta(U\mid t,X_t)
    \right].
\]
This proves the final representation.
\end{proof}

\subsection{Extension to SDEs}
\label{sec: extension to SDEs}
Previously, we had considered an ODE-based robust policy optimisation with a static adversary \(\xi \in \Xi\). In the following, we deal with the case where the underlying dynamics are represented by a stochastic differential equation (SDE) instead. More specifically, we consider the problem of robust continuous-time RL under \textit{controlled diffusions}. The main considerations we make are using a forward-backward stochastic differential equation (FBSDE) via the stochastic maximum principle to compute gradients. To this end, we present a unified min-max gradient framework where both \(\nabla_\theta J(\theta, \xi)\) and \(\nabla_\xi J(\theta, \xi)\) are computed with forward simulation and a backward adjoint SDE, avoiding high-dimensional PDE solvers. To this end, we consider SDEs of the form
\begin{equation}
\label{eq: SDE}
\dif X_t = b_{\xi}(t,X_t, U_t) \dif t +
\sigma_{\xi}(t,X_t, U_t) \dif W_t \,,
\end{equation}
where \(X_t \in \Reals^{d_x}\) as before, the drift \(b_\xi(t, X_t, U_t) \in \Reals^{d_x}\), and the diffusion \(\sigma_\xi(t, X_t, U_t) \in \Reals^{d_x \times d_m}\) where \(d_m\) is the dimension of \(\dif W_t\). 
The corresponding stochastic HJB equation is
\begin{equation}
\partial_t V(t,X_t) + \min_{u} \{\cA V(t,X_t) + r(t,X_t,U_t)\} = 0 \,,
\end{equation}
where $\cA$ is the infinitesimal generator of the process \eqref{eq: SDE}, defined to be
\begin{equation}
    \cA V(t, X_t) \: \nabla_x V(t,X_t)^\tr b_\xi(t, X_t, U_t) + \frac{1}{2}(\sigma_\xi(t,X_t, U_t)\sigma_\xi(t,X_t, U_t)^\tr\nabla^2 V(t, X_t)).
\end{equation}
We include the above for completeness, as it would be useful for (i) Verification: if \(V(t, X_t)\) solves the HJB and \(u^*(t) = \pi_{\theta^*}(t, X_t)\) minimizes the Hamiltonian, then it is optimal, and (ii) Duality viewpoints between the stochastic maximum principle and PDE theory. Solving the HJB is intractable beyond small \(d_x\) in general and as such, we will employ the stochastic maximum principle for implementable gradients.

\subsubsection{Stochastic Maximum Principle}
For the stochastic system \eqref{eq: SDE} above, with the same cost function \eqref{eq: policy objective}, we can construct a Hamiltonian and an adjoint backward SDE (BSDE) describing the costate process. 

In this setting, the Hamiltonian is given by 
\begin{equation}
\label{eq: hamiltonian for SDE}
H(t, X_t, U_t, P_t, Q_t; \xi) \: r(t, X_t, U_t) + P_T^\tr b_\xi(t, X_t, U_t) + \Tr(Q_t^\tr \sigma_\xi(t, X_t, U_t)) \,,
\end{equation} 
where there are now two costates \((P_t, Q_t)\) compared to just $p_t$ in the deterministic dynamics setting. The costate $P_t$ has the same interpretation as $p_t$ -- i.e. equivalent to \(\nabla_x V(t, X_t) \in \Reals^{d_x}\) -- while the additional costate
\begin{equation}
\label{eq: Q costate}
Q_t = \nabla^2_x V(t, X_t)\sigma_\xi(t,X_t) \in \Reals^{d_x \times d_m}    
\end{equation}
is interpreted as how the curvature of the value couples with the noise and its effect on the expected cost.\footnote{ This second costate \(Q_t\) is not equivalent to the discrete-time \(Q_h\)-function in earlier sections.}

The adjoint BSDE is the following system:
\begin{align*}
    \dif P_t &= - \nabla_x H(t, X_t, U_t, P_t, Q_t; \xi) \dif t + Q_t \dif W_t\\
    P_T &= \nabla_x R(X_T) \,.
\end{align*}
Note that \(\nabla_x H(t, X_t, U_t, P_t, Q_t; \xi) = \nabla_x  r(t, X_t, U_t) + P_t^\tr \nabla_x b_\xi(t, X_t, U_t) + Q_t^\tr \nabla_x \sigma_\xi(t, X_t, U_t)\), and similarly \(\nabla_u H(t, X_t, U_t, P_t, Q_t; \xi) = \nabla_u r(t, X_t, U_t) + P_t^\tr \nabla_u b_\xi(t, X_t, U_t) + Q_t^\tr \nabla_u \sigma_\xi(t, X_t, U_t)\). These will be useful for the construction of the policy gradients below.
Direct computation of $P_t$ and $Q_t$ through $V$ is generally impractical in high dimensions. However, they can be jointly approximated using standard numerical BSDE methods \citep{BouchardTouzi2004}.
To make this explicit, consider a uniform grid
\[
    0=t_0<t_1<\cdots<t_N=T,
    \qquad h=t_{n+1}-t_n,
\]
and write the adjoint BSDE generically as
\[
    dP_t=-F_t\,dt+Q_t\,dW_t.
\]
A one-step discretisation gives
\[
    P_{n+1}^h
    =
    P_n^h
    -
    hF_n^h
    +
    Q_n^h\Delta W_n
    +
    \text{local discretisation error}.
\]
Multiplying by $\Delta W_n^\top$ and conditioning on
$\cF_{t_n}$ motivates the standard approximation
\begin{equation}
\label{eq:Q_conditional_expectation}
    Q_n^h
    \approx
    \frac{1}{h}
    \Ex
    \left[
        P_{n+1}^h
        \Delta W_n^\top
        \,\middle|\,
        \cF_{t_n}
    \right].
\end{equation}
Indeed,
\[
    \Ex[
        \Delta W_n
        \mid\cF_{t_n}
    ]=0,
    \qquad
    \Ex[
        \Delta W_n\Delta W_n^\top
        \mid\cF_{t_n}
    ]=hI,
\]
where $\cF_{t_n}$ denotes the filtration of the BSDE. In the Markovian setting, the conditional expectation in
\eqref{eq:Q_conditional_expectation} can be written as a function of
$X_n$ and estimated using Monte Carlo regression. The corresponding
backward recursion can therefore approximate both components of the
adjoint pair $(P,Q)$ without requiring an explicit solution of the
value-function PDE.
For the theoretical gradient and regret analysis below, we continue to
adopt an exact-adjoint oracle abstraction: namely, we assume exact access
to the adapted solution $(P,Q)$. This
assumption isolates the optimisation analysis from the numerical error
of an inner BSDE solver. Numerical approximation of $(P,Q)$ would add an
additional gradient-estimation error, but does not alter the analytical
gradient identities derived below.

\subsubsection{Policy gradients for the SDE case}
Analogously to the deterministic setup, we also construct a (now stochastic) process for the sensitivity of the state \(X_t\) with respect to the policy parameter \(\theta\), denoted by $Z_t \: \nabla_\theta X_t$. Differentiating w.r.t. time to get the parameter sensitivity process, we obtain a linear SDE
\[\dif Z_t = (\nabla_x b_\xi(t, X_t, U_t) \cdot Z_t + \nabla_u b_\xi(t, X_t, U_t) \cdot \nabla_\theta U_t) \dif t + (\nabla_x \sigma_\xi(t, X_t, U_t) \cdot Z_t + \nabla_u \sigma_\xi(t, X_t, U_t) \cdot \nabla_\theta U_t) \dif W_t.\]
where we note that \(\nabla_\theta U_t \equiv \nabla_\theta \mu_\theta(t,X_t) + \nabla_{x} \mu_{\theta}(t,X_t) Z_t\). 

We have via Ito's lemma that 
\begin{align*}
\dif V(t,X_t) = \left( \partial_t V(t,X_t) + (\nabla_{x} V(t,X_t))^\tr b_{\xi}(t,X_t) +  \frac{1}{2} \nabla_x^2 V(t,X_t)  D(t,X_t)\right) \dif t + \nabla_x V(t,X_t) \sigma_{\xi}(t,X_t) \dif W_t \,,
\end{align*}
where $b_{\xi}(t,X_t) = b_{\xi}(t,X_t,\mu_{\theta}(t,X_t))$ and $D(t,X_t) = \sigma_{\xi}(t,X_t)\sigma_{\xi}(t,X_t)^{\tr}$ for $\sigma_{\xi}(t,X_t) = \sigma_{\xi}(t,X_t,\pi_{\theta}(t,X_t))$.
Taking the expectation, the diffusion term cancels since $\Ex[\dif W_t] = \mathbf{0}$. This can then be used to derive a policy gradient for the SDE.
\begin{theorem}[Deterministic policy gradient for SDE dynamics]
\label{th: deterministic policy gradient SDE}
Let the transition dynamics be the SDE \eqref{eq: SDE} and fix $\xi \in \Xi$. Then the deterministic policy gradient is given by
\begin{align}
\label{eq: policy gradient SDE}
    \nabla_\theta J(\theta,\xi) &= \Ex \Big[ \int_0^T \left[\nabla_\theta \mu_\theta(t,X_t) + \nabla_{x} \mu_{\theta}(t,X_t) Z_t \right]^\tr   \nabla_u \left( r(t, X_t, U_t) + \cA V(t,X_t) \right)\dif t  \Big\vert d_0, f_{\xi} \Big] \nonumber \\  
    &= \Ex \Big[ \int_0^T \left[\nabla_\theta \mu_\theta(t,X_t) + \nabla_{x} \mu_\theta(t,X_t) Z_t \right]^\tr   \nonumber \\
    &\hspace{3em}\nabla_u \left( r(t, X_t, U_t) + P_t^\tr  b_{\xi}(t,X_t) + Q_t^\tr \sigma_\xi(t,X_t) \right)\dif t  \Big\vert d_0, f_{\xi} \Big] \,.
\end{align}
\end{theorem}
\begin{proof}
This proof is exactly analogous to that of the ODE case, where the generator/Hamiltonian now has an extra term \(Q_t \sigma_\xi(t, X_t)\) to account for the effect of the diffusion. 
\end{proof}

Now we derive the adjoint-based adversary gradient, where $\xi$ parametrizes the drift and the diffusion. If $\xi$ only parametrizes the drift \(b_\xi\), it is easy to drop the trace term. To ascend in \(\xi\), one only needs a forward simulation of \((X_t, U_t)\), then solve the adjoint BSDE for \((P_t, Q_t)\), and subsequently integrate the expression below.
\begin{theorem}[Adjoint adversary gradient for SDE]
\label{th: adjoint adversary gradient SDE}
    If \(\xi\) parametrizes both the drift and diffusion as in \Cref{eq: SDE} and fixing \(\theta\), the adjoint adversary gradient is given by
\begin{align}
    \nabla_\xi J(\theta,\xi) &= \Ex_{u_t \sim \pi(t,x_t)} \Big[ \int_0^T \left[\nabla_\xi f_{\xi}(t,x_t,u_t) \right]^\tr \nonumber \\
    \hspace{3em}&\left( r(t, x_t, u_t) + (\nabla_x V(t,x_t))^\tr  b_{\xi}(t,x_t) + \frac{1}{2} \nabla_x^2 V(t,x_t) D(t,x_t)  \right) \dif t  \Big\vert d_0, f_{\xi} \Big]  \nonumber \\
    &= \Ex \left[ \int_0^T P_t^\tr \nabla_\xi b_\xi(t, X_t, U_t) + \Tr(Q_t^\tr \nabla_\xi \sigma_\xi(t, X_t, U_t)) \dif t\right] \,,
\end{align}
where \(\xi \in \Reals^{d_\xi}\), \(\nabla_\xi \sigma_\xi \in \Reals^{d_x \times d_m \times d_\xi}\).   
\end{theorem}
\begin{proof}
The proof for the above adversary gradient in the SDE dynamics setting follows similarly to that of the adversary gradient in the ODE case, with one notable change: the use of It\^{o}'s lemma to compute differentials. Additionally, for this proof, we abuse notation to denote \(b_\xi(t,X_t)\) and \(\sigma_\xi(t, X_t)\) for the closed-loop coefficients i.e. \(b_\xi(t, X_t, \pi_\theta(t, X_t))\) and \(\sigma_\xi(t, X_t, \pi_\theta(t, X_t))\), respectively. We fix $\theta$ and differentiate with respect to \(\xi\). 

\paragraph{Perturb the adversary parameter}: Let \(\xi^\epsilon = \xi + \epsilon \delta \xi\) where \(\delta \xi \in \Reals^{d_\xi}\) is an arbitrary direction. Let \(X^\epsilon_t\) be the trajectory obtain from using \(\xi^\epsilon\), i.e. \(X^\epsilon_t\) solves
\[\dif X^\epsilon_t = b_{\xi^\epsilon}(t, X^\epsilon_t) \dif t + \sigma_{\xi^\epsilon}(t, X^\epsilon_t), \quad X^\epsilon_0 = X_0.\]
Define the first variation \[Y_t \: \frac{\dif}{\dif \epsilon}X^\epsilon_t \bigg|_{\epsilon = 0}.\]
Linearising the original SDE, we get that $Y_t$ solves the linear variational SDE
\[\dif Y_t = \left( \nabla_x b_{\xi}(t, X_t) Y_t + \nabla_\xi b_\xi(t, X_t) \delta \xi \right)\dif t + \left( \nabla_x \sigma(t, X_t) Y_t + \nabla_\xi \sigma_\xi(t, X_t) \delta \xi \right)\dif W_t, \quad Y_0 = 0.\]
\paragraph{Linearising the objective}
Set \[J(\xi) = \Ex \left[ \int_0^T r(t, X_t) \dif t + R(X_T)\right].\]
Taking the directional derivative in \(\delta \xi\). 
\[DJ(\xi)[\delta \xi] = \frac{\dif}{\dif \epsilon} J(\xi^\epsilon) \bigg |_{\epsilon = 0} = \Ex \left[ \int_0^T \nabla_x r(t, X_t)^\tr Y_t \dif t + \nabla_x R(X_T)^\tr Y_T\right].\]
Similarly to the ODE case, now we introduce the adjoint BSDE to eliminate \(\delta \xi\).

\paragraph{The adjoint BSDE}
Let the costates \((P_t, Q_t)\) satisfy
\[\dif P_t = - \left( \nabla_x r(t, X_t) + \nabla_x b_\xi(t, X_t)^\tr P_t + \sum_{j=1}^m \underbrace{\nabla_x \sigma_\xi^{(j)}(t, X_t)}_{\in \Reals^{d \times d}} \underbrace{Q_t^{(j)}}_{\in \Reals^d}\right) \dif t + Q_t \dif W_t\]
where the superscript \((j)\) denotes the \(j\)-th column. 
\paragraph{It\^{o}'s rule for \(P^\tr Y_t\)}:
We compute \(\dif(P_t^\tr Y_t)\) using It\^{o}'s product rule: \(\dif(P_t^\tr Y_t) = P_t^\tr \dif Y_t + Y_t^\tr \dif P_t + \dif \langle P, Y \rangle_t\)
where the last term is the quadratic covariation. We substitute in the SDE for \(\dif Y_t\) in the first term, and the SDE for \(\dif P_t\) in the second term. 

The first term is
\[P_t^\tr \dif Y_t = P_t^\tr(\nabla_x b_\xi(t, X_t) Y_t + \nabla_\xi b_\xi(t, X_t) \delta \xi) \dif t + P_t^\tr\left( \nabla_x \sigma_\xi(t, X_t) Y_t + \nabla_\xi \sigma_\xi(t, X_t) \delta \xi\right) \dif W_t.\]

The second term is
\[Y_t^\tr \dif P_t = -Y_t^\tr \left( \nabla_x r(t, X_t) + (\nabla_x b_\xi(t, X_t))^\tr P_t + \sum_{j=1}^m \nabla_x \sigma_\xi^{(j)}(t, X_t)^\tr Q_t^{(j)}\right) \dif t + Y_t^\tr Q_t \dif W_t.\]

For the quadratic covariation term, only the diffusion parts contribute:
\[\dif \langle P, Y \rangle_t = \sum_{j=1}^m Q_t^{(j)\tr} \left( \nabla_x \sigma_\xi^{(j)}(t, X_t)Y_t + \nabla_\xi \sigma_\xi^{(j)}\right) \dif t.\]

The remaining drift, after some terms cancel, is
\begin{equation}\label{eq:drift_itorule}
   \dif (P_t^\tr Y_t) = \left(- \nabla_x r(t, X_t)^\tr Y_t + P_t^\tr \nabla_\xi b_\xi(t, X_t) \delta \xi + \sum_{j=1}^m \nabla_x \sigma_\xi^{(j)}(t, X_t)^\tr Q_t^{(j)} \right) \dif t + \dif M_t 
\end{equation}

where \(\dif M_t\) is a martingale. 

\paragraph{Integrate and take expectations}: 
Integrating \Cref{eq:drift_itorule} from 0 to $T$,
\[P_T^\tr Y_T - P_0^\tr Y_0 = \int_0^T \left(- \nabla_x r(t, X_t)^\tr Y_t + P_t^\tr \nabla_\xi b_\xi(t, X_t) \delta \xi + \sum_{j=1}^m \nabla_x \sigma_\xi^{(j)}(t, X_t)^\tr Q_t^{(j)}\right)\dif t + M_T. \]
Since $Y_0 = 0$, and noting that \(p_T = \nabla_x R(X_T)\) and \(\Ex[M_T] = 0\), we take expectations to get
\[\Ex [\nabla_x R(X_T)^\tr Y_T] = \Ex \int_0^T \left(- \nabla_x r(t, X_t)^\tr Y_t + P_t^\tr \nabla_\xi b_\xi(t, X_t) \delta \xi + \sum_{j=1}^m \nabla_x \sigma_\xi^{(j)}(t, X_t)^\tr Q_t^{(j)} \right) \dif t \]

Substituting this into the expression for \(D J(\xi)[\delta \xi]\) to get 
\[DJ(\xi)[\delta \xi] = \Ex \int_0^T \left( P_t^\tr \nabla_\xi b_\xi(t, X_t) \delta \xi +  \sum_{j=1}^m \nabla_x \sigma_\xi^{(j)}(t, X_t)^\tr Q_t^{(j)} \right) \dif t\]
As this holds for every perturbation \(\delta \xi\), we can conclude the adjoint formula above.
\end{proof}

\subsection{Discretisation analysis}
In this section, we are interested in deriving the Euler-based discretisation error of the policy gradient \(\nabla_\theta J(\theta, \xi)\) with a time step size of \(\Delta t \equiv h\) 
(the following is to distinguish Euler numerical error \(\neq\) sample-and-hold \(\rightarrow\) relaxed-policy error).
\paragraph{Numerical discretisation versus the sample-and-hold limit.}
We distinguish two uses of the time step $h$. First, for a fixed sample-and-hold policy, one may approximate the exact continuous evolution on each interval $[t_n, t_{n+1})$ by an Euler step. This is the numerical discretisation error studied in the present section. Second, in the stochastic policy construction of Theorem~\ref{th: stochastic adjoint gradient}, the holding interval itself tends to zero, and the resulting random sample-and-hold state process converges to the corresponding relaxed continuous-time dynamics. These are distinct approximation errors. In particular, the $O(h)$ bound below compares the exact sample-and-hold trajectory with its Euler approximation and should not be interpreted as a strong convergence rate of the sample-and-hold process to the relaxed dynamics.

\begin{lemma}[One-step state discretisation error under sample-and-hold control]
\label{lem: one-step state discretisation}
    Let the exact trajectory be denoted by \( \frac{d}{d t}\bar x(t) = f_\xi(t, \bar x(t), \bar u_n)\) for \(t \in [t_n, t_{n+1})\), \(\bar x (t_0) = x_0\) where \(\bar u_n \sim \pi_\theta (\cdot | t_n, \bar x (t_n))\) for a general stochastic policy $\pi_{\theta}$. Moreover, let \(f_\xi\) be Lipschitz in \(x\) with constant \(L_x\), in \(u\) with constant \(L_u\), and in time \(t\) with constant \(L_t\). Additionally, let the drift be bounded uniformly in \(t, x, \) and \(u\), i.e. \(\norm{f_\xi(t,x,u)} \leq M\), and let $\pi_{\theta}$ be Lipschitz in the state under the \(W_1\) (1-Wasserstein distance) metric:
\begin{equation*}
    W_1(\pi_\theta(\cdot | t,x), \pi_\theta(\cdot| t, y)) \leq L_\pi \norm{x - y}
\end{equation*} for all \(t, x, y\). Then it follows that an Euler-based discretisation with time step $h$ has an error upper bounded by
    \begin{equation*}
        \Ex\norm{\bar x(t_n) - x_n} \leq \underbrace{\frac{e^{L_{\rm cl}(T-t_0)}-1}{2 L_{\rm cl}}(L_xM + L_t)}_{C_x(T)} h = \cO(h)
    \end{equation*}
\end{lemma}

\begin{proof}
The proof is given in Appendix~\ref{app: one-step state discretisation}.
\end{proof}

Using the above, it is possible to show that the policy and adversary gradients also have a limited discretisation error as shown in the corollary below.
\begin{corollary}[discretisation error of policy and adversary gradients]
Under the preconditions of Lemma~\ref{lem: one-step state discretisation}, and the standard Lipschitz assumptions in our general assumptions, the policy and adversary gradients have a discretisation error bounded by  $\cO(h)$.  
\end{corollary}
\begin{corollary}[Euler discretisation error of gradient functionals]
Under the assumptions of Lemma~1, suppose additionally that the
continuous-time gradient under consideration admits a representation
\[
\nabla J
=
\Ex\left[
\int_0^T g(t,x_t,u_t)\,dt + G(x_T)
\right],
\]
where $g$ and $G$ satisfy the corresponding uniform Lipschitz and
boundedness conditions. Then the Euler approximation of this gradient,
for the same sample-and-hold control process, has global discretisation
error $\cO(h)$.
\end{corollary}
\begin{proof}
The policy and adversary gradients compute an expression of the form 
\begin{align*}
\nabla J(\theta,\xi) = \Ex \left[ \int_{0}^{T} g(t,x_t,u_t) \dif t + G(x_T) \right] \,,
\end{align*}
for some functionals $g: \Reals \times \cX \times \cU \to \Reals$, $G: \cX \to \Reals$. 

Adding and subtracting $g(t_n, \bar x(t_n), \bar u_n)$ shows that the integral discretisation error comes from the quadrature error and the grid error: 
\begin{align}
    \int_{t_n}^{t_{n+1}} g(t, \bar x(t), \bar u_n) \dif t - h g(t_n, x_n, u_n) &= \underbrace{\int_{t_n}^{t_{n+1}} g(t, \bar x(t), \bar u_n) - g(t_n, \bar x(t_n), \bar u_n) \dif t}_{\text{quadrature error}} \\
    & + \underbrace{h(g(t_n, \bar x(t_n), \bar u_n) - g(t_n, x_n, u_n))}_{\text{grid error}} \,.
\end{align}

Following the preconditions, $g$ is Lipschitz in state and time with constants $L_{g,x}$ and $L_{g,t}$. Therefore, the quadrature error gives
\begin{equation*}
    \left | \int_{t_n}^{t_{n+1}} g(t, \bar x(t), \bar u_n) - g(t_n, \bar x(t_n), \bar u_n)\right| \leq \frac{1}{2}(L_{g,x}M + L_{g,t})h^2 \,.
\end{equation*}
The grid error can be obtained based on a similar coupling argument as in \Cref{eq: coupling}, \begin{equation*}
    \Ex[\text{grid error}] \leq (L_{g,x} + L_{g,u}L_\pi)\Ex\norm{\bar x(t_n) - x_n}.
\end{equation*}
For the final bound, we sum the above over \(n\), using \(N = T/h\) and use the state discretisation error \eqref{eq: state error} to get the first-order error
\begin{equation*}
    |J - J_h| \leq h \left[ \frac{T}{2}(L_{g,x}M + L_{g,t}) + C_x(T) (T(L_{g,x} + L_{g,u}L_\pi) + L_G)\right] = \cO(h) \,.
\end{equation*}
\end{proof}

We can analogously bound the backward process' discretisation error by following a similar procedure as above (i.e. the discretisation for \(p_t\)), and then adding them. This leads to an overall first-order discretisation error \(\cO(h)\). 

Note that the overall order is limited by the \textbf{state discretisation error} (i.e. improvement can occur when one uses a higher-order integrator, such as Runge-Kutta 2 or 4). 
\begin{remark}[Relation to Theorem~\ref{th: stochastic adjoint gradient}]
The preceding corollary concerns numerical Euler approximation of a
fixed sample-and-hold gradient functional. For the stochastic policy
gradient of Theorem~\ref{th: stochastic adjoint gradient} there is an additional limiting step: as the
holding interval $h\to0$, the sample-and-hold policy converges to the
relaxed continuous-time policy. Under the assumptions of Theorem~\ref{th: stochastic adjoint gradient},
if $X^h$ denotes the exact sample-and-hold state and $\bar X$ the
relaxed state, then
\[
\Ex\!\left[\sup_{0\leq t\leq T}
\|X_t^h-\bar X_t\|^2\right]
\leq C h,
\]
so the corresponding strong root-mean-square state error is
$\cO(h^{1/2})$. This averaging error is distinct from the $\cO(h)$ Euler
error of Lemma~\ref{lem: one-step state discretisation}. The convergence of the stochastic policy gradient
itself is established in Theorem~\ref{th: stochastic adjoint gradient} through the exact discrete
likelihood-ratio identity and the dynamic-programming expansion of
$Q_h\theta-V_h^{\pi_\theta}$.
\end{remark}

\section{Regret analysis}
Having defined various policy gradient and adversary gradients for CT-RMDPs, we now turn to a regret analysis for two selected algorithms compatible with CT-RMDPs. Both algorithms alternate policy and adversary gradients to provide a robust optimisation to approximate solve \eqref{eq: RMDP}. To ensure the adversary's parameters remain within the uncertainty set at iteration $k \in [K]$, the adversary is updated via projected gradient ascent, i.e.
\begin{equation}
\xi_k \gets \text{proj}_{\Xi} \left( \xi_{k-1} + \eta \nabla_\xi J(\xi_{k-1})  \right) \,.
\end{equation}

The first analysis is based on a double-loop algorithm, where the inner problem is assumed to be solved $\epsilon_k$-optimally at each iteration $k \in [K]$. The double-loop analysis leads to a linear convergence rate in the oracle-based setting and $\tilde{\cO}(\frac{1}{\eps})$ and $\tilde{\cO}(\frac{1}{\eps^2})$ sample complexity in the sample-based setting under last-iterate and average regret, where $\eps$ is the specified error tolerance. The second analysis is based on a mean-field optimisation algorithm, which is based on a distributional min-max optimisation. It leads to an $\tilde{\cO}(\frac{1}{K})$ convergence rate in the oracle-based setting and an $\tilde{\cO}(\frac{N^2}{\eps})$ sample complexity in the sample-based setting. While this result may in some cases correspond to a weaker rate, it allows a parallel particle based optimisation and allows distributional optimisation.

\subsection{Rectangularity}
Rectangularity is one of the main conditions for efficient solutions to RMDPs. Traditional RMDP analysis is based on directly parametrised uncertainty sets, where the notion of $(x,u)$-rectangularity is given by Definition~\ref{def: rectangularity}.
\begin{definition}[Rectangularity]
\label{def: rectangularity}
An uncertainty set is $(x,u)$-rectangular if it can be represented as a Cartesian product of independent sets
for each state and action, i.e. $\cP = \bigotimes_{(x,u) \in \cX \times \cU} \cP_{x,u}$, where $\cP_{x,u} \subseteq \bP[\cX]$.
\end{definition}
To illustrate the meaning, consider a distance measure $D$. A typical example is the uncertainty set defined by a distance metric according to $\cP = \otimes_{(x,u) \in \cX \times \cU} \{  \Pr \in \bP[\cX]^{\cX \times \cU}: D(\Pr(\cdot \vert x,u),\bar{\Pr}(\cdot \vert x,u)) \leq \psi(x,u) \}$, where $\psi(x,u)$ is the budget for a particular $(x,u)$. A non-rectangular set example is $\cP = \{   \Pr \in \bP[\cX]^{\cX \times \cU}: D(p,\bar{\Pr}) \leq \psi \}$ for some fixed budget $\psi$.

Below we demonstrate that the inner problem optimisation problem, which typically assumes the rectangularity property, can be feasible for indirect parametrisations. In general though, uncertainty sets may be non-rectangular, with the efficiency of the adversarial optimiser being related to the distance from the rectangular approximation. Due to the additional time dependency and the need to construct more statistically efficient (i.e. more narrow, less conservative) indirectly parametrised uncertainty sets, it is still of interest to solve non-rectangular uncertainty sets. In such cases, the efficiency of the uncertainty set is traded off with a worse regret bound. To modify our regret analysis to account for non-rectangular uncertainty sets, an additional factor similar to that introduced in \cite{Li2026} may need to be introduced, which depends on the non-rectangularity degree (i.e. the distance to the closest rectangular uncertainty set). 

\paragraph{Discretised setting}
We first analyse rectangularity in the setting of discrete state and time, which aligns with the total cost MDP analysis (Section~\ref{sec: double-loop}), in which we distinguish between transient states and recurrent states and maintain the generality of RL without requiring a special-case control theory setting. In our parametrisation, the system dynamics are time-, state-, and action-dependent. When the dynamics are discretised with time step $\Delta t$, we can write any SDE dynamics model (and therefore also any ODE dynamics model as a special case with zero variance) as some transition kernel
\begin{align*}
\Pr_{\xi}(\cdot \vert t_{n}, x_{n},u_{n}) = Z \cN(x_n + \Delta t b_{\xi}(t_n,x_n,u_n),\Delta t \sigma_{\xi}(t_n,x_n,u_n)  \sigma_{\xi}(t_n,x_n,u_n)^{\intercal}) \,,
\end{align*}
where $Z > 0$ is a normalisation constant. 

If the time-dependency is omitted, one obtains the Markov transition kernel of a traditional MDP
\begin{equation}
\label{eq: equivalent transition kernel}
\Pr_{\xi}(\cdot \vert x_{n},u_{n}) = Z \cN(x_n + \Delta t b_{\xi}(x_n,u_n),\Delta t \sigma_{\xi}(x_n,u_n)\sigma_{\xi}(t_n,x_n,u_n)^{\intercal}) ) \,.
\end{equation}
Similarly, our adjoint-based gradient formulations demonstrate equivalence with the discrete policy and adversarial policy gradient as in Theorem~\ref{th: stochastic adjoint gradient}.

With the above discretisation argument, it is possible to cast the problem in terms of a traditional $(x,u)$-rectangular parametrisation with the uncertainty set being comprised of the Cartesian product of all uncertainty sets with state-action dependent budgets. We identify a set of indirectly parametrised Gaussian transition kernels that are consistent with an SDE with linear dynamics and which satisfy the rectangularity property as in Definition~\ref{def: rectangularity}. For simplicity, we investigate the case where $\xi$ parametrizes only the drift although the variance may still be state-action dependent. Note that in our continuous-time dynamics parametrisation \eqref{eq: uncertainty set parametrisation}, the deviation from the nominal is given by $g_{\xi}(x,u)$. For several classes of functions, e.g. linear functions and many classes of MLPs, setting $\xi=\mathbf{0}$ will result in $g_{\xi}(x,u) = \mathbf{0}$. Formalizing the resulting transition dynamics in discrete-time, the problem is equivalent to shifting the mean of the normal distribution. If we project the state-actions to a high-dimensional feature space, it is possible to achieve state-action rectangularity as shown below. 
\begin{lemma}[A rectangular indirect parametrisation for linear SDEs]
\label{lem: rectangularity}
Let $\cZ = \{(x_n,u_n)\}_{n=1}^{N} \subset \cX \times \cU$ be a grid discretising the state-action space. Moreover, let $\phi: \cX \times \cU \to \Reals^d$ be a fixed feature-map, where $d \geq N$, such that $\cB = \{\phi(x,u)\}_{(x,u) \in \cX \times \cU} \subset \Reals^d$ are linearly independent. Then an uncertainty set $\cP_{\Xi}$ of Gaussian transition kernels can be derived from the drift parametrisation 
\begin{equation}
\label{eq: rectangular indirect parametrisation}
b_{\xi}(x,u) = \phi(x,u)^{\tr} \xi   \,,  \xi \in \Xi \subset \Reals^{d \times d_x}  \,,
\end{equation}
for a linear SDE such that it is a well-defined $(x,u)$-rectangular uncertainty set with respect to $Z$. If further $\Xi$ is a dense set with drifts spanning the sets $B_{x,u} = \{ b_{\xi}(x,u) \in \Reals^{d_x}: \xi \in \Xi \} =  \{ b \in \Reals^{d_x}: \norm{\Delta t b} \leq \rho(x,u)\}$ for all $(x,u) \in \cZ$ and $\rho(x,u) \in \Reals$, then this parametrisation satisfies $\cP_{\Xi} = \cP$ for some $L_1$ uncertainty set $\cP$.
\end{lemma}
\begin{proof}
The proof is given in Appendix~\ref{app: rectangularity}.
\end{proof}

\paragraph{Continuous setting}
While the above discretisation argument facilitates the analysis in Section~\ref{sec: double-loop}, the mean-field optimisation analysis in Section~\ref{sec: mean-field optimisation} is more general in the sense that it can be applied to the full CT-RMDP with continuous states and continuous-time dynamics. We expect that rectangularity conditions may be required to meet some of the preconditions of the analysis, so we present here a generalisation to continuous states and continuous-time. 

To deal with continuous states in RMDPs, recent work \citep{li2026robustmarkovdecisionprocesses} has designed an indirectly parametrised $x$-rectangular uncertainty set for continuous state spaces. The uncertainty set is based on a mixture of parametrised transition kernels and uses the $\phi$-divergence as a distance measure:
\begin{equation}
\label{eq: phi-divergence}
\cP_x = \left\{ \sum_i w_i \Pr_{\xi_i}(\cdot \vert x, \cdot) : D_{\phi}(w,\bar{w}) \leq \psi(x) \,, \sum_i w_i = 1  \right\} \,,
\end{equation}
where $\{\xi_i\}_i$ is a set of random samples or a grid over $\Xi$, $D_{\phi}$ is a $\phi$-divergence, $N$ is the number of components of the mixture, and $\bar{w}$ is the nominal weight vector. The exact parametrisation of $\Pr_{\xi_i}$ may vary.

Applying \eqref{eq: phi-divergence} to the standard SDE setting, the weights are set as one-hot vectors. Under this uncertainty set, one form is the linear SDE, which gives transition kernels of the form $\Pr_{\xi_i}(\cdot \vert x, u) = \cN(x + \Delta t b_{\xi_i}(x,u),\Delta t \sigma_{\xi_i}
(x,u)\sigma_{\xi}(t_n,x_n,u_n)^{\intercal})$, i.e. a Gaussian density for state $x \in \cX$ and action $u \in \cU$. The nominal model will be a mixture model, i.e. a Gaussian mixture for linear SDEs, while the models in the uncertainty set are Gaussians (with only one component). 

To obtain the state dynamics in continuous-time, note that the time-homogeneous Fokker-Planck equation gives the evolution of the state density $p_{\xi}$ over time,
\begin{equation}
\frac{\dif}{\dif t} p_{\xi}(x_t) = - \nabla_x \left( b_{\xi}(x_t,u_t) p_{\xi}(x_t) \right) + \nabla_x^2  \left(\frac{1}{2} \sigma_{\xi}(x_t,u_t) \sigma_{\xi}(x_t,u_t)^{\intercal} p_{\xi}(x_t) \right) \,.
\end{equation}
This leads to an alternative, continuous-time characterisation of the uncertainty set for linear SDEs as
\begin{equation}
\label{eq: phi-divergence SDE}
\cP_x = \left\{ \sum_i w_i \frac{\dif}{\dif t} p_{\xi_i}(x_t) : D_{\phi}(w,\bar{w}) \leq \psi(x) \,, \sum_i w_i = 1 \,, w_i \in \{0,1\} \right\} \,,
\end{equation}
where it is emphasised that $w$ is a one-hot vector but the nominal weight vector $\bar{w}$ is any vector with sum 1 such that the nominal model is a mixture SDE. The downside of the formulation \ref{eq: phi-divergence SDE} is that the nominal model itself is typically not in the uncertainty set. 

For mixture SDEs, the restriction to one-hot vectors can be removed,
\begin{equation}
\label{eq: phi-divergence mixture SDE}
\cP_x = \left\{ \sum_i w_i \frac{\dif}{\dif t} p_{\xi_i}(x_t) : D_{\phi}(w,\bar{w}) \leq \psi(x) \,, \sum_i w_i = 1  \right\} \,.
\end{equation}
However, this is no longer a standard SDE, and so the policy and adversary gradient formulations of Section~\ref{sec: extension to SDEs} do not apply, and this is beyond the scope of the paper.

\subsection{Double-loop analysis}
\label{sec: double-loop}
The first analysis uses a double-loop robust policy gradient algorithm (Algorithm~\ref{alg: double-loop}). The algorithm performs basic gradient descent in the policy space and projected gradient ascent in the space of transition dynamics. In this setting, we first prove the Polyak-{\L}ojasiewicz (PL) assumption (see 
Assumption~\ref{ass: PL}). We then show the linear oracle-based convergence rate and 
our analysis obtains $\tilde{\cO}(\frac{1}{\eps^2})$ (Theorem~\ref{th: sample-based regret double-loop}) and $\tilde{\cO}(\frac{1}{\eps^3})$ (Theorem~\ref{th: average regret sample complexity}) sample complexities in the last-iterate and average regret analyses, respectively. We provide the result for a general parametrisation.

Previously, prior works have established PL conditions with respect to the robust policy gradient, in the context of directly parametrised discounted $x$-rectangular RMDPs \citep{Kumar2023}, average cost $(x,u)$-rectangular RMDPs  \citep{Sun2024a}, and discounted non-rectangular  RMDPs \citep{Li2026}. We provide results for a general stochastic policy parametrisation under PL condition assumptions within an undiscounted total cost RMDP framework. Alternative proofs under the softmax and direct parametrisations can be found in the double-loop mirror descent ascent scheme considered in prior works \citep{Wang2024,bossens2025mirror}. The analysis is based on the discrete-time policy gradient of our algorithm, in which case each continuous-time MDP in the uncertainty set is equivalent to an undiscounted total cost MDP for $\Delta t \to 0$; future work can consider deriving the convergence rates for the large time step case, where the continuous-time policy gradient will behave differently compared to the discrete-time policy gradient.

In the oracle-based setting, the linear convergence can be obtained for the last-iterate regret of transition mirror ascent in discounted MDPs \citep{Wang2024,bossens2025mirror} and the last-iterate robust regret \citep{Wang2024}. While these techniques rely on direct parametrisations and require properties such as rectangularity and robust PL condition, below we focus on projected gradient ascent under general parametrisation on the undiscounted total cost MDP. This analysis relies on different PL assumptions (which may or may not partially overlap with rectangularity properties), our general Lipschitz and smoothness assumptions, and a geometrically shrinking mismatch of transition kernels.

\begin{algorithm}
\caption{Double-loop robust policy gradient for RMDPs.} \label{alg: double-loop}
\begin{algorithmic}[1]
\State \textbf{Inputs:} initial parameters $\theta_0$,$\xi_0$, learning rate schedules $\eta_1,\eta_2$
\State \textbf{Outputs:} near-optimal parameters $\theta_K$,$\xi_K$
 \For{$k = 1,\dots,K$}
\State Sample trajectories using $\pi_{\theta_{k-1}}$ and $f_{\xi_{k-1}}$
\State Estimate $\nabla_{\theta} J(\theta_{k-1},\xi_{k-1})$ 
\State $\theta_k \gets \theta_{k-1} - \eta_1(k-1) \nabla_{\theta} J(\theta_{k-1},\xi_{k-1})$
 \State $t \gets 0$
\While {$J(\theta_{k},\xi_{k-1}^t) + \epsilon_k < \max_{\xi \in \Xi} J(\theta_{k},\xi)$}
\State Sample trajectories using $\pi_{\theta_{k}}$ and $f_{\xi_{k-1}^t}$
\State Estimate $\nabla_{\xi} J(\theta_{k},\xi_{k-1}^t)$ 
\State $\xi_{k-1}^{t+1} \gets \text{proj}_{\Xi} \left(\xi_{k-1}^t + \eta_2(k-1,t) \nabla_{\xi} J(\theta_{k},\xi_{k-1}^t)\right)$
\State $t \gets t + 1$.
\EndWhile
\State $\xi_k \gets \xi_{k-1}^{t}$.
 \EndFor
\end{algorithmic}
\end{algorithm}

\paragraph{PL condition analysis}
For PL coefficients in traditional MDPs under softmax policies, previous work established a PL coefficient under a sufficient exploration condition \citep{Mei2020}. Our regret analysis is based on posing PL coefficients for the adversary and the policy. Our theorems do not require PL coefficients to be strictly positive at all iterations. However, the below analyses specify sufficient conditions for nonnegative PL coefficients. The analysis derives a squared norm based PL coefficient as opposed to the common non-squared norm based PL coefficient in the RMDP literature.
\begin{assumption}[Polyak-{\L}ojasiewicz (PL) condition]
\label{ass: PL}
Let $J: \Theta \to \Reals^+$ be a cost function and let $\mu > 0$, with optimal parameters $\theta^*$. Then $J$ satisfies the PL condition if 
\begin{equation}\label{eq:PL_condn}
    \frac{1}{2}\norm{\nabla_\theta J(\theta)}^2 \geq \mu (J(\theta) - J(\theta^*)) \,.
\end{equation}
for all $\theta \in \Theta$.
\end{assumption}
We will make the assumption for both the policy optimisation and the adversarial optimisation, and we provide a lemma for both to support the assumption.

We formulate the analysis within the framework of undiscounted total cost MDPs, which can be seen as a discretised form of our continuous-time formulation. The transient visitation measure is a key concept within this framework which facilitates our analysis. Rather than working with the measure on all states, it is convenient to restrict the measure to the non-zero entries by only considering the set of transient states.
\begin{definition}[Transient visitation measures]
\label{def: visitation measure}
For an undiscounted total cost MDP $\cM = (s_0,\Pr,r,\cX,\cU)$, where $s_0$ is the starting state, $p$ is the transition kernel, $r$ is the cost function, and $\cX$ and $\cU$ are the state and action spaces, define the \textbf{transient visitation measure} \citep{Lee2026} as $\tilde{d}_{s_0,\Pr}^{\pi}(s) = \sum_{t=0}^{\infty} \Pr(s_t = s \vert s_0, \cT^{\pi,\Pr}) = e_{s_0}^{\intercal} (\mathbb{I} - \cT^{\pi,\Pr})^{-1} e_{s}$, where $\cT^{\pi,\Pr}$ is the transient matrix
\begin{align*}
\cT^{\pi,\Pr}(s, s')  = \begin{cases}
P^{\pi,\Pr}(s,s') \quad \text{if $s, s' \in \cX$ are transient} \\
0 \quad \text{otherwise} \,,
\end{cases}
\end{align*}
based on the induced transition probability $P^{\pi,\Pr}$ under policy $\pi$ and transition dynamic $\Pr$, with $e_{j}$ denoting the $j$'th unit vector. Similarly, we define the \textbf{restricted transient visitation measure} $d_{s_0,\Pr}^{\pi}(s)$ as the transient visitation measure when restricted to the set of transient states. 
\end{definition}

While we previously worked with time-dependent policies and dynamics, the analysis here uses the reduction $x_n = (\tilde{x}_n, t_n)$, where $\tilde{x}_n$ denotes the state in the original, time-dependent formulation. Due to the time index, these augmented states are naturally transient. The set of recurrent states is based on defining any state with $t > T$ as an absorbing state (i.e. the agent will stay in the state indefinitely) and consistent with the transient performance difference lemma (Lemma~\ref{lem: performance difference}), the recurrent states all result in a zero value.

With the above formulation, we first derive PL conditions for the adversary and the policy in the undiscounted total cost MDP framework.
Due to applying projected gradient ascent, the PL condition is slightly modified. Related to the LPL condition \citep{Balashov2019}, we derive a modified PL condition for projected gradient ascent (i.e. for the adversary).
\begin{assumption}[Modified PL condition]
\label{ass: MPL}
Let $J: \Xi \to \Reals$ be a utility function with optimal parameters $\xi^*$. Then $J$ satisfies the Modified PL condition if 
\begin{equation}\label{eq: MPL condition}
    \frac{1}{2}\norm{h_{\eta}(\xi)}^2 \geq \mu (J(\xi^*) - J(\xi)) \,.
\end{equation}
for all $\xi \in \Xi$ for some $\mu \in \Reals$, where $h_{\eta}(\xi) = \frac{1}{\eta} \left( \text{proj}_{\Xi}\left( \xi + \eta \nabla_\theta J(\xi)   \right)  - \xi \right)$ is the projected gradient mapping.
\end{assumption}
Where appropriate, we analogously consider projected gradient \textit{descent} for the policy and replace $\Xi$ with $\Theta$.

The below lemma proves the condition for a directly parametrised uncertainty set under an additional condition that the projected gradient should be a non-zero fraction of the unprojected gradient within the interior. This can be achieved using a soft projection to a point within the interior that is close to the boundary rather than exactly at the boundary, e.g. using barrier functions or by applying conservative policy updates near the boundary. Note that a similar technique could be applied to obtain a projected policy gradient. Similarly, the traditional PL condition is enough for particular uncertainty set definitions as a limited uncertainty set can be obtained under unbounded parameters as well (see also Remark~\ref{rem: uncertainty set parametrisation}).
\begin{lemma}[Modified PL condition for directly parametrised adversary]
\label{lem: PL condition adversary}
Let $J(\theta): \Xi \to \Reals$ be a utility function with optimal parameters $\xi^*=p^*$ for an undiscounted total cost RMDP $\cM = (s_0,\cP,r,\cX^+,\cU)$ with directly parametrised uncertainty set $\cP = \Xi$ with recurrent states $\cX_{\text{rec}}$ and transient states  $\cX = \cX^+ \setminus \cX_{\text{rec}}$. Let $M  = \max_{\pi,\pi'} \norm{\frac{d_{d_0,\Pr}^{\pi}}{d_{d_0,\Pr}^{\pi'}}}_{\infty}$ and $0 < D_{\text{min}} \leq \min_{\pi} d_{d_0,\Pr}^{\pi}(x')$ for all $x \in \cX$, where $\pi$ is in the policy class parametrised by $\Theta$. Denote the state-value as $V_h(x)$ for $x \in \cX \setminus \cX_{\text{rec}}$ and note that $\cX_{\text{rec}}$ are absorbing states with $V_h(x)=0$ for $x \in \cX_{\text{rec}}$. Moreover, define a projected gradient mapping which satisfies
\begin{align}
\label{eq: condition projected grad}
 \norm{h_{\eta}(\xi)} \geq \beta \norm{\nabla_\xi J(\theta,\xi)}   
\end{align} 
for $\beta \in (0,1]$ for all $\xi \in \text{Int}(\Xi)$. Last, let $B_{\text{min}}(\Xi) \leq \norm{\nabla_{\xi} J(\theta,\xi)}$ and define $\cP$ such that there exists a $\nu > 0$ such that for any $\Pr,\Pr' \in \cP$ and any $(x,u,x') \in \cX\times\cU\times\cX$, $\Pr(x' \vert x,u) \leq (1+\nu) \Pr'(x' \vert x,u)$. Then the modified PL condition holds for $\text{Int}(\Xi)$ with parameter $\mu = c^2/2$, where  $c = \sqrt{\frac{B_{\text{min}}(\Xi)}{M(1+\nu)}} \beta$.
\end{lemma}
\begin{proof}
The proof is given in Appendix~\ref{app: PL condition adversary}.
\end{proof}

The below lemma formulates the conditions for the PL coefficients of the policy updates.
\begin{lemma}[PL condition for the policy updates under direct and indirect parametrisations]
\label{lem: PL condition policy}
Let  $\cM = (d_0,\Pr,r,\cX^+,\cU)$ be an undiscounted total cost MDP and let $J: \Theta \to \Reals^+$ be the total cost function with recurrent states $\cX_{\text{rec}}$ and transient states  $\cX = \cX^+ \setminus \cX_{\text{rec}}$. Let $M  = \max_{\pi,\pi'} \norm{\frac{d_{d_0,\Pr}^{\pi}}{d_{d_0,\Pr}^{\pi'}}}_{\infty}$ and $0 < \partial_{\text{min}} \leq \min_{\theta_i} \partial_{\theta_i} J(\theta)$, where  $\theta \in \Theta$ parametrises the policy. Define a projected gradient mapping which satisfies
\begin{align}
\label{eq: condition projected grad policy}
 \norm{h_{\eta}(\theta)} \geq \beta' \norm{\nabla_\theta J(\theta,\xi)}   
\end{align} 
for $\beta' \in (0,1]$ for all $\theta \in \text{Int}(\Theta)$. Denote the discretised action-value and state-value as $Q_h(x,u)$ and $V_h(x)$ for $x \in \cX$ and $u \in \cU$. Then 
the PL condition is satisfied with coefficient $\mu = c^2 / 2$, where \\
\textbf{a)} under the direct parametrisation, the modified PL condition (analogous to Assumption~\ref{ass: MPL}) is satisfied with 
\begin{align*}
c = \beta \left(\frac{\partial_{\text{min}}}{M}\right)^{1/2}    \,,
\end{align*}
where $\partial_{\text{min}} > 0$ lower bounds the partial derivatives. \\
\textbf{b)}  under an indirect parametrisation with gradient dominance condition $\norm{\Ex[Q_h(x,u)\nabla_{\theta} \log(\pi_{\theta}(u \vert x))]} \geq \alpha \Ex[Q_h(x,u)]$ for $\alpha > 0$ for all $x \in \cX$, 
\begin{align*}
c = M^{-1/2} \alpha \kappa^{-1/2} B_{\text{min}}(\Theta)^{1/2} 
\end{align*}
for unprojected gradient descent and 
\begin{align*}
c =  \beta M^{-1/2} \alpha \kappa^{-1/2}  B_{\text{min}}(\Theta)^{1/2} 
\end{align*}
for projected gradient descent, where  $\kappa = \frac{\max_{x'  \in  \text{supp}(d_{d_0,\Pr_{\xi}}^{\pi_{\theta}})}  d_{d_0,\Pr_{\xi}}^{\pi_{\theta}}(x')}{\min_{x''  \in  \text{supp}(d_{d_0,\Pr_{\xi}}^{\pi_{\theta}})}  d_{d_0,\Pr_{\xi}}^{\pi_{\theta}}(x'')}$,  and $B_{\text{min}}(\Theta) \leq \norm{\nabla_{\theta} J(\theta,\xi)}$.
\end{lemma}
\begin{proof}
The proof is given in Appendix~\ref{app: PL condition policy}.
\end{proof}

The two examples below are used to illustrate the feasibility of the preconditions w.r.t. the gradient norm and the gradient dominance in Lemma~\ref{lem: PL condition policy}.  \\
Part a) holds under an assumption on the reward density of the MDP. 
\begin{example}[Gradient norm under direct policy parametrisation]
Define $G_{\theta}(x):= d_{d_0}^{\pi_{\theta},\Pr}(x) \Ex[Q_h(x,u)]$. Suppose that there exists an $x^+ \in \cX$ for which $G_{\theta}(x^+) \geq G^+ > 0$ for every $\theta \in \Theta$. Then for every $\theta \in \text{Int}(\Theta)$, $\norm{\nabla J(\theta)} \geq B_{\text{min}}(\Theta) := G^+$.
\end{example}
\begin{proof}
\begin{align*}
\norm{\nabla J(\theta)} &= \norm{\sum_{x} d_{d_0}^{\pi_{\theta},\Pr} \sum_u  \pi_{\theta}( u \vert x) Q_h(x,u)} \\
                        &\geq G^+ =:  B_{\text{min}}(\Theta)  \,.
\end{align*}
\end{proof}

Part b) follows under a similar reward density assumption for particular policy classes. The gradient dominance condition  and the $B_{\text{min}}$ bound in Lemma~\ref{lem: PL condition policy} b) are exemplified below with a sigmoid policy.
\begin{example}[Gradient norm and gradient dominance condition under indirect parametrisation]
Let $\cX = [a,b]$ and let $\Pi$ be the class of sigmoid policies of the form  $\pi_{\theta}: \cX \to \cU \quad ; \pi_{\theta}(u \vert x) = \text{sig}(x \theta_1 + \theta_2)$ where $\text{sig}(x) = \frac{1}{1 + \exp(-x)}$ and $\theta \in \Theta$. Further, define $G_{\theta}(x):= d_{d_0}^{\pi_{\theta},\Pr}(x) \Ex[Q_h(x,u)]$. Suppose that for every $\pi_{\theta} \in \Pi$ there exists an $x^+ \in \cX$ for which $G_{\theta}(x^+) \geq G(x^+) > 0$. Then $\pi_{\theta}$ satisfies the gradient dominance condition $\alpha \Ex[Q_h(x,u)] \leq \norm{\Ex[Q_h(x,u)\nabla_{\theta} \log(\pi_{\theta}(u \vert x))]}$ and the gradient norm is lower bounded by $\norm{\nabla J(\theta)} \geq B_{\text{min}}(\Theta) :=\alpha \norm{G(x^+)}$.
\end{example}
\begin{proof}
The partial derivatives are given by
\begin{align*}
&\partial_1 \pi_{\theta}(u \vert x) = x \text{sig}(x\theta_1 + \theta_2) ( 1 - \text{sig}(x\theta_1 + \theta_2) ) \\
&\partial_2  \pi_{\theta}(u \vert x)  = \text{sig}(x\theta_1 + \theta_2) (1 - \text{sig}(x\theta_1 + \theta_2) ) \,.
\end{align*}
It follows that 
\begin{align*}
\sum_{i} \vert \partial_{i} \pi_{\theta}(u \vert x) \vert > 0 \,.
\end{align*}
and therefore that for some $\alpha > 0$,
\begin{align*}
 \sum_{i} \vert \partial_{i} \log(\pi_{\theta}(u \vert x)) \vert  = \sum_{i} \frac{1}{\pi_{\theta}(u \vert x)} \vert \partial_{i}  \pi_{\theta}(u \vert x) \vert \geq \alpha.
\end{align*}
Consequently,
\begin{align*}
 \alpha \Ex[Q_h(x,u)]    &= \alpha \sum_{u} \pi_{\theta}(u \vert x)  Q_h(x,u) \\
                         &\leq \sum_{u} \pi_{\theta}(u \vert x) Q_h(x,u)   \sum_{i} \vert \partial_{i} \log(\pi_{\theta}(u \vert x)) \vert  \\
                         &= \sum_{i} \sum_{u} \pi_{\theta}(u \vert x) Q_h(x,u) \vert  \partial_{i}\log(\pi_{\theta}(u \vert x)) \vert  \\
                        &= \norm{\sum_{u} \pi_{\theta}(u \vert x) Q_h(x,u) \partial_{i}\log(\pi_{\theta}(u \vert x)) }_{1} \tag{non-negative Q} \\
                        &\leq \norm{\sum_{u} \pi_{\theta}(u \vert x) Q_h(x,u)\nabla_{\theta} \log(\pi_{\theta}(u \vert x))} \\
                         &= \norm{\Ex[Q_h(x,u)\nabla_{\theta} \log(\pi_{\theta}(u \vert x))]} \,.
\end{align*}
Further, from the above, the non-negativity of costs, and an $x^+ \in \cX$ for which $G(x)>0$, it follows that for some $C_5 > 0$,
\begin{align*}
\norm{\nabla J(\theta)} &= \norm{\sum_x d_{d_0}^{\pi,\Pr}(x) \Ex[Q_h(x,u)\nabla_{\theta} \log(\pi_{\theta}(u \vert x))]} \\
                        &\geq \norm{\sum_x d_{d_0}^{\pi,\Pr}(x) \alpha \Ex[Q_h(x,u)]} \\
                        &= \norm{\sum_x \alpha G(x)} \\
                        &\geq \alpha G^+ = B_{\text{min}}(\Theta) \,.
\end{align*}
\end{proof}

For $B_{\text{min}}(\Xi)$ in Lemma~\ref{lem: PL condition adversary}, a similar assumption of reward density can be made.
\begin{example}[Gradient norm under direct adversary parametrisation]
Define $G_{\Pr}(x):= d_{d_0,\Pr}^{\pi}(x) \sum_{u \in \cU} \pi(u \vert x)  \sum_{x'} V_h(x' \vert \Pr),  \Pr(x' \vert x,u)$. Suppose that there exists an $x^+ \in \cX$ for which $G_{\xi}(x^+) \geq G^+ > 0$ for every $\xi\in \Xi$. Then for every $\xi \in \text{Int}(\Xi)$, $\norm{\nabla J(\xi)} \geq B_{\text{min}}(\Xi) := G^+$.
\end{example}
\begin{proof}
\begin{align*}
\norm{\nabla J(\xi)} &= \norm{\sum_{x \in \cX} d_{d_0,\Pr}^{\pi}(x) \sum_{u \in \cU} \pi(u \vert x)  \sum_{x'} V_h(x' \vert \Pr)\Pr(x' \vert x,u) }^{1/2}   \\
                     &\geq G^+ =:  B_{\text{min}}(\Xi)  \,.
\end{align*}
\end{proof}

\paragraph{Regret analysis of the policy iterates}
Before deriving the regret of the policy with respect to $\Phi$, we first obtain an expression for the regret based on the observed iterates $J(\theta_k,\xi_k)$ for $k=1,\dots,K$. We start with a regret analysis of the policy iterates. The analysis itself is general under PL condition assumptions and thus can be applied to the above-mentioned undiscounted total cost analysis.

The PL condition leads to the following descent property.
\begin{lemma}[PL based descent property]
\label{lem: descent}
Let $(\theta_k, \xi_k)_{k=1}^{K}$ be a sequence of iterates. Let $J(\theta_k, \xi_k)$ be $l(\xi_k)$-smooth as a function of $\theta$ and let $J(\cdot,\xi_k)$ satisfy the PL condition (Assumption~\ref{ass: PL}) with coefficient $\mu_k = \mu(\xi_k)$ for all $k=0,\dots,K-1$. If the learning rate is $\eta_1(k-1) \leq 1/L$ for all $k = 1, \dots, K$, where $L = \sup_{\xi \in \Xi} l(\xi)$, then
\begin{equation}
J(\theta_{k},\xi_{k-1}) - J(\theta_{k-1}^*,\xi_{k-1})\leq  (1- \eta_1(k-1) \mu_k) (J(\theta_{k-1},\xi_{k-1}) - J(\theta_{k-1}^*,\xi_{k-1}))
\end{equation}
where $\theta_{k-1}^*$ represents the optimal policy parameters under transition dynamics $\xi_{k-1}$.
\end{lemma}
\begin{proof}
This follows from the definition of $l(\xi_k)$-smoothness and the standard descent lemma (Lemma~\ref{lem: standard descent}). For any $\Delta \theta$, 
\begin{align*}
    J(\theta_{k-1} + \Delta \theta,\xi_{k-1}) \leq J(\theta_{k-1},\xi_{k-1}) + \langle \nabla_\theta J(\theta_{k-1},\xi_{k-1}), \Delta \theta \rangle + \frac{L}{2}\norm{\Delta \theta}^2
\end{align*}
Set $\Delta \theta = -\eta_1(k-1) \nabla_\theta J(\theta_{k-1},\xi_{k-1})$. Then,
\begin{align*}
    J(\underbrace{\theta_{k-1} - \eta_1(k-1) \nabla_\theta J(\theta_{k-1},\xi_{k-1})}_{\theta_{k}},\xi_k) &\leq J(\theta_{k-1},\xi_{k-1}) - \eta_1(k-1) \left( 1 - \frac{L}{2}\eta_1(k-1)\right)\norm{\nabla_\theta J(\theta_{k-1},\xi_{k-1})}^2 \\
                         &= J(\theta_{k-1},\xi_{k-1}) - \frac{\eta_1(k-1)}{2}  \norm{\nabla_\theta J(\theta_{k-1},\xi_{k-1})}^2
\end{align*}
From the PL condition (Assumption~\ref{ass: PL}), we have that
\begin{equation}
    \frac{1}{2}\norm{\nabla_\theta J(\theta_k,\xi_{k-1})}^2 \geq \mu_{k-1} (J(\theta_{k-1},\xi_{k-1}) - J(\theta_{k-1}^*,\xi_{k-1})) \,.
\end{equation}
Subtracting the optimal cost from both sides and substituting the expression for the gradient norm, we get a linear rate to the optimum
\begin{equation}
    J(\theta_{k},\xi_{k-1}) - J(\theta_{k-1}^*,\xi_{k-1}) \leq (1 - \eta_1(k-1)\mu_{k-1}) \left(J(\theta_{k-1},\xi_{k-1}) - J(\theta_{k-1}^*,\xi_{k-1}) \right)\,.
\end{equation}
\end{proof}

Rather than assuming positive PL coefficients at each iterate, we can relax this assumption while maintaining a linear convergence. The following assumption states our specific requirements.
\begin{assumption}[Bound on PL coefficient sum]
\label{ass: S_k}
Define $S_k = \sum_{j=1}^{k} \mu_{j} \eta(j)$ for PL coefficients $\mu_j$ and learning rates $\eta(j)$.  The PL coefficient need not be strictly positive at all times: for all $K$, $S_K \geq C_1 K$ for some $C_1 > 0$.
\end{assumption}
When using a projected gradient algorithm, an analogous assumption will be made w.r.t. the projected gradient in a modified PL condition (see Assumption~\ref{ass: MPL}).

To further support our policy regret analysis, a convergence criterion is assumed on the transition n kernel to ensure that the error to a standard recursion is bounded.
\begin{assumption}[Convergent transition kernel]
\label{ass: v_k}
For $\rho \in (0,1)$, $v_k = \cO(\rho^k)$, where $v_k = \sup_{\theta} \vert J(\theta,\xi_{k-1}) - J(\theta,\xi_{k})   \vert $.
\end{assumption}
\begin{theorem}[Regret of policy iterates]
\label{th: regret upper bound policy}
Let $\{\xi_k\}_{k=1}^{K}$ be the sequence of transition dynamics parameters obtained from the algorithm and let $\eps > 0$ be the desired regret upper bound. Then under the preconditions of Lemma~\ref{lem: descent} and Assumptions~\ref{ass: S_k} and ~\ref{ass: v_k},  we have that
\begin{equation}
\label{eq: policy iterate regret}
J(\theta_{K}, \xi_{K}) - J(\theta_{K}^*, \xi_{K})  \leq  \exp(-c_{\text{pol}} K) C_3 
\end{equation}
where $c_{\text{pol}} \in (0,C_1)$ and $C_3 = J(\theta_{0}, \xi_{0}) - J(\theta_{0}^*, \xi_{0})  + \frac{2}{1-\rho} \exp(C_2)$. Within $K = \cO(\log(\frac{1}{\eps}))$ iterations, and $\cO(\eps)$-optimal policy is obtained under exact policy gradient.\
For the inexact, sample-based policy gradient, an $\cO(\eps)$-optimal policy can be obtained with the same setting of $K$ with $n = \cO(\frac{\log(\frac{1}{\eps})}{\eps})$ trajectories per policy gradient update.
\end{theorem}
\begin{proof}
Define shorthands $F_k(\theta) = J(\theta,\xi_k)$ and $F_k^* = \min_{\theta} J(\theta,\xi_k)$ and $\Delta_k = F_k(\theta_k) - F_k^*$.
Via the PL assumption, suppose that the gradient step gives the recursion
\begin{align*}
F_{k-1}(\theta_k) - F_{k-1}^* \leq  (1 - \mu_{k-1} \eta_{k}) \Delta_{k-1}
\end{align*}
The upper bound on the objective change across two iterations is
$v_k = \sup_{\theta} \vert F_{k-1}(\theta) - F_k(\theta') \vert $
concludes that 
\begin{align*}
\Delta_k & = F_k(\theta_k) - F_k^* \\
         & =  F_{k-1}(\theta_k) - F_{k-1}^* + (F_{k}(\theta_k) - F_{k-1}(\theta_k)) + (F_{k-1}^* - F_k^*) \tag{decomposing} \\
         &\leq F_{k-1}(\theta_k) - F_{k-1}^*  + v_k  + F_{k-1}(\theta_{k}^*) - F_k(\theta_{k}^*)  \tag{$v_k$ definition, cost of suboptimal policy is bigger}\\
         &\leq (1 - \mu_{k-1} \eta_{k}) \Delta_{k-1} + 2 v_k \,. \tag{recursion, $v_k$ definition}
\end{align*}
Denoting $S_K = \sum_{k=1}^{K} \mu_{k-1} \eta_{k}$, and noting that for any $x$, $1 - x \leq \exp(-x)$, we have via Assumption~\ref{ass: S_k} and Assumption~\ref{ass: v_k} that
\begin{align*}
\Delta_K          & \leq  \exp(-S_K) \Delta_0 + 2 \sum_{j=1}^{K} \exp(-(S_K - S_j)) v_j \\
	               &\leq  \exp(-S_K) \Delta_0 + \frac{2 K }{1-\rho} \tag{$\mu_j \geq 0$} \\
                  & \leq  \exp(-C_1 K) (\Delta_0 + \frac{2 K }{1-\rho}  \\
                  & \leq K \exp(-C_1 K) (\Delta_0 + \frac{2}{1-\rho} \\
                  & \leq \exp(-c_{\text{pol}} K) (\Delta_0 + \frac{2}{1-\rho} )  \\
                  &\leq  \exp(-c_{\text{pol}} K) C_3 \,.
\end{align*}
where $c_{\text{pol}} \in (0,C_1)$.\\ 
Via the PL property, the inexact gradient descent produces a recursion
\begin{align*}
\Delta_k &\leq (1 - \mu_{k-1} \eta_{k}) \Delta_{k-1} + C_4 \norm{e_k}^2 + 2 v_k\,.
\end{align*}
Now let $K=\frac{1}{c_{\text{pol}}} \log(\frac{1}{\eps})$ and let $n = \cO(\frac{\log(\frac{1}{\eps})}{\eps})$. Then  with probability at least $1-\delta$, we obtain via Lemma~\ref{lem: trajectories} that
\begin{align*}
C_4 \sum_{k=1}^{K}\norm{e_k}^2 = \cO(K n^{-1}) =  \cO(\eps) \,.
\end{align*}
Further, the setting of $K$ implies that $\exp(-c_{\text{pol}} K) C_3 = \exp(-c_{\text{pol}} \frac{1}{c_{\text{pol}}}\log(\frac{1}{\eps})) C_3 =  \cO(\eps)$. The $\cO(\eps)$-optimality then follows from the previous derivations.
\end{proof}

\paragraph{Regret analysis of the adversary iterates}  
Based on the modified PL condition, the below lemma formulates the ascent property for adversary gradient ascent.
\begin{lemma}[Modified PL based ascent property]
\label{lem: ascent}
Let $(\xi_k^{t'})_{t'=1}^{t_k}$ be a sequence of iterates from the inner problem, applying projected gradient ascent in a compact uncertainty set $\Xi$ with $\ell_q$ distance projection for some $q \geq 1$. Let $J(\theta_k, \xi)$ be $l(\theta_k)$-smooth as a function of $\xi$ and let $J(\theta_{k},\xi_{k-1})$ satisfy the Modified PL condition (Assumption~\ref{ass: MPL}) with coefficient $\mu_{k,t}$ for all $k \in [K]$ and all $t \in \{0,\dots,t_k-1\}$. If the learning rate is $\eta_2(k,t) \leq \frac{1}{L}$ for all $k \in \{0,\dots,K-1\}$ and all $t \in \{0,\dots,t_k-\}$, where $L = \sup_{\theta \in \Theta} l(\theta)$, then
\begin{equation}
J(\theta_{k},\xi_{k-1}^{*}) - J(\theta_{k},\xi_{k-1}^t) \leq  (1- \eta_2(k-1,t-1) \mu_{k,t-1}) (J(\theta_{k},\xi_{k-1}^{*}) - J(\theta_{k},\xi_{k-1}^{t-1}))
\end{equation}
where $\xi_{k-1}^*$ represents the optimal transition dynamics for policy $\theta_{k}$.
\end{lemma}
\begin{proof}
The proof is analogous to that of Lemma~\ref{lem: descent} but makes use of the projected gradient mapping. The full derivation is provided in Appendix~\ref{app: ascent}. 
\end{proof}

The above leads to a linear convergence for the adversary, allowing to obtain an $\epsilon_k$-optimal transition kernel at each macro-iteration $k \in [K]$ within a limited number of updates $t_k$.
\begin{theorem}[Regret upper bound of adversary]
\label{th: regret upper bound adversary}
Let $\theta_{k}$ parametrise the policy at macro-iteration $k$ before the adversary update, and let $\xi_{k-1}^{t}$ parametrise the transition dynamics at initialization $t=0$ and after micro-iterations $t \in [t_k]$. Let $\eta_2(k,t)$ be the step size of the adversary at macro-iteration $k$ and micro-iteration $t$ with $\eta_2(k,t) \leq 1/L$ for all $k,t \geq 0$, where $L = \sup_{\theta \in \Theta} l(\theta)$ is the maximal smoothness when $J$ is parametrised with fixed $\theta$. Moreover, let  $\xi_{k-1}^*$ represent the optimal adversary parameters for macro-iteration $k$. After $t_k$ micro-iterations of updating the inner parameters, the regret is bounded by 
\begin{equation}
J(\theta_{k}, \xi_{k-1}^{*}) - J(\theta_{k}, \xi_{k-1}^{t_k}) \leq  \exp(-C_1’ t_k) (J(\theta_{k},\xi_{k-1}^{*}) - J(\theta_{k},\xi_{k-1}^{0}))  \,,
\end{equation}
where $C_1’ > 0$ is the constant derived from Assumption~\ref{ass: S_k} with PL coefficient $\mu_{k,t-1}$ defined as the modified PL coefficient for the projected gradient ascent algorithm at macro-iteration $k$ and micro-iteration $t$. For each macro-iteration $k \in [K]$, a setting of $t_k = \cO(\log(\frac{1}{\epsilon_k}))$ is sufficient to obtain an $\epsilon_k$-optimal solution. \\
For a sample-based setting, with $n = \cO(\frac{\log(\frac{1}{\eps})}{\eps})$ trajectories per gradient update,  the regret is bounded by 
\begin{equation}
J(\theta_{k}, \xi_{k-1}^{*}) - J(\theta_{k}, \xi_{k-1}^{t_k}) \leq  \exp(-C_1’ t_k) (J(\theta_{k},\xi_{k-1}^{*}) - J(\theta_{k},\xi_{k-1}^{0})) + \cO(\eps) \,.
\end{equation}
\end{theorem}
\begin{proof}
Let $k \in [K]$ be a macro-iteration and let $\Delta_{k,t} = J(\theta_k,\xi_{k-1}^t) - J(\theta_k,\xi_{k-1}^{*})$. Via the modified PL assumption, suppose that the gradient step gives the recursion
\begin{align*}
\Delta_{k,t} \leq  (1 - \mu_{k,t-1} \eta_2(k-1,t-1)) \Delta_{k,t-1} \,.
\end{align*}
Denoting $S_{k,t_k} = \sum_{t=1}^{t_k} \mu_{k,t-1} \eta_2(k-1,t-1)$, and noting that for any $x$, $1 - x \leq \exp(-x)$, we have via Assumption~\ref{ass: S_k} for some $C_1’ > 0 $ that
\begin{align*}
\Delta_{k,t_k} &\leq  \exp(-S_{t_k}) \Delta_{k,0} \\
                  & \leq  \exp(-C_1’ t_k) \Delta_{k,0} \,.
\end{align*}
Hence, an $\epsilon_k$-optimal transition kernel is obtained within $t_k = \cO(\log(\frac{1}{\epsilon_k})) \leq \cO(\log(\frac{1}{\eps}))$ micro-iterations.
For a sample-based analysis, define the errors in the projected gradient as $(e_t')_{t=1}^{t_k}$. Via the PL recursion for some constant $C_4' > 0$,
\begin{align*}
\Delta_{k,t} \leq  (1 - \mu_{k,t-1} \eta_2(k-1,t-1)) \Delta_{k,t-1} + C_4' \norm{e_t'}^2 \,.
\end{align*}
Via Lemma~\ref{lem: trajectories}, setting $n = \cO(\frac{\log(\frac{1}{\eps})}{\eps})$ gives an $\cO(\sqrt{\eps\log^{-1}(\frac{1}{\eps})})$ correct estimate. Repeating over iterations gives, for $t_k =  \cO(\log(\frac{1}{\epsilon_k})) \leq \cO(\log(\frac{1}{\eps}))$,
\begin{align*}
\Delta_{t_k} &\leq \exp(-C_1' t_k) \Delta_{k,0} + C_4' \sum_{t=1}^{t_k} \norm{e_t'}^2 \\
             & = \cO(\eps)) \,.
\end{align*} 
\end{proof}

\paragraph{Oracle-based regret analysis}
The above theorems do not yet give an expression for the robust objective, $\Phi$, in the RMDPs. The theorem below provides a linear convergence for the last-iterate regret upper bound with respect to $\Phi$ under a further condition on the learning rate.
\begin{theorem}[Oracle-based last-iterate regret of double-loop algorithm]
\label{th: oracle-based regret double loop}
Let $\epsilon_k \leq \exp(- c_{\text{pol}}) \epsilon_{k-1}$ be the adversary error tolerance. Following the definitions in Lemma~\ref{th: regret upper bound policy}, the double-loop algorithm obtains a linear convergence in the robust objective according to
\begin{equation}
 \Phi(\theta_{K}) - \Phi(\theta^*) \leq \exp(-c_{\text{pol}} K) (C_3 + \cO(\epsilon_0)) \,.    
\end{equation}
For an $\cO(\eps)$-optimal policy, at most $K = \cO(\log(\frac{1}{\eps}))$ iterations and $K_2 = \cO\left(\log^2(\frac{1}{\eps})\right)$ gradient evaluations are required.
\end{theorem}
\begin{proof}
The robust regret is upper bounded by  
\begin{align*}
\Phi(\theta_{K}) - \Phi(\theta^*) & \leq J(\theta_K,\xi_K) + \epsilon_K  - J(\theta^*,\xi_K)  \tag{Lemma~\ref{lem: robust-objective regret}} \\
                                  & \leq J(\theta_K,\xi_{K}) + \epsilon_K  - J(\theta_{K}^*,\xi_{K}) \tag{ $\theta_{K}^*$ is minimiser for $\xi_{K}$} \\
                                  & \leq  \exp(-c_{\text{pol}} K)  \left(C_3 + \epsilon_0 \right),
\end{align*}
where the last step follows from Lemma~\ref{th: regret upper bound policy} and the definition of $\epsilon_k$.
The $\cO(\eps)$-optimality follows from the setting of $K$, analogous to Theorem~\ref{th: regret upper bound policy}. \\
Defining $\overline{t} = \max_{k \in [K]} t_k$, the number of gradient evaluations per inner optimisation loop is given by $\overline{t} = \Theta\left(\log(\frac{1}{\eps})\right)$. With one policy step and at most $\overline{t}$ adversary steps, this leads to a total of $K_2 \leq K(1+\overline{t}) = \cO\left(\log^2(\frac{1}{\eps})\right)$ gradient evaluations. 
\end{proof}

\paragraph{Sample-based regret analysis}
To account for the inaccuracy of the policy gradient estimates, the below analysis provides the last-iterate regret upper bound based on a sample-based analysis. We also briefly provide an average regret based analysis that extrapolates the last-iterate  bound.

\begin{theorem}[Sample-based last-iterate regret of double-loop algorithm]
\label{th: sample-based regret double-loop}
Let $\eps > 0$, let $(\epsilon_k)_{k=0}^{K}$ be the suboptimality of the adversary w.r.t. the policy at initialisation and iteration $k \in [K]$, and let $\eta_1,\eta_2$ satisfy their respective upper bounds as in Theorem~\ref{th: regret upper bound policy} and Theorem~\ref{th: regret upper bound adversary}. The last-iterate regret with respect to the robust objective \eqref{eq: RMDP} is given by 
\begin{equation} 
 \Phi(\theta_{K}) - \Phi(\theta^*) \leq \exp(-c_{\text{pol}} K) (C_3 + \cO(\epsilon_0))  + \tilde{\cO}(\eps) \,.   
\end{equation}
For an $\cO(\eps)$-optimal policy, at most $K = \cO(\log(\frac{1}{\eps}))$ iterations and $K_2 = \cO\left(\log^2(\frac{1}{\eps})\right)$ gradient evaluations are required, and only $n = \cO(\frac{\log(\frac{1}{\eps})}{\eps})$ trajectories are required per gradient evaluation, for a sample complexity of  $\tilde{\cO}\left(\frac{1}{\eps}\right)$.
\end{theorem}
\begin{proof}
The sample-based analysis adds the gradient approximation errors $(e_k)_{k=1}^{K}$ to the recursion of Theorem~\ref{th: regret upper bound policy}:
\begin{align*}
\Delta_k \leq (1 - \mu_{k-1} \eta_1(k-1)) \Delta_{k-1} + 2 v_k + C_4 \norm{e_k}^2 \,,
\end{align*}
which leads to 
\begin{align*}
\Delta_K \leq M \exp(-c_{\text{pol}} K)  \left(C_3 + \epsilon_0 \right) + \sum_{k=1}^{K} C_4 \norm{e_k}^2 \,,
\end{align*}
where $e_k = \cO(\sqrt{\eps})$. \\
The argument is analogous to that of Theorem~\ref{th: regret upper bound policy}. From setting $K=\cO(\log(\frac{1}{\eps}))$, following Lemma~\ref{lem: trajectories} we can set $n = \cO(\frac{\log(\frac{1}{\eps})}{\eps})$, and then with probability at least $1-\delta$, we have
\begin{align*}
C_4 \sum_{k=1}^{K}\norm{e_k}^2 = \cO(K n^{-1}) =  \cO(\eps) 
\end{align*}
so that
\begin{align*}
\Delta_K \leq \exp(-c_{\text{pol}} K)  \left(C_3 + \epsilon_0 \right) + \cO(\eps) = \cO(\eps)\,.
\end{align*}
Analogously to the proof in Theorem~\ref{th: oracle-based regret double loop}, the robust regret is upper bounded by  
\begin{align*}
\Phi(\theta_{K}) - \Phi(\theta^*) \leq J(\theta_K,\xi_{K}) + \epsilon_K  - J(\theta_{K}^*,\xi_{K} = \Delta_K + \cO(\eps) = \cO(\eps)\,.
\end{align*}
Now the sample complexity can be derived. Defining $\overline{t} = \max_{k \in [K]} t_k$, the number of gradient evaluations per inner optimisation loop is given by $\overline{t} = \Theta\left(\log(\frac{1}{\eps})\right)$. For the outer loop, $K = \cO\left(\log(\frac{1}{\eps})\right)$ iterations are sufficient for an $\eps$-optimal policy. With one policy step and at most $\overline{t}$ adversary steps, this leads to a total of $K_2 \leq K(1+\overline{t}) = \cO\left(\log^2(\frac{1}{\eps})\right)$ gradient evaluations. With $n = \frac{\log(\frac{1}{\eps})}{\eps}$, this corresponds to an overall last-iterate sample complexity of $\tilde{\cO}\left(\frac{1}{\eps}\right)$.
\end{proof}

\subsection{A sample-based analysis for the average regret}
Having obtained the sample complexity for the last-iterate regret, it is straightforward to obtain a bound for the average regret. 
\begin{theorem}
\label{th: average regret sample complexity}
Under the preconditions of Theorem~\ref{th: sample-based regret double-loop}, except that we now set $K = \Theta(\frac{1}{\eps})$, an average regret of $\cO(\eps)$ can be obtained within $\tilde{\cO}(\frac{1}{\eps})$ trajectories.
\end{theorem}
\begin{proof}
Taking the last-iterate regret from Theorem~\ref{th: sample-based regret double-loop} for every $k \in [K]$, we have
\begin{align*}
\frac{1}{K} \sum_{k=1}^{K} \Phi(\theta_{K}) - \Phi(\theta^*) & \leq  \frac{1}{K} \sum_{k=1}^{K}   \exp(-c_{\text{pol}} k)  \left(C_3 + \epsilon_0 \right) + \cO(\eps) \nonumber \\
                                                            &\leq \left(1 + \frac{\left(C_3 + \epsilon_0 \right)}{1 - \exp(-c_{\text{pol}})}\right) \cO(\eps)  \tag{geometric series and $K = \Theta(\frac{1}{\eps})$} \\
                                                            &= \cO(\eps) \,.
\end{align*}
Note that we required $K_2 \leq K(1+\overline{t}) = \Theta(\frac{1}{\eps} (1 + \log(\frac{1}{\eps})))$ gradient evaluations and $n = \cO(\frac{\log(\frac{1}{\eps})}{\eps})$ trajectories per gradient evaluation, yielding a total sample complexity of $\tilde{\cO}(1/\eps^2)$.
\end{proof}

\subsection{Mean-field optimisation}
\label{sec: mean-field optimisation}
Now we consider mean-field optimisation, which is based on a distributional optimisation formulation. Such a formulation can more readily be used for global optimisation e.g. using particle-based optimisation. The algorithm leads to a second regret analysis which is based on mean-field Langevin stochastic descent-ascent (MFL-SDA) \citep{Liu2025}, an algorithm for distributional min-max optimisation. We formulate the approach for RMDPs as summarised in Algorithm~\ref{alg: mfl-sda}. The assumptions are inherited from \citep{Liu2025}. Rather than the PL condition, its distributional counterpart, the log-Sobolev inequality is assumed. 

\begin{definition}[Log-Sobolev inequality (LSI)]
\label{def: LSI}
A distribution $\nu'$ satisfies LSI with parameter $\alpha > 0$ if for all $\nu \ll \nu'$ (i.e. absolutely continuous),
\begin{equation}
\cH( \nu \vert\vert \nu') \leq \frac{1}{2\alpha} \cI(\nu \vert \nu') \,,
\end{equation}
where $\cI(\nu \vert \nu') = \Ex_{\nu}\left[\norm{ \nabla \log \frac{\dif \nu}{\dif \nu'}}^2 \right]$ is the relative Fisher information.
\end{definition}

\subsubsection{Regret analysis}
MFL-SDA solves the min-max optimisation problem
\begin{equation}
\label{eq: energy}
\min_{\nu \in \bP[\Theta]} \max_{\rho \in \bP[\Xi]} E(\nu,\rho) = G(\nu,\rho) + \tau \left( \cH(\nu \vert \nu_{\text{ref}})    - \cH(\rho \vert \rho_{\text{ref}})  \right) \,,
\end{equation}
where $G(\nu,\rho) = \Ex_{\theta \sim \nu,\xi \sim \rho}\left[J(\theta,\xi) \right]$ and the ref subscripts indicate Gaussian reference distributions. \(\bP[\Theta]\) denotes the set of distributions over \(\Theta\), the policy parameter space, and \(\bP[\Xi]\) is defined similarly for the space of transition kernels. In terms of the relation to the CT-RMDP objective \eqref{eq: RMDP}, note that the total cost $J(\theta,\xi)$ is incorporated into the MFL-SDA objective through the term $G(\mu,\rho) = \Ex_{\theta \sim \nu,\xi \sim \rho}[J(\theta,\xi)]$ to obtain a distributional min-max optimisation problem; the formulation effectively averages the cost of the CT-RMDP over the policy and adversary distributions. The additional relative entropy regularisation term is there to guide the optimisation to remain close to particular reference distributions (here chosen as Gaussians). The relative entropy to the Gaussian reference distribution is analytically convenient and useful for encouraging exploration as it can be decomposed into an entropy maximisation term and a quadratic penalty that resolves to a linear term when taking the gradient.). 

The energy functional is strongly convex in $\nu$ and strongly concave in $\rho$, leading to the existence of a unique Nash equilibrium. This leads to Wasserstein gradient flow equations given by 
\begin{align}
&\dif\theta_t =   - \nabla \frac{\delta G[\nu_{t},\rho_{t}](\theta_{t})}{\delta \nu}\dif t   - \tau \theta_{t}  \dif t  + \sqrt{2 \tau} \dif B_t^1 \label{eq: WGF1} \\
&\dif \xi_t =  \eta \nabla \frac{\delta G[\nu_{t},\rho_{t}](\xi_{t})}{\delta \rho}\dif t  - \tau \xi_{t}  \dif t  + \sqrt{2 \tau} \dif B_t^2 \label{eq: WGF2} \,,
\end{align}
where $\frac{\delta G[\nu_{t},\rho_{t}](\theta_{t})}{\delta \nu}$ and $\frac{\delta G[\nu_{t},\rho_{t}](\xi_{t})}{\delta \rho}$ represent the first variations of $G$ with respect to $\nu$ and $\rho$, and $B_t^i$ for $i=1,2$ represent separate Brownian terms initialised at $B_0^i = 0$. The terms $\tau \theta_{t}$ and $\tau \xi_{t}$ show up due to the KL divergence being equal to the entropy plus second order moment for standard-normal reference distributions. The time-scale ratio $\eta > 0$ can be set following \cite{Lu2023} to ensure linear convergence. discretising the flow in time, MFL-SDA updates are given by the lines 4--5 of Algorithm~\ref{alg: mfl-sda}. 

The specific LSI assumption over the policy and adversary distributions is defined through Gibbs distributions representing the inner maximisers.
\begin{assumption}
\label{ass: LSI}   
For any $\nu \in \bP[\Theta]$ and any $\rho \in \bP[\Xi]$, the Gibbs distributions
\begin{align*}
&P_{\nu}[\rho](\xi) \propto \exp\left( \tau^{-1}  \frac{\delta G[\nu,P_{\nu}[\rho]](\xi)}{\delta \rho} - \frac{1}{2}\norm{\xi}^2  \right) \\
&P_{\rho}[\nu](\theta) \propto \exp\left( - \tau^{-1}  \frac{\delta G[P_{\rho}[\nu],\rho](\theta)}{\delta \nu} - \frac{1}{2}\norm{\theta}^2  \right)
\end{align*}
both satisfy LSI (Definition~\ref{def: LSI}) with parameter $\alpha > 0$.
\end{assumption}

The relation between the PL condition and LSI has been discussed in previous works. In particular, one of the interpretations is as a PL condition over the functional $\cH(\nu \vert \nu') = \int \log(\frac{\dif \nu}{\dif \nu'}) \dif \nu$  in the Wasserstein geometry \citep{Otto2000}. Applying this reasoning with $\nu' = P_{\nu}[\rho]$, 
\begin{align}
\cH(\nu \vert \nu') \leq \frac{1}{2\alpha} \cI(\nu \vert \nu') = \frac{1}{2\alpha} \int \norm{\nabla \log \frac{\dif \nu}{\dif \nu'}}^2 \dif \nu  = \frac{1}{2\alpha} \norm{\nabla  \cH(\nu \vert \nu')}^2 \,.
\end{align}
Analogously $\cH(\rho \vert \rho') \leq \frac{1}{2\alpha} \norm{\cH(\rho \vert \rho')}^2$ for $\rho' = P_{\rho}[\nu]$. Such a PL condition enables linear convergence to the optimal Gibbs distributions.

In addition to the LSI assumption, MFL-SDA requires the three below assumptions. 
\begin{assumption}[Initial condition]
\label{ass: initial condition}
The initial distribution has finite energy, i.e. $E(\nu_0,\rho_0) < \infty$, and the initial third order moments satisfy $\Ex_{\theta \sim \nu_0} \left[ \norm{\theta}_2^4 \right], \Ex_{\xi \sim \rho_0} \left[ \norm{\xi}_2^4 \right] < \infty$.
\end{assumption}

\begin{assumption}[Regularity of $G$]
\label{ass: regularity}
The functional $G(\nu,\rho)$ is convex in $\nu$ and concave in $\rho$. The first variations have bounded higher-order derivatives, up to the fourth power, i.e. $\norm{\nabla^i \frac{\delta G}{\delta \nu}}_F,\norm{\nabla^i \frac{\delta G}{\delta \rho}}_F \leq M_i$ where $\norm{\cdot}_{F}$ indicates the Frobenius norm. The cross second-order variations are bounded, i.e. $\norm{\frac{\delta^2 G}{\delta \nu \delta \rho}[\nu, \rho](\xi)}_{\infty} \leq C_0$. Moreover, the Hessian and Jacobian of its second variations are bounded, i.e. $\norm{\nabla_{\theta} \nabla_{\theta}^{\tr} \frac{\delta^2 G}{\delta \nu \delta \rho}}, \norm{\nabla_{\theta} \nabla_{\theta}^{\tr} \frac{\delta^2 G}{\delta \nu \delta \rho}} \leq  C_1$ and $\norm{\nabla_{\theta} \frac{\delta^2 G}{\delta \nu \delta \rho}}, \norm{\nabla_{\rho} \frac{\delta^2 G}{\delta \nu \delta \rho}} \leq  C_2$.
\end{assumption}
Note that the convexity in $\nu$ is not restrictive since for any $\gamma \in  (0,1)$
\begin{align*}
&G(\gamma \nu \! + \!(1-\gamma) \nu',\rho) = \int_{\Xi} \int_{\Theta} \! J(\theta,\xi) \, (\gamma \nu \!+ \!(1-\gamma) \nu)(\theta) \rho(\xi) \\
&= \int_{\Xi} \int_{\Theta} \! J(\theta,\xi) \gamma \nu(\theta) \rho(\xi) +  \int_{\Xi} \int_{\Theta} \! J(\theta) (1-\gamma) (\theta,\xi) \rho(\xi) 
= \gamma G(\nu,\rho) +  (1-\gamma) G(\nu',\rho) \,.
\end{align*}
For concavity in $\rho$, the analogous equality follows.

\begin{assumption}[Bounded moments]
\label{ass: bounded moments}
Define $g_k\:\nabla \frac{\delta}{\delta \nu}  G[\nu_{k},\rho_{k}](\theta_{k})$ and $h_k\:\frac{\delta J}{\delta \rho} [\nu_{k+1},\rho_{k}](\xi_{k})$. Then there exists a scalar $\zeta \geq 0$ such that $\Ex \left[\norm{g_k - \hat{g}_k}^4 \vert \nu_{k},\rho_{k} \right], \Ex \left[\norm{h_k - \hat{h}_k}^4 \vert \nu_{k+1},\rho_{k} \right] \leq \zeta$.
\end{assumption}

For one-population MFLD, uniform-in-time propagation of chaos \citep{Chen_2025, Suzuki2023} guarantees that an $N$-particle implementation tracks the mean-field
dynamics with an $\cO(1/N)$ error in suitable law-level metrics; for the two-population descent-ascent dynamics considered here no such result is available, and we formulate it as a hypothesis. We implement the algorithm with $N$ particles.

Due to $G$ being bilinear in its arguments, the exact Wasserstein gradients are given by
\begin{align}
& \nabla \frac{\delta G}{\delta \nu}[\nu,\rho](\theta)  
=\Ex_{\xi \sim \rho}[ \nabla_\theta J(\theta,\xi)], \label{eq: exact Wasserstein policy}
\\
& \nabla \frac{\delta G}{\delta \rho} [\nu,\rho](\xi) 
=\Ex_{\theta \sim \nu}[ \nabla_\xi J(\theta,\xi)]\,.  \label{eq: exact Wasserstein adversary}
\end{align}
We then consider finite-particle approximations similar to SGD-MFLD \citep{Suzuki2023}, 
\begin{align}
& \nabla \frac{\delta G}{\delta \nu}[\nu,\rho](\theta^i)  
\approx \frac{1}{N}\sum_{j=1}^N \nabla_\theta J(\theta^i,\xi^j),
\label{eq: particle approximation policy} \\
& \nabla \frac{\delta G}{\delta \rho} [\nu,\rho](\xi^j) 
\approx \frac{1}{N}\sum_{i=1}^N \nabla_\xi J(\theta^i,\xi^j)\,.  \label{eq: particle approximation adversary}
\end{align}

\begin{algorithm}
\caption{Mean-field Langevin Stochastic Descent-Ascent (MFL-SDA) for RMDPs.} \label{alg: mfl-sda}
\begin{algorithmic}[1]
\State \textbf{Inputs:} initial parameter distributions $\nu_0$,$\rho_0$, initial $N$ policy and adversary parameters $\theta_0,\xi_0$, learning rates $\eta_1,\eta_2$, Gaussian noise $\{e_k^i\}_{k=1}^{K}$ with $e_k^i \in \Reals^{N \times d}$ for $i=1,2$.
\State \textbf{Outputs:} near-optimal parameter distributions $\nu_K$,$\rho_K$
 \For{$k = 1,\dots,K$}
 \State Sample trajectories using $\pi_{\theta_{k-1}^i}$ and $f_{\xi_{k-1}^i}$ for all $i=1,\dots,N$.
\State Estimate $\nabla \frac{\delta G}{\delta \nu} [\nu_{k-1},\rho_{k-1}](\theta_{k-1})$ (e.g. \eqref{eq: particle approximation policy}).
\State Update $\theta_k \gets \theta_{k-1} - \eta_1 \left(\hat{\nabla} \frac{\delta G}{\delta \nu} [\nu_{k-1},\rho_{k-1}](\theta_{k-1}) + \tau \theta_{k-1}    \right)  + \sqrt{2 \eta_1 \tau} e_k^1$
 \State Sample trajectories using $\pi_{\theta_{k-1}^i}$ and $f_{\xi_{k-1}^i}$ for all $i=1,\dots,N$.
\State Estimate $\nabla \frac{\delta G}{\delta \rho} [\nu_{k},\rho_{k-1}](\xi_{k-1})$ (e.g. \eqref{eq: particle approximation adversary}). 
\State Update $\xi_k \gets \xi_{k-1} + \eta_2 \left( \hat{\nabla} \frac{\delta G}{\delta \rho} [\nu_{k},\rho_{k-1}](\xi_{k-1}) - \tau \xi_{k-1}    \right)  + \sqrt{2 \eta_2 \tau} e_k^2$
 \EndFor
\end{algorithmic}
\end{algorithm}

MFL-DA assumes the availability of exact gradients while MFL-SDA works under unbiased estimates of the true gradient. Both lead to a regret bound composed of a geometric decay term and a bias term, and differ only in that the bias term of now depends on higher-order moments, which are bounded by Assumption~\ref{ass: bounded moments}, and by the regret being formulated in terms of an expectation over the stochasticity of the gradient. Therefore, for brevity, we only treat the stochastic gradient case here.

The stochastic gradient analysis sets the number of iterations to $K= \cO(\frac{1}{\eps} 
\log(\frac{1}{\eps}))$, where the $\eps = \cO(\eta_1)$ in the denominator appears due to the time discretisation of \Cref{eq: WGF1,eq: WGF2} into Algorithm~\ref{alg: mfl-sda}. While this number is slightly increased compared to the setting $K= \cO(\frac{1}{\eps})$ in the sample-based analysis double-loop algorithm in Section~\ref{sec: double-loop}, this algorithm operates on the space of parameter distributions and may have practical benefits in global optimisation.

The proof relies on the Lyapunov function, which is defined for a fixed $\lambda > 0$ as
\begin{align}
\label{eq: Lyapunov}
&\cL(\nu,\rho) = \cL_1(\nu) + \lambda \cL_2(\nu,\rho) \\
&\cL_1(\nu) = \max_{\rho' \in \bP[\Xi]} E(\nu,\rho')   - \min_{\nu \in \bP[\Theta]} \max_{\rho \in \bP[\Xi]} E(\nu,\rho)    \nonumber \\ 
&\cL_2(\nu,\rho) = \max_{\rho' \in \bP[\Xi]} E(\nu,\rho') - E(\nu,\rho)  \nonumber \,,  
\end{align}
and which vanishes to zero for distributions equal to the global optimum $(\nu^*,\rho^*)$ in the weak sense.

\begin{theorem}[Last-iterate regret of MFL-SDA (Corollary~1 of \cite{Liu2025} restated)]
\label{th: regret MFL-SDA}
Let LSI with $\alpha > 0$ hold for the Gibbs distributions $P_{\nu}[\rho]$ and $P_{\rho}[\nu]$ as in Assumption~\ref{ass: LSI}. Let $\eta_1 \leq \frac{1}{C_1}$, $\tau \leq \frac{1}{2C_1^2}$, $\eta_2 \leq \frac{1}{2\tau \alpha}$, and $\eta_1 = \min\{\lambda,0.2,\frac{1}{\tau \alpha} \} \tau \alpha \eta_2$. Further, require that $\eta_1 = \cO(\frac{\eps\tau^3\alpha^3}{d})$, and the fixed-point Gibbs maps satisfy LSI with constant \(\alpha_1\). Then under \cref{ass: initial condition,ass: bounded moments,ass: regularity}, MFL-SDA obtains a (distributional) last-iterate regret w.r.t. the saddle point of 
\begin{equation}
\label{eq: regret MFL-SDA}
\tau\big(H(\nu_K|\nu^*)+H(\rho_K|\rho^*)\big)
\le Q\Big[(1-\eta_1\tau\alpha)^K\cL(\nu_0,\rho_0)+\cO(\varepsilon)\Big],
\qquad Q:=1+\tfrac{2}{\lambda}+\tfrac{4C_0^2}{\tau^2}.
\end{equation}
\end{theorem}

Under the aforementioned $N$-particle approximation and hypotheses, we obtain an $\tilde{\cO}(\frac{1}{\eps})$ iteration complexity.

\subsubsection{Regret analysis under particle approximation}
The convergence guarantee of \cite{Liu2025} applies to the population-level MFL-SDA dynamics. In the implementation, the distributions \((\nu_k, \rho_k)\) are replaced by the aforementioned empirical measures \((\nu_k^N, \rho_k^N)\). For one-population MFLD minimization, \cite{Suzuki2023} established \textit{uniform-in-time} particle error bounds of the order \(\cO(1/N)\) in suitable objective/squared metrics under LSI and regularity assumptions. However, their result does not directly apply to our present two-population descent-ascent dynamics. In particular, uniform-in-time propagation-of-chaos for minimax MFLD requires additional stability arguments/assumptions since the coupled saddle dynamics need not be contractive under the assumptions used for the one-population result.

Propagation of chaos refers to the phenomenon that, as the number of particles \(N\) increases, any fixed subset of the interacting particles becomes asymptotically independent and distributed according to the population mean-field law. 

Let $\Theta=(\theta^1,\dots,\theta^N)$, $\Xi=(\xi^1,\dots,\xi^N)$,
$\nu_\Theta:=\frac1N\sum_{i}\delta_{\theta^i}$ and $\rho_\Xi:=\frac1N\sum_{j}\delta_{\xi^j}$.
Since $\cH(\cdot\,|\,\nu_{\text{ref}})=+\infty$ on atomic measures, the regularised energy
cannot be evaluated at $\nu_\Theta$; the $N$-particle counterpart must be defined on the
\emph{law} of the particle system. For a convex functional $F$ on $\cP_2(\Theta)$
with $\cF_\tau(\nu):=F(\nu)+\tau \cH(\nu|\nu_{\text{ref}})$ and minimiser $\nu_\star$, let
\begin{equation}
  \cF^{(N)}_\tau(\nu^{(N)}) := N\,\Ex_{\Theta\sim\nu^{(N)}}\big[F(\nu_\Theta)\big]
  + \tau H\big(\nu^{(N)}\,\big|\,\nu_{\text{ref}}^{\otimes N}\big),\qquad
  \nu^{(N)}\in\cP_2(\Theta^N),
\end{equation}
so that the empirical measure only enters through $F$. Its minimiser is
$\nu_\star^{(N)}(\mathrm d\Theta)\propto\exp(-\tfrac N\tau F(\nu_\Theta))\,
\nu_{\text{ref}}^{\otimes N}(\mathrm d\Theta)$, and \citep{Chen_2025, nitanda2024MFNN}
\begin{equation}
  \frac{\tau}{N}H\big(\nu^{(N)}\,\big|\,\nu_\star^{\otimes N}\big)
  \le \frac1N\cF^{(N)}_\tau(\nu^{(N)})-\cF_\tau(\nu_\star)\ \ \forall\,\nu^{(N)},
  \qquad
  0\le \frac1N\cF^{(N)}_\tau(\nu_\star^{(N)})-\cF_\tau(\nu_\star)
  \le \Ex_{\Theta\sim\nu_\star^{\otimes N}}\big[B_F(\nu_\Theta,\nu_\star)\big],
\end{equation}
where $B_F(\mu,\mu')=F(\mu)-F(\mu')-\langle\frac{\delta F}{\delta\mu}(\mu'),\mu-\mu'\rangle$;
the right-hand side is $\cO(1/N)$ for smooth $F$. In our setting $F=G(\cdot,\rho)$ is
linear for fixed $\rho$, so $B_F\equiv0$ and $\nu_\star^{(N)}=\nu_\star^{\otimes N}$.
The same holds for the lifted game
\begin{equation}
  E^{(N)}(\nu^{(N)},\rho^{(N)}):=N\,\Ex_{\Theta\sim\nu^{(N)},\,\Xi\sim\rho^{(N)}}
  \big[G(\nu_\Theta,\rho_\Xi)\big]+\tau \cH(\nu^{(N)}|\nu_{\text{ref}}^{\otimes N})
  -\tau \cH(\rho^{(N)}|\rho_{\rm ref}^{\otimes N}),
\end{equation}
with $\Theta$ and $\Xi$ independent: since $\Ex[G(\nu_\Theta,\rho_\Xi)]=G(\bar\nu,\bar\rho)$
for the averaged marginals $\bar\nu,\bar\rho$, and
$\cH(\nu^{(N)}|\nu_{\text{ref}}^{\otimes N})\ge N \cH(\bar\nu|\nu_{\text{ref}})$,
the pair $(\nu^{*\otimes N},\rho^{*\otimes N})$ is a saddle point of $E^{(N)}$ with value
$N E(\nu^*,\rho^*)$. The finite-particle error of Algorithm~\ref{alg: mfl-sda} is therefore purely
dynamical: it stems from each population interacting with the \emph{empirical}
measure of the other, and from the correlations this induces between $\Theta_k$ and $\Xi_k$.

Let the particles of Algorithm~\ref{alg: mfl-sda} be initialised i.i.d.\ from $\nu_0\otimes\rho_0$, so that
their law $\pi_k^{(N)}:=\mathrm{Law}(\Theta_k,\Xi_k)$ is exchangeable, and let
$(\nu_k,\rho_k)$ be the population MFL-SDA iterates. We assume uniform-in-time
propagation of chaos in the form
\begin{equation}
  \sup_k \frac1N W_2^2\big(\pi_k^{(N)},(\nu_k\otimes\rho_k)^{\otimes N}\big)
  \le C_{\rm PoC}\,N^{-\gamma}
\end{equation}
for some $\gamma>0$ (or the same bound with $\frac1N \cH(\cdot\,|\,\cdot)$ in place of
$\frac1N W_2^2$). By exchangeability, the one-particle marginal
$\bar\nu_K^N:=\mathrm{Law}(\theta_K^1)=\Ex[\nu_{\Theta_K}]$, which has a density due to
the Gaussian noise, satisfies $W_2^2(\bar\nu_K^N,\nu_K)\le C_{\rm PoC}N^{-\gamma}$. Since
$\nu^*$ satisfies LSI, and hence Talagrand's inequality \citep{Otto2000}, with
constant $\alpha_1$,
\begin{equation}
  W_2(\bar\nu_K^N,\nu^*)\le \sqrt{C_{\rm PoC}N^{-\gamma}}
  +\sqrt{\tfrac{2}{\alpha_1}\cH(\nu_K|\nu^*)},
\end{equation}
where the second term is controlled by Theorem~11. Entropy-free functionals of the
empirical measures are also well defined; e.g.\ if $J$ is $L_J$-Lipschitz in each argument,
$\big|\Ex[G(\nu_{\Theta_K},\rho_{\Xi_K})]-G(\nu_K,\rho_K)\big|
\le \sqrt2\,L_J\sqrt{C_{\rm PoC}N^{-\gamma}}$.

Combining the two displays above with $(a+b)^2\le 2a^2+2b^2$ gives
\begin{equation}
  W_2^2(\bar\nu_K^N,\nu^*) \le 2C_{\rm PoC}N^{-\gamma} + \tfrac{4}{\alpha_1}\cH(\nu_K|\nu^*),
\end{equation}
and analogously for $\bar\rho_K^N$. This yields a \emph{conditional} finite-particle
complexity statement, with accuracy measured on the squared-$W_2$ (equivalently, KL)
scale of Theorem~\ref{th: regret MFL-SDA}. Choosing $N=\Theta(\eps^{-1/\gamma})$ makes the particle term
$\cO(\eps)$, while the population iteration complexity $K=\cO(\eps^{-1}\log(1/\eps))$
of Theorem~\ref{th: regret MFL-SDA} makes the optimisation term $\cO(\eps)$. Since each MFL-SDA iteration
uses $N^2$ policy-adversary pairs, the resulting trajectory complexity is
$KN^2=\tilde{\cO}(\eps^{-1-2/\gamma})$, with constants depending on $\tau$, $\alpha$,
$\alpha_1$ and $Q$. In particular, $\gamma=1$ yields $\tilde{\cO}(\eps^{-3})$ and
$\gamma=1/2$ yields $\tilde{\cO}(\eps^{-5})$. We stress that $C_{\rm PoC}$ must be
independent of $k$ for this argument; a time-dependent constant would enter $K$
(see the discussion below). The rate $\gamma=1$ is what a synchronous-coupling argument
typically delivers when the interaction is Lipschitz and the drift is contractive.
By contrast, the only uniform-in-time result for a distributional minimax dynamics we
are aware of \citep{kim2024symmetric} applies to the averaged-gradient variant MFL-AG.
It corresponds to $\gamma=1/2$ in the weaker metric $W_1^2$ of the expected empirical
measures.

For \textit{finite-time} propagation-of-chaos, note that in existing results (e.g. Theorem 3 of \cite{Suzuki2023}), an additional constant of $C_K \sim e^{-CK}$ is introduced for some $C_K > 0$ depending on the learning rate $\eta$ and the regularisation $\tau$. Such results require a contractivity/dissipativity assumption to obtain a uniform-in-time result in \(\cO(1/N)\). Informally, the Langevin regularisations \(-\tau \theta\) and \(-\tau \xi\) must dominate the destabilizing coupling caused by \(J\). In our setting, the \(\tau\) terms do give some dissipativity, but the cross terms from \(\nabla_\theta J(\theta, \xi)\) and \(\nabla_\xi J(\theta, \xi)\) can destroy contraction.

In other words, for MFL-SDA, the dissipativity is a condition on the drift/vector field of the Langevin descent-ascent dynamics. 
Let \[z = (\theta, \xi).\]
Then for the population MFL-SDA flow, ignoring measure-dependence for clarity, the drift for regularised descent-ascent is
\[b(z) = \begin{pmatrix}
 -\nabla_\theta J(\theta, \xi) - \tau \theta\\
    \nabla_\xi J(\theta, \xi) - \tau \xi
\end{pmatrix}.\]

Dissipativity states that for two points \(z, z'\),
\[\langle z - z' , b(z) - b(z')\rangle \leq -m \|z-z' \|^2\]
for some \(m > 0\). In the gradient flow dynamics, it implies that  $\norm{z_k - z_k'}^2 \leq e^{-2mk} \norm{z_0 - z'_0}^2$ by Gr\"{o}nwall's inequality. 

Expanding the RHS, we obtain the terms
\[- \langle \theta - \theta' , \nabla_\theta J(\theta, \xi) - \nabla_\theta J(\theta', \xi') \rangle +  \langle \xi - \xi' , \nabla_\xi J(\theta, \xi) - \nabla_\xi J(\theta', \xi') \rangle - \tau \|\theta - \theta'\| - \tau \|\xi - \xi'\|^2.\]
The last two terms are always negative, whereas the first two can oscillate depending on \(J\). This necessitates further investigation and a derivation of stability criterion for this setting.

The precise drift $b$ of the ascent-descent Langevin dynamics would also have measure-dependence, i.e.
\begin{align*}
    &b_\theta(\theta; \rho) = - \Ex_{\xi \sim \rho}[\nabla_\theta J(\theta , \xi)] - \tau \theta\\
    &b_\xi(\xi; \nu) = \Ex_{\theta \sim \nu}[\nabla_\xi J(\theta, \xi)] - \tau \xi.
\end{align*}

The stability condition for this more precise setting would need to be of the form
\begin{align}
    \langle z-z' , b(z; \nu, \rho) - b(z'; \nu', \rho') \rangle \leq - m \|z-z'\|^2 + C \left( W^2_2(\nu, \nu') + W^2_2(\rho, \rho') \right)
\end{align}
where the first term gives the contraction between the particles and the second term accounts for the fact that the two particles may see different empirical distributions.

\begin{remark}
Note that Assumption~\ref{ass: regularity} provides the convex-concave regularity required for the population-level MFL-SDA convergence theorem. However, this is distinct from the joint dissipativity/monotonicity condition needed to obtain uniform-in-time finite-particle propagation of chaos for the two-population particle system. The latter would require additional control of the coupled descent-ascent drift, for example strong monotonicity of the game vector field or dominance of the Langevin regularisation over the cross-interaction constants.
\end{remark}

\subsubsection{A note on Mirror Mean-Field Langevin Dynamics}
The above MFL-SDA algorithm typically solves an \textit{unconstrained} mean-field Langevin descent-ascent problem in Euclidean ambient space, with feasibility handled by the energy function \[E(\nu,\rho) = G(\nu,\rho) + \tau \left( \cH(\nu \vert \nu_{\text{ref}})    - \cH(\rho \vert \rho_{\text{ref}})  \right)\] where \(\nu_{\text{ref}}\) and \(\rho_{\text{ref}}\) denote reference measures (taken to be standard Gaussians). Note that there is no projection operator in \Cref{alg: mfl-sda}. Given that our ascent measure \(\rho\) (parametrised by \(\xi\)) is defined to be in a compact set in our original RMDP problem \eqref{eq: RMDP}, one may need to find an alternate way to incorporate the constraint if the MFL-SDA is to be used. A naive replacement by projection here would (under certain conditions) change the dynamics qualitatively. A hard projection introduces a non-smooth boundary operation, resulting in the inability to use the arguments/techniques in \cite{Liu2025}. To this end, we can leverage the techniques in \cite{gu2025mirrormeanfieldlangevindynamics}, which deal with mirror dynamics of mean-field Langevin dynamics (MFLD). 

In this case, we can assume that the min-max problem \eqref{eq: energy} (at least, for the adversary ascent \(\rho\)) is performed over measures supported on a \textit{convex} set \(X \subset \Reals^d\). Thus, the optimisation problem reads 
\[\min_{\nu \in \bP[\Theta]} \max_{\rho \in \bP[\Xi]} E(\nu,\rho) = G(\nu,\rho) + \tau \left( \cH(\nu \vert \nu_{\text{ref}})    - \cH(\rho \vert \rho_{\text{ref}})  \right).
\]

For the ascent, recall the first variation with respect to \(\rho\):
\[\frac{\delta E}{\delta \rho}[\nu, \rho](\xi), \quad h_\rho(\nu, \rho; \xi) \: \nabla_\xi \frac{\delta E}{\delta \rho}[\nu, \rho](\xi).\]

We assume the existence of an appropriate Legendre mirror map/barrier function for the constrained set; that is, there exists a \(\psi : \bP[\Xi] \rightarrow \Reals\). The primal variable at iteration \(k\), \(\xi_k = \nabla \psi^*(\tilde \xi_k)\) where \(\tilde \xi\) is the dual variable. 

Then the deterministic mirror mean-field ascent step for a constrained particle \(\xi^j_k\) at iteration \(k\) is
\begin{align*}
    \tilde \xi^j_k & = \nabla \psi(\xi^j_k),\\
    \tilde \xi^j_{k + 1/2} &= \tilde \xi^j_k + \eta_2 h_\rho(\nu_k, \rho_k; \xi^j_k),\\
    \xi^j_{k+1} &= \nabla \psi^*(\xi^j_{k+1/2}).
\end{align*}
Equivalently, in proximal/Bregman form, we have the particle update
\[\xi^j_{k+1} = \argmax_{\xi \in \Xi} \left \{  \langle h_\rho(\nu_k, \rho_k; \xi^j_k), \xi \rangle - \frac{1}{\eta_2}D_\psi (\xi, \xi^j_k)\right\}\]
where 
\[D_\psi(\xi, \xi^j_k) = \psi(\xi) - \psi(\xi^j_k) - \langle \nabla \psi(z^j_k), \xi - \xi^j_k \rangle\]
is the Bregman divergence.
Thus, we update in dual coordinates \(\tilde \xi\) and map it back with \(\nabla \psi^*: \Reals^{d_\xi} \rightarrow \Xi\), ensuring feasibility. For more details, we refer to \cite{gu2025mirrormeanfieldlangevindynamics}.

\begin{remark}
\label{rem: uncertainty set parametrisation}
Certain parametrisations involving squashing functions at the output layer can be practically useful even when the parameter space \(\Xi\) is unconstrained. For instance, the parametrisation $g_{\xi}(t,x) = \tanh(W (t,x)^{\tr} + B)$ with $\xi = (W,B) \in \Xi = \Reals^{d_x \times (1+d_x)} \times \Reals^{d_x}$ will obtain deviations from the nominal that are limited to the space $[-1,1]^{d_x}$ despite the parameters being unbounded.
\end{remark}

\section{Experiments}
We now turn to an empirical comparison of the different policy gradient estimators  for CT-RMDPs. We focus on the double-loop robust policy gradient algorithm (Algorithm~\ref{alg: double-loop}). Similarly to earlier implementations of double-loop robust policy gradients \citep{Wang2024,bossens2025mirror}, our double-loop algorithms are implemented by using multiple updates per macro-iteration for the policy and the transition kernel, and the inner problem is solved using a projected gradient ascent.

\subsection{Robust LQR Neural ODE domain} 
The first domain of interest is based on a classical linear quadratic control problem but with an added non-linear term to represent the uncertainty set. The nominal dynamics are given by 
\begin{equation}
\bar{f}(t,x_t,u_t) = A x_t + B u_t \,,
\end{equation}
 for $A \in \Reals^{d_x \times d_x}$ and $B \in \Reals^{d_x \times d_u}$, and the running cost is given by the quadratic equation
\begin{equation}
r(t,x_t,u_t) = x_t^{\tr} Q x_t + u_t^{\tr} B u_t \,,
\end{equation}
for $Q \in \Reals^{d_x \times d_x}$, and the terminal cost is $R(x_T) = x_T^{\tr} Q x_T \Delta t$. The dynamics of the system under parametrisation $\xi$ are given by the addition of the nominal dynamics and an uncertainty term depending on $\xi$ as in \eqref{eq: uncertainty set parametrisation}. The uncertainty term is implemented as an MLP of the form
\begin{equation}
g_{\xi}(t,x_t,u_t) = 2 (W_2 \tanh(W_1 (t, x_t, u_t) + B_1) + B_2) \,,
\end{equation}
where $\xi \in \Xi = \{(W_1,B_1,W_2,B_2) : \norm{(W,b)}_{\infty} \leq \phi\}$ represents the weights and biases. All experiments are conducted with $\cX, \cU = [-5.0,5.0]^2$ and $\phi = 1.0$. The ODE parameters, including $A$, $B$ and the starting state, are uniquely determined for each independent run based on samples from Gaussian distributions.

The policy is an MLP that takes the time and state as input and outputs the action. The policy has 2 fully connected hidden layers of size 128, the first of which is followed by dropout regularisation (with drop probability $0.6$, a strong regularisation that follows previous robust policy gradient works \citep{Wang2024,bossens2025mirror}) and ReLU activation, and the second of which is followed by ReLU activation. Its output layer depends on the type of algorithm. For the pathwise gradient and the deterministic Hamiltonian algorithms, the output layer is a fully connected layer which directly outputs the action through a sigmoid and normalisation. For the deterministic Hamiltonian gradient with exploration, the standard deviation for a Gaussian distribution is defined by an initial log standard deviation of -1.0 and geometric annealing is performed with factor 0.95. For the stochastic Hamiltonian gradient and discrete (REINFORCE) algorithms, the standard deviation is defined through a learnable parameter on the log scale, with range in $[-3.0,0.0]$. The zero order optimiser is based on SPSA \citep{Spall1987} with a single perturbation vector.

\paragraph{Effectiveness of the adversary} Table~\ref{tab: adversaries} shows that both adversarial policy gradients are equally effective in maximizing the cost, with a higher cost compared to zero order optimisation when training for 100 iterations on a randomly initialized policy. As a validation, Table~\ref{tab: test performance} also shows that the empirical maxima observed in the randomised test are well below the costs observed in the adversary test.

\begin{table}[htbp!]
\centering
\caption{Comparison of adversaries w.r.t. the mean and standard error of the total cost across 5 independent runs of adversarial optimisation on a randomly initialised policy in the LQR domain. For each run, the score is normalised to $[0,1]$ based on the the cost obtained by the pathwise adversary on that run, such that 1.0 indicates a score on par with the pathwise adversary.}
\label{tab: adversaries}
\begin{tabular}{c|c}
\toprule\\  
Pathwise Adversary & $1.000 \pm 0.000 $\\ 
ZO Adversary & $0.818 \pm 0.041 $\\ 
Hamiltonian Adversary & $1.000 \pm 0.000 $\\ 
\bottomrule
\end{tabular}
\end{table}

\begin{table}[htbp!]
    \centering
\caption{Normalised performance (mean $\pm$ standard error) on adversarial test (top) and randomised test (bottom) for each discretisation time step ($\Delta t$), aggregated across 5 independent runs of the LQR domain. The bold entries highlight the best two performing algorithms in the adversary test. The adversary test performs 100 iterations of pathwise gradient ascent against the deterministic evaluation policy obtained from the training; the randomised test runs 50 randomly sampled transition kernels and takes the mean and maximum, corresponding to average-case robustness and soft-robustness evaluations (as opposed to worst-case evaluations).}
\label{tab: test performance}
\begin{subtable}[t]{0.99\textwidth}
\caption{Adversary test.}
\label{tab: test performance adversary}
\resizebox{\textwidth}{!}{
\begin{tabular}{l|*{5}{c}}
\toprule \\ 
& \multicolumn{1}{l}{$\Delta t = 0.0005$}& \multicolumn{1}{l}{$\Delta t = 0.001$}& \multicolumn{1}{l}{$\Delta t = 0.005$}& \multicolumn{1}{l}{$\Delta t = 0.01$}& \multicolumn{1}{l}{$\Delta t = 0.05$}\\
\midrule\\  
Discrete Non-robust & $1.303$ \pmt{0.506} & $1.303$ \pmt{0.506} & $1.305$ \pmt{0.506} & $1.308$ \pmt{0.506} & $1.337$ \pmt{0.506} \\ 
Discrete Robust& $0.996$ \pmt{0.196} & $0.997$ \pmt{0.196} & $0.999$ \pmt{0.196} & $1.002$ \pmt{0.196} & $1.029$ \pmt{0.198} \\ 
Stochastic Hamiltonian Robust& $0.591$ \pmt{0.145} & $0.591$ \pmt{0.145} & $0.592$ \pmt{0.145} & $0.594$ \pmt{0.146} & $0.614$ \pmt{0.149} \\ 
Stochastic Hamiltonian Robust ZO Both& $1.123$ \pmt{0.016} & $1.124$ \pmt{0.016} & $1.125$ \pmt{0.016} & $1.127$ \pmt{0.016} & $1.144$ \pmt{0.017} \\ 
Stochastic Hamiltonian Robust ZO Inner& $0.572$ \pmt{0.112} & $0.572$ \pmt{0.112} & $0.574$ \pmt{0.112} & $0.576$ \pmt{0.112} & $0.592$ \pmt{0.113} \\ 
Deterministic Hamiltonian Non-robust & $0.550$ \pmt{0.107} & $0.550$ \pmt{0.107} & $0.551$ \pmt{0.107} & $0.553$ \pmt{0.108} & $0.565$ \pmt{0.110} \\ 
Deterministic Hamiltonian Robust& $\mathbf{0.476}$ \pmt{0.116} & $\mathbf{0.476}$ \pmt{0.116} & $\mathbf{0.477}$ \pmt{0.116} & $\mathbf{0.479}$ \pmt{0.116} & $\mathbf{0.489}$ \pmt{0.117} \\ 
Deterministic Hamiltonian (explore) Non-robust & $0.558$ \pmt{0.106} & $0.558$ \pmt{0.106} & $0.559$ \pmt{0.107} & $0.561$ \pmt{0.107} & $0.574$ \pmt{0.109} \\ 
Deterministic Hamiltonian (explore) Robust& $0.522$ \pmt{0.106} & $0.522$ \pmt{0.106} & $0.523$ \pmt{0.106} & $0.524$ \pmt{0.106} & $0.534$ \pmt{0.108} \\ 
Pathwise Non-robust & $0.550$ \pmt{0.107} & $0.550$ \pmt{0.107} & $0.551$ \pmt{0.108} & $0.553$ \pmt{0.108} & $0.565$ \pmt{0.110} \\ 
Pathwise Robust & $\mathbf{0.434}$ \pmt{0.117} & $\mathbf{0.434}$ \pmt{0.117} & $\mathbf{0.435}$ \pmt{0.117} & $\mathbf{0.436}$ \pmt{0.118} & $\mathbf{0.445}$ \pmt{0.118} \\ 
Pathwise Robust ZO Both & $1.574$ \pmt{0.043} & $1.574$ \pmt{0.043} & $1.576$ \pmt{0.043} & $1.579$ \pmt{0.043} & $1.603$ \pmt{0.043} \\ 
Pathwise Robust ZO Inner & $0.521$ \pmt{0.114} & $0.521$ \pmt{0.114} & $0.522$ \pmt{0.115} & $0.524$ \pmt{0.115} & $0.535$ \pmt{0.116} \\ 
\end{tabular}}
\end{subtable}

\begin{subtable}[t]{0.99\textwidth}
\caption{Randomised test.}
\label{tab: randomised test performance}
\resizebox{\textwidth}{!}{
\begin{tabular}{l|*{10}{c}}
\toprule \\ 
& \multicolumn{2}{l}{$\Delta t = 0.0005$}& \multicolumn{2}{l}{$\Delta t = 0.001$}& \multicolumn{2}{l}{$\Delta t = 0.005$}& \multicolumn{2}{l}{$\Delta t = 0.01$}& \multicolumn{2}{l}{$\Delta t = 0.05$}\\
\midrule\\  
& Mean & Maximum & Mean & Maximum & Mean & Maximum & Mean & Maximum & Mean & Maximum \\  
\midrule\\  
Discrete Non-robust & $0.421$ \pmt{0.222} & $1.207$ \pmt{0.527}& $0.421$ \pmt{0.222} & $1.207$ \pmt{0.527}& $0.422$ \pmt{0.222} & $1.207$ \pmt{0.528}& $0.423$ \pmt{0.222} & $1.208$ \pmt{0.528}& $0.431$ \pmt{0.223} & $1.210$ \pmt{0.533}\\ 
Stochastic Hamiltonian Non-robust& $0.099$ \pmt{0.016} & $0.358$ \pmt{0.079}& $0.099$ \pmt{0.016} & $0.358$ \pmt{0.079}& $0.099$ \pmt{0.016} & $0.358$ \pmt{0.079}& $0.100$ \pmt{0.016} & $0.359$ \pmt{0.079}& $0.103$ \pmt{0.016} & $0.360$ \pmt{0.078}\\ 
Stochastic Hamiltonian Robust& $0.167$ \pmt{0.028} & $0.575$ \pmt{0.127}& $0.167$ \pmt{0.028} & $0.575$ \pmt{0.127}& $0.168$ \pmt{0.028} & $0.574$ \pmt{0.126}& $0.169$ \pmt{0.028} & $0.573$ \pmt{0.126}& $0.178$ \pmt{0.028} & $0.561$ \pmt{0.122}\\ 
Stochastic Hamiltonian Robust ZO Both& $0.322$ \pmt{0.010} & $0.999$ \pmt{0.036}& $0.322$ \pmt{0.010} & $0.999$ \pmt{0.036}& $0.323$ \pmt{0.010} & $1.001$ \pmt{0.036}& $0.323$ \pmt{0.010} & $1.003$ \pmt{0.036}& $0.329$ \pmt{0.011} & $1.020$ \pmt{0.035}\\ 
Stochastic Hamiltonian Robust ZO Inner& $0.128$ \pmt{0.023} & $0.400$ \pmt{0.101}& $0.128$ \pmt{0.023} & $0.400$ \pmt{0.101}& $0.129$ \pmt{0.023} & $0.399$ \pmt{0.101}& $0.129$ \pmt{0.023} & $0.398$ \pmt{0.100}& $0.136$ \pmt{0.023} & $0.391$ \pmt{0.093}\\ 
Deterministic Hamiltonian Non-robust & $0.094$ \pmt{0.013} & $0.353$ \pmt{0.072}& $0.094$ \pmt{0.013} & $0.353$ \pmt{0.072}& $0.094$ \pmt{0.013} & $0.353$ \pmt{0.072}& $0.095$ \pmt{0.013} & $0.354$ \pmt{0.072}& $0.098$ \pmt{0.013} & $0.356$ \pmt{0.074}\\  
Deterministic Hamiltonian Robust& $0.118$ \pmt{0.026} & $0.474$ \pmt{0.113}& $0.118$ \pmt{0.026} & $0.474$ \pmt{0.113}& $0.118$ \pmt{0.026} & $0.474$ \pmt{0.113}& $0.119$ \pmt{0.026} & $0.475$ \pmt{0.113}& $0.123$ \pmt{0.026} & $0.480$ \pmt{0.115} \\ 
Deterministic Hamiltonian (explore) Non-robust & $0.097$ \pmt{0.014} & $0.353$ \pmt{0.075}& $0.097$ \pmt{0.014} & $0.353$ \pmt{0.075}& $0.097$ \pmt{0.014} & $0.354$ \pmt{0.075}& $0.098$ \pmt{0.014} & $0.355$ \pmt{0.075}& $0.101$ \pmt{0.014} & $0.362$ \pmt{0.076}\\ 
Deterministic Hamiltonian (explore) Robust& $0.115$ \pmt{0.023} & $0.477$ \pmt{0.111}& $0.115$ \pmt{0.023} & $0.477$ \pmt{0.111}& $0.115$ \pmt{0.023} & $0.478$ \pmt{0.111}& $0.115$ \pmt{0.023} & $0.478$ \pmt{0.112}& $0.119$ \pmt{0.024} & $0.484$ \pmt{0.113}  \\ 
Pathwise Non-robust & $0.094$ \pmt{0.013} & $0.353$ \pmt{0.072}& $0.094$ \pmt{0.013} & $0.353$ \pmt{0.072}& $0.094$ \pmt{0.013} & $0.353$ \pmt{0.072}& $0.095$ \pmt{0.013} & $0.354$ \pmt{0.072}& $0.099$ \pmt{0.013} & $0.356$ \pmt{0.074}\\ 
Pathwise Robust & $0.120$ \pmt{0.026} & $0.465$ \pmt{0.112}& $0.120$ \pmt{0.026} & $0.465$ \pmt{0.112}& $0.120$ \pmt{0.026} & $0.466$ \pmt{0.112}& $0.121$ \pmt{0.026} & $0.466$ \pmt{0.112}& $0.124$ \pmt{0.026} & $0.471$ \pmt{0.114}\\ 
Pathwise Robust ZO Both & $0.559$ \pmt{0.019} & $1.511$ \pmt{0.034}& $0.559$ \pmt{0.019} & $1.511$ \pmt{0.034}& $0.560$ \pmt{0.019} & $1.513$ \pmt{0.034}& $0.561$ \pmt{0.019} & $1.514$ \pmt{0.034}& $0.570$ \pmt{0.019} & $1.530$ \pmt{0.032}\\ 
Pathwise Robust ZO Inner & $0.105$ \pmt{0.019} & $0.434$ \pmt{0.112}& $0.105$ \pmt{0.019} & $0.434$ \pmt{0.112}& $0.105$ \pmt{0.019} & $0.434$ \pmt{0.112}& $0.106$ \pmt{0.019} & $0.435$ \pmt{0.111}& $0.109$ \pmt{0.019} & $0.437$ \pmt{0.109}\\ 
\hline
\end{tabular}}
\end{subtable}
\end{table}

\paragraph{Effectiveness of non-robust and robust policy optimisation.} The general test performance of policies can be observed in the adversary test in Table~\ref{tab: test performance}, which is based on the hyperparameter settings as shown in Table~\ref{tab: hyperparameters ODE} of Appendix~\ref{app: double-loop details}. The pathwise and deterministic Hamiltonian policy gradient algorithms are effective in both non-robust and robust optimisation, and the best performing algorithm is the robust pathwise optimiser.  The stochastic Hamiltonian policy gradient algorithm performs well without robust optimisation while adding robust optimisation gave inconsistent benefits -- it is possible more iterations are required for consistent convergence across the runs. The traditional, discrete-time algorithm (REINFORCE) yields a comparatively low performance. The zero order optimisers generally perform lower, although if only the adversary is zero order, a reasonable test performance on par with non-robust optimisation can be obtained. The robust policy optimisers do not always obtain the best performance in the randomised test, which considers average-case robustness and soft-robustness rather than worst-case robustness. This is because the worst-case environment is significantly more challenging than the average case, and in line with the CT-RMDP objective, the worst-case is being optimised.

\paragraph{Discretisation effects} 
Sample-and-hold training appears to be nearly equivalent to the training without it (i.e. considering factors involving e.g. $\nabla_x \mu_{\theta}$ to reflect policy changes between time steps). The best score was found in the pathwise robust sample-and-hold based training, perhaps due to it being slightly less complex to learn. Due to the similarity, and slightly improved performance, we only report the sample-and-hold training.

Large improvements are found when comparing continuous-time algorithms to the discrete-time algorithm, even for the nearly equivalent stochastic Hamiltonian gradient algorithm. The large effect is attributed to the continuous-time algorithms taking into account the dynamics in between two time steps to compute the policy gradient more accurately, i.e. based on an accurate approximation of the quadratic cost terms. These effects are likely underestimated because -- in the absence of an adversarial gradient for discrete-time algorithms in deterministic dynamics -- the inner problem was solved with our deterministic Hamiltonian adversary.  

The trained policies themselves do not seem to depend on the discretisation time step at test time. The larger time step has slightly higher performance  due to setting the terminal reward proportional to the time step (consistent with discrete-time algorithms). Beyond this effect, there are only limited effects observed from the test-time time steps. The results shown in Table~\ref{tab: test performance} are based on an adversary trained with the coarsest time step ($\Delta t = 0.05$). However, we also conducted preliminary experiments where the tests separately trained the adversaries for each time step; such adversarial tests were found to be unreliable, with the adversary performing lower than some empirical maxima, which can be attributed to the 100 iterations being insufficient to optimise the much more complex optimisation problems. We also confirm that with larger number of iterations inversely proportional to the time step, the performance still remains similar across time steps.

\begin{table}[htbp!]
    \centering
\caption{Normalised adversarial test performance of the best policy under different MFL-SDA optimisers in the LQR domain. For each discretisation time step ($\Delta t$), the lowest cost is selected out of the $N=25$ policies and it is aggregated across 5 independent runs (mean $\pm$ standard error). The bold entries highlight the best-performing algorithm for each discretisation time step. The normalisation is consistent with Table~\ref{tab: test performance}, allowing direct comparison.}
\label{tab: test performance MFL-SDA}
\begin{tabular}{l|*{5}{c}}
\toprule \\ 
& \multicolumn{1}{l}{$\Delta t = 0.0005$}& \multicolumn{1}{l}{$\Delta t = 0.001$}& \multicolumn{1}{l}{$\Delta t = 0.005$}& \multicolumn{1}{l}{$\Delta t = 0.01$}& \multicolumn{1}{l}{$\Delta t = 0.05$}\\
\midrule\\  
Deterministic Hamiltonian Robust& $\mathbf{0.198}$ \pmt{0.037} & $\mathbf{0.198}$ \pmt{0.037} & $\mathbf{0.200}$ \pmt{0.037} & $\mathbf{0.202}$ \pmt{0.037} & $0.228$ \pmt{0.038}\\ 
Pathwise Robust& $0.205$ \pmt{0.048} & $0.205$ \pmt{0.048} & $0.206$ \pmt{0.048} & $0.207$ \pmt{0.048} & $\mathbf{0.218}$ \pmt{0.047} \\ 
Discrete Robust & $0.481$ \pmt{0.139} & $0.481$ \pmt{0.139} & $0.481$ \pmt{0.138} & $0.481$ \pmt{0.138} & $0.485$ \pmt{0.132} \\ 
\bottomrule
\end{tabular}
\end{table}

\begin{figure}[htbp!]
    \centering
    \includegraphics[width=0.99\linewidth]{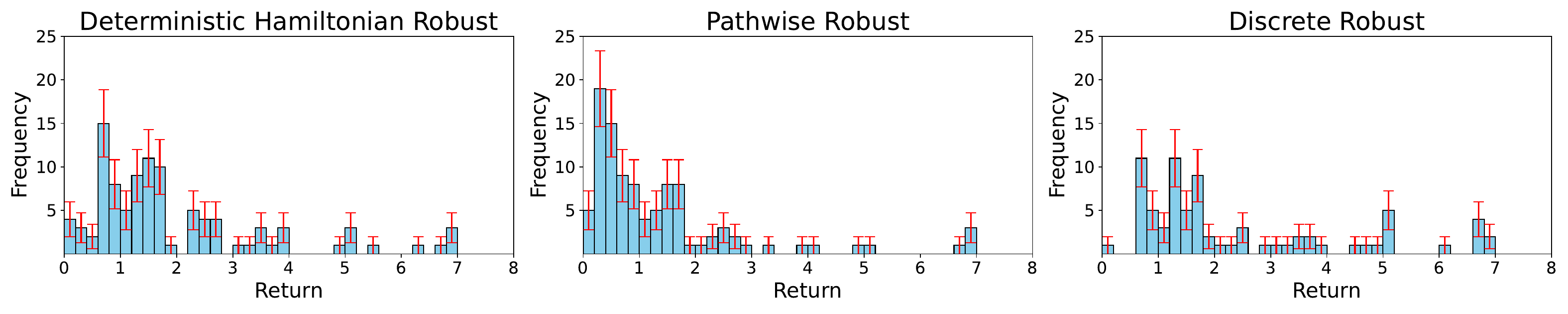}
    \caption{Histogram of the test performance of the policies obtained by the distributional optimisers in the LQR domain.  The $y$-axis indicates the total frequency across all policies for all seeds. The error bars indicate $\sqrt{n}$, where $n = \text{seeds} \times \text{N}$, in line with a Poisson distribution assumption. The discretisation time step is $\Delta t = 0.0005$.}
    \label{fig: histogram}
\end{figure}

\paragraph{Distributional optimisation with MFL-SDA} 
In addition to the above-mentioned double-loop algorithm experiments, we also present MFL-SDA based experiments. These experiments are conducted in the same environment with the same adversary test. We conduct experiments with $N=25$ particles for the policies and the adversaries to form the particle-based gradient estimates as in \Cref{eq: particle approximation policy,eq: particle approximation adversary}. To limit computational expense, we only evaluate Discrete Robust, Deterministic Hamiltonian Robust, and Pathwise Robust, and perform only 50 iterations. The hyperparameter settings are shown in Table~\ref{tab: hyperparameters ODE} of Appendix~\ref{app: double-loop details}. The results in Table~\ref{tab: test performance MFL-SDA} show that the best policy in the set of particles is invariably better-performing than the obtained double-loop algorithm policies, showcasing a solid global optimisation property. Consistent with the previous results, the discrete algorithm is outperformed by the continuous-time algorithms. The deterministic Hamiltonian and pathwise algorithm are comparable in performance with the Hamiltonian algorithm performing slightly better for small time steps and the pathwise performing slightly better for $\Delta t = 0.05$. The test performance varies widely (Figure~\ref{fig: histogram}), and we note that more narrow distributions can potentially be obtained using a lower regularisation constant $\tau$. Overall, the pathwise gradient obtains the best scores in few-shot adaptation and in the distributional sense.

\subsection{High-dimensional Duffing Oscillator Neural ODE domain }
The next benchmark is a controlled Duffing oscillator under the dynamics 
\begin{align*}
\frac{\dif v}{\dif t}+\delta\frac{\dif q}{\dif t}+\alpha q+\beta x^3 = u \,,
\end{align*}
which includes nonlinear dynamics over the velocity $v$ as a function of the position $q$ and the action $u$.\\
The nominal dynamics of the benchmark is an oscillator that is parametrised according to
\begin{align*}
\frac{\dif x}{\dif t}=\bar{f}(t,x,u) \,,
\end{align*}
where $x = (q, v)^{\intercal}$.
The system may have potentially unknown deviations from the model according to 
\begin{align*}
\frac{\dif x}{\dif t}= \bar{f}(x,u) + g_{\xi}(t,x,u) \,,
\end{align*}
for $\xi \in \Xi$. The system has a double well, and the task of the agent is to move from one well at $-\mathbf{1}$ to the well at $\mathbf{1}$. To assess the capability of the algorithms in high-dimensional problems, the experiments consider a 200-dimensional system. \\
The deviation from nominal $g_{\xi}$ is an MLP with inputs consisting of the time, the position, and the velocity, three hidden layers of size 64 each with $\tanh $ activations, and the 100-dimensional output which represents the acceleration due to unmodelled forces not considered in the oscillator model. The cost function is given by 
\begin{align*}
J(\theta,\xi) = \Ex \left[ \int_0^T \left( c_x (x_t - x_{\text{tgt}})^2 + c_u u_t^2 \right) \dif t \vert \pi_{\theta}, f_{\xi} \right],
\end{align*}
where $c_x, c_u > 0$ and $x_{\text{tgt}}$ is the target state, which concatenates the $\mathbf{1}$ vector for the position of the target well and the $\mathbf{0}$ vector for the velocity.

The experiments include only deterministic gradient algorithms due to runtime considerations. The hyperparameter settings are included in Table~\ref{tab: hyperparameters oscillator}. The results in Table~\ref{tab: test performance Duffing} confirm the effect of the robust algorithms as opposed to their non-robust counterparts. Except for a single outlier in the Deterministic Hamiltonian Robust algorithm runs, he robust algorithms present a 3- to 4-fold improvement in average test performance over their non-robust counterparts.

\begin{table}[htbp!]
    \centering
\caption{Normalised performance (mean $\pm$ standard error) on adversarial test (top) and randomised test (bottom) for each discretisation time step ($\Delta t$), aggregated across 5 independent runs, on the Duffing Oscillator. The $^*$ superscript is used to emphasise that the result is from a single outlier run while removing the outlier would give a performance similar to Pathwise Robust. }
\label{tab: test performance Duffing}
\begin{subtable}[t]{0.99\textwidth}
\caption{Adversary test.}
\label{tab: adversary test performance  Duffing}
\resizebox{\textwidth}{!}{
\begin{tabular}{l|*{5}{c}}
\toprule \\ 
& \multicolumn{1}{l}{$\Delta t = 0.0005$}& \multicolumn{1}{l}{$\Delta t = 0.001$}& \multicolumn{1}{l}{$\Delta t = 0.005$}& \multicolumn{1}{l}{$\Delta t = 0.01$}& \multicolumn{1}{l}{$\Delta t = 0.05$}\\
\midrule\\  
Deterministic Hamiltonian Non-robust& $2.477$ \pmt{0.099} & $2.490$ \pmt{0.100} & $2.600$ \pmt{0.107} & $2.752$ \pmt{0.117} & $4.823$ \pmt{0.251} \\ 
Deterministic Hamiltonian Robust& $1.113^*$ \pmt{0.353} & $1.117^*$ \pmt{0.356} & $1.144^*$ \pmt{0.378} & $1.183^*$ \pmt{0.409} & $1.805^*$ \pmt{0.920} \\ 
Pathwise Non-robust& $2.555$ \pmt{0.101} & $2.569$ \pmt{0.102} & $2.686$ \pmt{0.109} & $2.847$ \pmt{0.120} & $5.109$ \pmt{0.294} \\ 
Pathwise Robust& $0.727$ \pmt{0.006} & $0.728$ \pmt{0.006} & $0.730$ \pmt{0.006} & $0.734$ \pmt{0.006} & $0.778$ \pmt{0.003} \\ 
\bottomrule
\end{tabular}}
\end{subtable}
\begin{subtable}[t]{0.99\textwidth}
\caption{Randomised test.}
\label{tab: randomised test performance Duffing}
\resizebox{\textwidth}{!}{
\begin{tabular}{l|*{10}{c}}
\toprule \\ 
Deterministic Hamiltonian Non-robust& $0.159$ \pmt{0.001} & $0.162$ \pmt{0.001}& $0.159$ \pmt{0.001} & $0.162$ \pmt{0.001}& $0.159$ \pmt{0.001} & $0.162$ \pmt{0.001}& $0.160$ \pmt{0.001} & $0.162$ \pmt{0.001}& $0.162$ \pmt{0.001} & $0.164$ \pmt{0.001}\\ 
Deterministic Hamiltonian Robust& $0.200$ \pmt{0.002} & $0.203$ \pmt{0.002}& $0.200$ \pmt{0.002} & $0.203$ \pmt{0.002}& $0.201$ \pmt{0.002} & $0.203$ \pmt{0.002}& $0.201$ \pmt{0.002} & $0.204$ \pmt{0.002}& $0.204$ \pmt{0.002} & $0.207$ \pmt{0.002}\\ 
Pathwise Non-robust& $0.159$ \pmt{0.001} & $0.162$ \pmt{0.001}& $0.159$ \pmt{0.001} & $0.162$ \pmt{0.001}& $0.159$ \pmt{0.001} & $0.162$ \pmt{0.001}& $0.159$ \pmt{0.001} & $0.162$ \pmt{0.001}& $0.162$ \pmt{0.001} & $0.164$ \pmt{0.001}\\ 
Pathwise Robust& $0.208$ \pmt{0.002} & $0.212$ \pmt{0.002}& $0.208$ \pmt{0.002} & $0.212$ \pmt{0.002}& $0.209$ \pmt{0.002} & $0.212$ \pmt{0.002}& $0.209$ \pmt{0.002} & $0.213$ \pmt{0.002}& $0.211$ \pmt{0.002} & $0.215$ \pmt{0.002}\\   
\end{tabular}}
\end{subtable}
\end{table}

\subsection{SDE experiments}
As a last set of experiments we also show the capability of the techniques on SDEs. We apply a ridge regression BSDE estimator on a stochastic linear-quadratic system for which the exact adjoint pair is available. Apart from system parameters, the problem is similar to the aforementioned LQR problem; however, now the policy is a linear function, i.e. $u_t = \theta x_t$ for $\theta \in \Reals^{d_u\times d_x}$, and zero-mean Gaussian noise is added to the state at each time step. The experiments took 50 time steps within the time interval $[0,1]$ and was repeated for $\{1,10,100,1000\}$ trajectories.  Examining the relative errors of the policy and adversary gradients (see Figure~\ref{fig: SDE}), our BSDE estimator performs at least an order of magnitude better than the central differencing estimator. This finding demonstrates that the second adjoint ($Q_t$) need not be supplied analytically and can be recovered from stochastic rollouts using standard conditional-expectation BSDE estimation.

\begin{figure}
    \centering
    \begin{subfigure}[b]{0.49\textwidth}
    \includegraphics[width=0.9\linewidth]{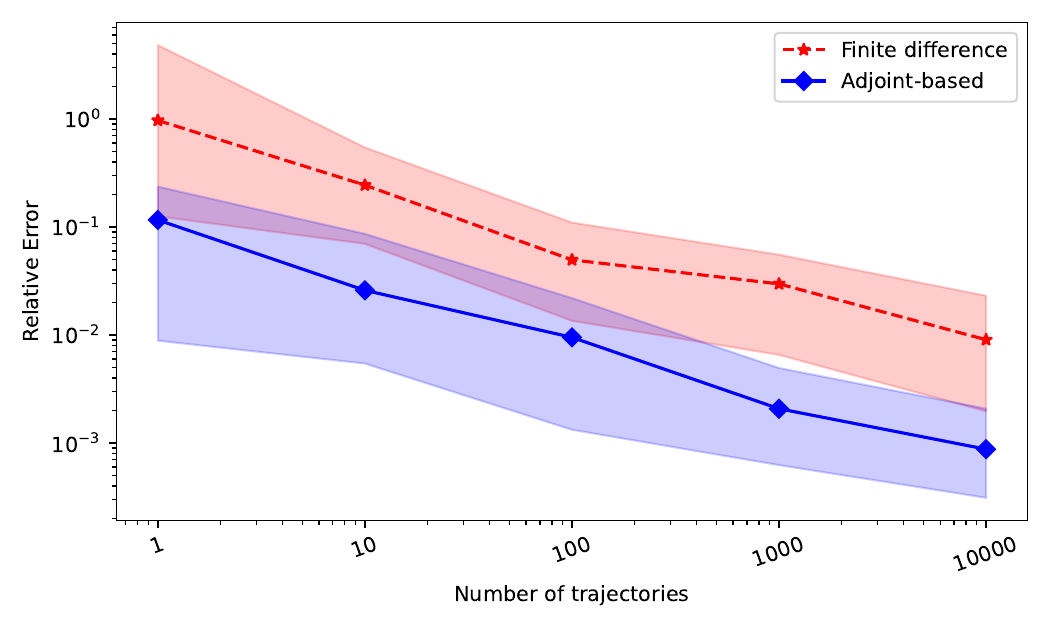}
    \caption{Policy}
    \end{subfigure}
    \begin{subfigure}[b]{0.49\textwidth}
    \includegraphics[width=0.9\linewidth]{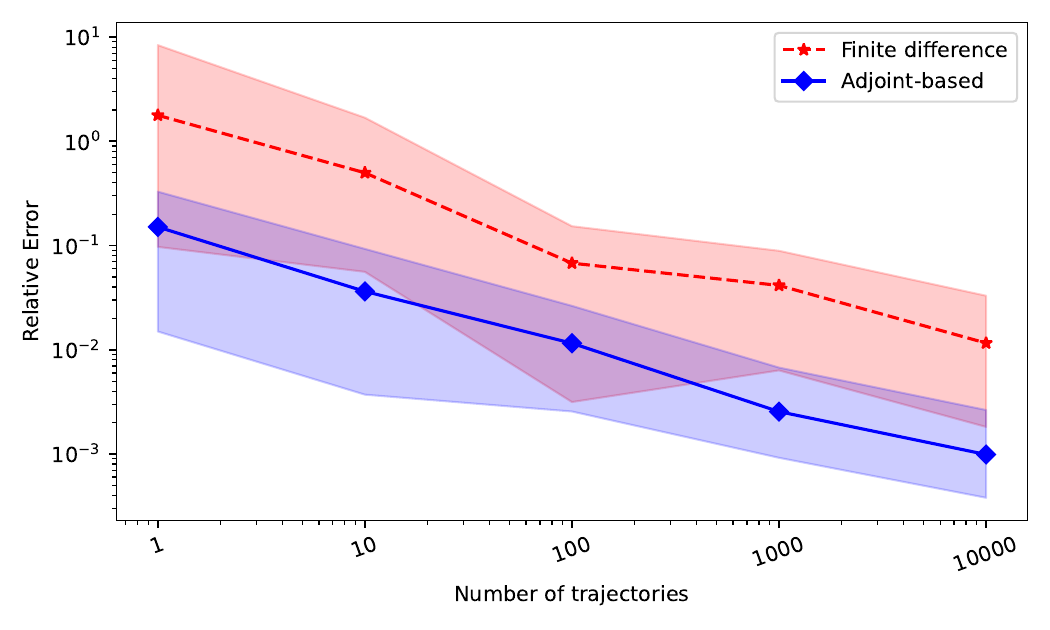}
    \caption{Adversary}
    \end{subfigure}
    \caption{Relative error of gradient estimates as a function of the number of sampled trajectories. The solid/dashed lines indicate the mean while the shaded areas indicate the minimum and maximum. The finite difference method is based on central differencing with a fixed-magnitude $\pm \mathrm{e}^{-5}$ perturbations for each parameter dimension.  }  
    \label{fig: SDE}
\end{figure}

\section{Limitations}
Our paper presents theoretical and empirical results that come with some limitations. The applicability of the convergence rates has limitations in its scope since there are different PL conditions as well as an assumption on the convergence of the transition dynamics. The SDE formulation currently relies on an adjoint state oracle and our convergence analysis does not consider the estimation error in this regard. The derivations for the MFL-SDA approach are conditional as contractivity/dissipativity assumptions are required for  uniform-in-time finite-particle propagation of chaos. Our empirical results may not extrapolate to other settings and need to be interpreted cautiously. For instance, in the best-of-N evaluation in the mean-field optimisation algorithm, a higher (lower) setting of $N$ will obtain a better (worse) optimum with the same number of iterations. 

\section{AI Use Statement}
The paper made use of generative AI for code prototyping, correction of proofs, and as an additional search engine. All the final written material is written by and approved by the authors.

\section{Conclusion}
This paper addresses the problem of continuous-time algorithms for RMDPs, focusing specifically on policy gradient algorithms. We first derive analytical gradient formulas for the policy and the adversary in the context of ODEs, random ODEs, and SDEs. We then derive double-loop optimization and mean-field optimization based regret bounds and investigate double-loop optimization empirically in a robust linear quadratic regulator domain. Through an empirical study, the effectiveness of stochastic and deterministic continuous-time policy gradient algorithms is confirmed, and a large improvement is found compared to discrete-time optimization and zero-order optimization. The empirical study also shows mean-field optimization as a powerful framework to achieve improved robust optimizers. Additional contributions are the tools derived to analyse rectangularity, gradient dominance, discretisation, and properties of double-loop algorithms.

\bibliographystyle{plainnat}
\newpage
\bibliography{references.bib} 
\appendix 
 \section{Supporting lemmata}\label{sec:appendix}
\subsection{General lemmata}
The discrete-time policy gradient can be derived for random ODEs according to the following lemma.
\begin{lemma}[Discrete-time policy gradient]
\label{lem: discrete-time policy gradient}
Let $h$ be the discretisation step size and $t_n \: nh$. Recall that we work under a stochastic policy $u_n \sim \pi_\theta(\cdot|t_n, x_n)$ for the discrete cost $J_h(\theta, \xi) = \Ex \left[ \sum_{n=1}^{N-1} r(t_n, x_n, u_n)h + R(x_N)\right]$. Let the trajectory be $\tau = (x_0, u_0, x_1, u_1, \ldots, x_N)$ obtained under a random ODE. Then the discrete-time policy gradient is given by 
\begin{equation}
\label{eq: discrete-time policy gradient}
\nabla_\theta J_h(\theta, \xi) = \Ex \left[ \sum_{n=0}^{N-1} \nabla_\theta \log \pi_\theta (u_n|t_n, x_n) Q_h(t_n, x_n, u_n)\right] \,,
\end{equation}
where $Q_h(t_n,x_{t_n},u_n) \: \Ex[G(\tau) \vert t_n, x_n, u_n]$.
\end{lemma}
 \begin{proof} 
Under the random ODE assumption of the state dynamics, the only $\theta$-dependence in the trajectory distribution is though the policy, since the state transition is deterministic given $u_n$. Thus, we can decompose the trajectory density w.r.t. the action product measure as
\[p_\theta(\tau) = d_0(x_0) \Pi_{n=0}^{N-1}\pi_\theta(u_n, t_n, x_n),\]  where $d_0$ is the initial distribution. Let the total return be $G(\tau) = \sum_{n=0}^{N-1} r(t_n, x_n, u_n)h + R(x_N)$. Then,
\[J_h(\theta, \xi) = \int G(\tau) p_\theta(\tau) \dif \tau.\]
Differentiating w.r.t. $\theta$, 
\[\nabla_\theta J_h = \int G(\tau) \nabla_\theta p_\theta(\tau) \dif \tau = \int G(\tau) p_\theta(\tau) \nabla_\theta \log p_\theta(\tau)\dif \tau\]
which is simply \(\Ex_{\tau}[G(\tau) \nabla_\theta \log p_\theta(\tau)]\).
It follows that \(\nabla_\theta \log p_\theta(\tau) = \sum_{n=0}^{N-1} \nabla_\theta \log \pi_\theta(u_n, t_n, x_n)\), implying \[\nabla_\theta J_h(\theta, \xi) = \Ex \left[ \sum_{n=0}^{N-1} \nabla_\theta \log \pi_\theta(u_n, t_n, x_n) G(\tau)\right].\]
Now for the policy gradient, we simply replace $G(\tau)$ for a cost-to-go function:
$Q_h(t_n,x_{t_n},u_n) \: \Ex[G(\tau) \vert t_n, x_{t_n}, u_n]$. By using properties of the conditional expectation, we obtain
 \[\Ex[\nabla_\theta \log \pi_\theta(u_n| t_n, x_n) G(\tau) \vert t_n,x_{t_n},u_n] = \nabla_\theta \log \pi_\theta(u_n| t_n, x_{t_n}) \Ex[G(\tau)| t_n, x_{t_n}, u_n],\]
which is equivalent to $\Ex[\nabla_\theta \log \pi_\theta(u_n| t_n, x_n) Q_h(t_n, x_n, u_n)]$.
So the standard discrete-time identity is
\[\nabla_\theta J_h(\theta, \xi) = \Ex \left[ \sum_{n=0}^{N-1} \nabla_\theta \log \pi_\theta (u_n|t_n, x_n) Q_h(t_n, x_n, u_n)\right].\]
\end{proof}

Below is a standard descent lemma which is convenient in the regret analysis.
\begin{lemma}[Standard descent lemma]
    \label{lem: standard descent}
Let \(J(\theta)\) be \(L\)-smooth. Then, for any \(0 \leq \eta \leq 1/L\), we have 
\begin{equation*}
    J(\theta_{t+1}) \leq J(\theta_t) - \frac{\eta}{2}\norm{\nabla J(\theta)}^2 \,.
\end{equation*} 
\end{lemma}
\begin{proof}
    This is a well-known result, and we include a brief sketch. Let \(J(\theta)\) be \(L\)-smooth. Then, for any \(\Delta \theta\), 
    \begin{equation*}
        J(\theta + \Delta \theta) \leq J(\theta) + \langle  \nabla_\theta J(\theta), \Delta \theta\rangle + \frac{L}{2} \norm{\Delta \theta}^2 \,,
    \end{equation*}
    using a first-order Taylor expansion with a Lagrangian remainder bounded by $\frac{1}{2} \nabla^2 J(\theta_c) (\Delta \theta)^2 \leq \frac{L}{2} \norm{\Delta \theta}^2$ where $\theta_c \in [\theta,\theta+\Delta \theta]$. The result follows from setting \(\Delta \theta = - \eta \nabla J(\theta)\) and noting that \(\eta \leq 1/L\) implies \(1 - L \eta/2 \geq 1/2\).
\end{proof}

\subsection{Validation of numerical adjoint estimation under stochastic dynamics}
\label{sec:sde_validation}
To complement the SDE gradient theory, we consider a
stochastic linear-quadratic benchmark for which the adjoint processes
$(P_t,Q_t)$ are independently available. The objective is to assess
whether the martingale costate $Q_t$, which is assumed to be exactly
available in the oracle analysis, can instead be recovered from simulated
stochastic trajectories using a standard numerical BSDE approximation.
\paragraph{Dynamics and uncertainty.}
Let $X_t\in\mathbb R^{d_x}$, $U_t\in\mathbb R^{d_u}$ and
$W_t\in\mathbb R^{d_w}$. We consider
\begin{equation}
\label{eq:sde_val_dynamics}
    dX_t
    =
    (A_\xi X_t+BU_t)\,dt
    +
    \sum_{\ell=1}^{d_w}
    (D_{\xi,\ell}X_t+E_\ell U_t)\,dW_t^\ell,
\end{equation}
where
\begin{equation}
\label{eq:sde_val_uncertainty}
    A_\xi
    =
A_0+\sum_{j=1}^{d_\xi}\xi_jA_j,
    \qquad
    D_{\xi,\ell}
    =
    D_{0,\ell}
    +
\sum_{j=1}^{d_\xi}\xi_jD_{j,\ell},
\end{equation}
and $\xi\in\Xi$ is constrained to a compact uncertainty set.
We use a deterministic linear feedback policy
\begin{equation}
\label{eq:sde_val_policy}
    U_t=KX_t.
\end{equation}
For fixed $(K,\xi)$ define
\[
    A_{\mathrm{cl}}
    :=
    A_\xi+BK,
    \qquad
    G_{\xi,\ell}
    :=
    D_{\xi,\ell}+E_\ell K.
\]
The closed-loop process is therefore
\begin{equation}
\label{eq:sde_val_closed_loop}
    dX_t
    =
    A_{\mathrm{cl}}X_t\,dt
    +
    \sum_{\ell=1}^{d_w}
G_{\xi,\ell}X_t\,dW_t^\ell.
\end{equation}
\paragraph{Quadratic cost.}
We consider
\begin{equation}
\label{eq:sde_val_cost}
    J(K,\xi)
    =
    \mathbb E
    \left[
        \int_0^T
        \left(
            X_t^\top CX_t
            +
            U_t^\top R_uU_t
        \right)dt
        +
        X_T^\top C_TX_T
    \right],
\end{equation}
with
\[
    C\succeq0,\qquad
    R_u\succ0,\qquad
    C_T\succeq0.
\]
\paragraph{Reference value and adjoint.}
For a fixed feedback matrix $K$, the policy-evaluation value function
takes the quadratic form
\begin{equation}
\label{eq:sde_val_value}
    V^{K,\xi}(t,x)
    =
    x^\top S_tx,
\end{equation}
where $S_t$ solves the backward generalized Lyapunov equation
\begin{equation}
\label{eq:sde_val_lyapunov}
\begin{aligned}
    -\dot S_t
    ={}&
    C
    +
    K^\top R_uK
    +
    A_{\mathrm{cl}}^\top S_t
    +
    S_tA_{\mathrm{cl}}
    \\
    &+
    \sum_{\ell=1}^{d_w}
    G_{\xi,\ell}^\top
    S_t
    G_{\xi,\ell},
    \qquad
    S_T=C_T.
\end{aligned}
\end{equation}
Consequently, the reference adjoint processes are
\begin{equation}
\label{eq:sde_val_P_ref}
    P_t^{\mathrm{ref}}
    =
    \nabla_xV^{K,\xi}(t,X_t)
    =
    2S_tX_t,
\end{equation}
and, for each Brownian component,
\begin{equation}
\label{eq:sde_val_Q_ref}
    Q_{t,\ell}^{\mathrm{ref}}
    =
    2S_tG_{\xi,\ell}X_t.
\end{equation}
Thus
\[
    Q_t^{\mathrm{ref}}
    =
    \left[
        Q_{t,1}^{\mathrm{ref}},
        \ldots,
Q_{t,d_w}^{\mathrm{ref}}
    \right]
\]
provides an independent reference for evaluating a numerical BSDE
solver.
\paragraph{Closed-loop adjoint.}
Define
\[
\begin{aligned}
H(t,x,u,P,Q;\xi)
={}&
x^\top Cx
+
u^\top R_uu
+
P^\top(A_\xi x+Bu)
\\
&+
\sum_{\ell=1}^{d_w}
Q_\ell^\top
(D_{\xi,\ell}x+E_\ell u).
\end{aligned}
\]
For the feedback $u=Kx$, the corresponding closed-loop adjoint satisfies
\begin{equation}
\label{eq:sde_val_closedloop_bsde}
    dP_t
    =
    -
    \left(
        \nabla_xH
        +
        K^\top\nabla_uH
    \right)dt
    +
    \sum_{\ell=1}^{d_w}
    Q_{t,\ell}\,dW_t^\ell,
    \qquad
    P_T=2C_TX_T.
\end{equation}
Equivalently,
\begin{equation}
\label{eq:sde_val_closedloop_bsde_expanded}
\begin{aligned}
    dP_t
    ={}&
    -
    \Bigg[
        2(C+K^\top R_uK)X_t
        +
        A_{\mathrm{cl}}^\top P_t
        +
        \sum_{\ell=1}^{d_w}
        G_{\xi,\ell}^\top Q_{t,\ell}
    \Bigg]dt
    +
    \sum_{\ell=1}^{d_w}
    Q_{t,\ell}\,dW_t^\ell.
\end{aligned}
\end{equation}
Substituting
\eqref{eq:sde_val_P_ref}--\eqref{eq:sde_val_Q_ref}
into \eqref{eq:sde_val_closedloop_bsde_expanded} recovers exactly
\eqref{eq:sde_val_lyapunov}.
\paragraph{Numerical BSDE approximation.}
We simulate the state using Euler--Maruyama,
\begin{equation}
\label{eq:sde_val_EM}
\begin{aligned}
    X_{n+1}^{(m)}
    ={}&
    X_n^{(m)}
    +
    hA_{\mathrm{cl}}X_n^{(m)}+
    \sum_{\ell=1}^{d_w}
    G_{\xi,\ell}X_n^{(m)}
    \Delta W_{n,\ell}^{(m)},
\end{aligned}
\end{equation}
where
\[
    \Delta W_n^{(m)}
    \sim\mathcal N(0,hI_{d_w}).
\]
Starting from
\[
    \widehat P_N^{(m)}
    =
    2C_TX_N^{(m)},
\]
we estimate the martingale costate through the discrete conditional
expectation
\begin{equation}
\label{eq:sde_val_Q_est}
    \widehat Q_n(x)
    \approx
    \frac1h
    \mathbb E
    \left[
        \widehat P_{n+1}
        \Delta W_n^\top
        \,\middle|\,
        X_n=x
    \right].
\end{equation}
For the present linear-quadratic system the exact map is linear in the
state, and hence we parameterise
\[
    \widehat Q_{n,\ell}(x)
    =
    L_{n,\ell}x.
\]
The coefficients $L_{n,\ell}$ can be obtained by least-squares regression
using the targets
\[
    \frac1h
    \widehat P_{n+1}^{(m)}
    \bigl(\Delta W_n^{(m)}\bigr)^\top.
\]
The first adjoint may be estimated simultaneously by fitting a linear
model
\[
    \widehat P_n(x)=M_nx
\]
to the one-step BSDE relation
\begin{equation}
\label{eq:sde_val_bsde_residual}
\begin{aligned}
\mathcal L_n
:=
\frac1M\sum_{m=1}^M
\Bigg\|
&
\widehat P_{n+1}^{(m)}
-
\widehat P_n(X_n^{(m)})
+
h\widehat F_n^{(m)}
-
\widehat Q_n(X_n^{(m)})
\Delta W_n^{(m)}
\Bigg\|^2,
\end{aligned}
\end{equation}
where $\widehat F_n$ denotes the BSDE driver evaluated at the estimated
$(\widehat P_n,\widehat Q_n)$.
\paragraph{Gradient validation.}
The control derivative of the Hamiltonian is
\begin{equation}
\label{eq:sde_val_Hu}
    \nabla_uH
    =
    2R_uU_t
    +
    B^\top P_t
    +
    \sum_{\ell=1}^{d_w}
    E_\ell^\top Q_{t,\ell}.
\end{equation}
With the closed-loop adjoint, the gradient with respect to $K$ is
\begin{equation}
\label{eq:sde_val_K_gradient}
    \nabla_KJ(K,\xi)
    =
    \mathbb E
    \int_0^T
    \nabla_uH(t,X_t,U_t,P_t,Q_t;\xi)
    X_t^\top\,dt.
\end{equation}
Furthermore, since
\[
    \partial_{\xi_j}A_\xi=A_j,
    \qquad
    \partial_{\xi_j}D_{\xi,\ell}=D_{j,\ell},
\]
the adversarial gradient satisfies
\begin{equation}
\label{eq:sde_val_xi_gradient}
\frac{\partial J}{\partial\xi_j}
=
\mathbb E
\int_0^T
\left[
    P_t^\top A_jX_t
    +
    \sum_{\ell=1}^{d_w}
    Q_{t,\ell}^\top D_{j,\ell}X_t
\right]dt.
\end{equation}
The second term directly measures the contribution of the numerically
estimated martingale costate $Q_t$ to the robust SDE gradient.

\subsection{Lemmata for the regret analysis of the double-loop algorithm}
The lemma below, obtained from Theorem~4.2 of \cite{Wang2024}, shows an upper bound can be established based on the iterates. 
\begin{lemma}[Upper bound on RMDP regret,  \cite{Wang2024}]
\label{lem: robust-objective regret}
Let $\Phi$ be the objective in \eqref{eq: RMDP} and let $\theta^* \in \argmin_{\theta \in \Theta} \Phi(\theta)$. Then for any macro-step $k \geq 1$,
\begin{align}
\label{eq: robust-objective regret}
\Phi(\theta_{k}) - \Phi(\theta^*)  &\leq J(\theta_k,\xi_k) + \epsilon_k  - J(\theta^*,\xi_k) \,,
\end{align}
where $\epsilon_k > 0$ is the error tolerance for the inner problem.
\end{lemma}

Performance difference lemmas are used to bound the value difference between two policies in discounted MDPs. Note that in discrete-time with $\Delta t \to 0$, our formulation is equivalent to the total cost MDP with states after the terminal state being absorbing states with cost 0. Hence we can use the so-called transient performance difference lemma for our formulation.
\begin{lemma}[Transient performance difference lemma, \cite{Lee2026}]
\label{lem: performance difference}
For a total cost MDP $\cM = (x_0,\Pr,r,\cX,\cU)$, where $x_0$ is the starting state, $\Pr $ is the transition kernel, $r$ is the cost function, and $\cX$ and $\cU$ are the state and action spaces. Let $\cM$ have a finite value function in the sense of Assumption 2.1 in \cite{Lee2026}. 
Then if $V(x \vert \pi') = 0$ for all recurrent states, the performance difference of two policies $\pi,\pi' \in \Pi$ is given by
\begin{equation}
V(x_0 \vert \pi) - V(x_0 \vert \pi') = \sum_{x \in \cX} \tilde{d}_{x_0,\Pr}^{\pi}(x) \sum_{u \in \cU} (\pi(u \vert x) - \pi'(u \vert x) ) Q(x,u \vert \pi') \,,
\end{equation}
where $\tilde{d}_{x_0,p}^{\pi}$ is the transient visitation measure as defined in Definition~\ref{def: visitation measure}.
\end{lemma}

Using the above transient visitation measure also leads to recovering results from discounted MDPs, in particular the performance difference lemmas for transition kernels \citep{Wang2024,Li2026}. 
\begin{lemma}[Transient performance difference lemma for transition kernels]
\label{lem: performance difference transition kernel}
Let $\theta \in \Theta$ and let $\Pr,\Pr' \in \cP$. Under the conditions and definitions of Lemma~\ref{lem: performance difference}, we have that
\begin{align*}
J(\theta,\Pr) - J(\theta,\Pr') = \sum_{x \in \cX} \tilde{d}_{d_0,\Pr'}^{\pi}(x) \sum_{(u,x')} \pi(u \vert x) V(x' \vert \Pr) \left(\Pr(x' \vert x,u) -  \Pr'(x' \vert x,u) \right) \,.    
\end{align*}
\end{lemma}
\begin{proof}
Let $\pi = \pi_{\theta}$. Then the value obtained under $\pi$ and transition dynamics $p,p' \in \cP$ satisfies
\begin{align*}
    & V(x \vert \Pr)- V(x \vert \Pr') \\
    &= \sum_{u,x'} V(x' \vert \Pr) \pi(u \vert x) \Pr(x' \vert x,u) - \sum_{(u,x')} V(x' \vert \Pr') \pi(u \vert x) \Pr'(x' \vert x,u) \\
     &= \sum_{(u,x')} \pi(u \vert x) V(x' \vert \Pr) \left(  \Pr(x' \vert x,u) -  \Pr'(x' \vert x,u) \right)  + \sum_{(u,x')} \pi(u \vert x) \Pr'(x \vert x,u) (V(x' \vert \Pr) - V(x' \vert \Pr')) \\
    &= \sum_{(u,x')} \pi(u \vert x) V(x' \vert \Pr) \left(\Pr(x' \vert x,u) -  \Pr'(x' \vert x,u) \right)  + \sum_{(u,x')} \pi(u \vert x) \Pr'(x \vert x,u) (V(x' \vert \Pr) - V(x' \vert \Pr')) \,.
\end{align*}
 
Let $u(x) = \sum_{u,x'} \pi(u \vert x)  V(x' \vert \Pr) \left(\Pr(x' \vert x,u) -  \Pr'(x' \vert x,u) \right)$.  The above implies
\begin{align*}
V(\cdot \vert \Pr) - V(\cdot \vert \Pr') &= u +  \cT^{\pi,\Pr'} (V(\cdot \vert \Pr) - V(\cdot \vert \Pr'))
\end{align*}
such that 
\begin{align*}
V(\cdot \vert \Pr) - V(\cdot \vert \Pr') &= (\mathbb{I} - \cT^{\pi,\Pr'})^{-1} u
\end{align*}
and 
\begin{align*}
J(\theta,\Pr) - J(\theta,\Pr') &= d_0^{\intercal} (\mathbb{I} - \cT^{\pi,\Pr'})^{-1} u  \\
					       &= \sum_{x \in \cX} d_0(x) \tilde{d}_{x,\Pr'}^{\pi}(x) \sum_{(u,x')} \pi(u \vert x) V(x' \vert \Pr) \left(\Pr(x' \vert x,u) -  \Pr'(x' \vert x,u) \right) \\
					       &= \sum_{x \in \cX} \tilde{d}_{d_0,\Pr'}^{\pi}(x) \sum_{(u,x')} \pi(u \vert x) V(x' \vert \Pr) \left(\Pr(x' \vert x,u) -  \Pr'(x' \vert x,u) \right) \,.
\end{align*}
\end{proof}

The policy gradient for total cost MDPs is given as follows.
\begin{lemma}
\label{lem: gradient total cost MDP}
For a total cost MDP $\cM = (d_0,\Pr,r,\cX,\cU)$ the policy gradient for a directly parametrised $\pi = \theta$ is given by 
\begin{align}
\label{eq: direct gradient}
\partial_{x',u'} J(\theta,\xi) &= \tilde{d}_{d_0,\Pr_{\xi}}^{\pi_{\theta}}(x') Q(x',u' \vert \pi_{\theta'}) \,. 
\end{align}
and for a general parametrisation by 
\begin{align}
\label{eq: general gradient}
\nabla_{\theta} J(\theta,\xi) &= \sum_{x \in \cX} \tilde{d}_{d_0,\Pr_{\xi}}^{\pi_{\theta}}(x) \sum_{u \in \cU} \pi_{\theta}(u \vert x)  Q(x,u \vert \pi_{\theta'}) \frac{\nabla_{\theta}  \pi_{\theta}(u\vert x)}{\pi_{\theta}(u \vert x)} \,. 
\end{align}
\end{lemma}
\begin{proof}
Applying the policy gradient definition (e.g. Theorem 4.7 in \cite{Lee2026}) to the general parametrisation gives the result in \Cref{eq: general gradient}. Under the  direct parametrisation, we can derive that
\begin{align*}
\partial_{x',u'} J(\theta,\xi) &= \sum_{x \in \cX} \tilde{d}_{d_0,\Pr_{\xi}}^{\pi_{\theta}}(x) \sum_{u \in \cU} \pi_{\theta}(u \vert x)  Q(x,u \vert \pi_{\theta'}) \frac{\nabla_{\theta}  \pi_{\theta}(u\vert x)}{\pi_{\theta}(u \vert x)}) \mathbb{I}((x,u) = (x',u')) \nonumber \\
                     &= \tilde{d}_{d_0,\Pr_{\xi}}^{\pi_{\theta}}(x') Q(x',u' \vert \pi_{\theta'}) \,. 
\end{align*}
\end{proof}

The adversary gradient for a direct parametrisation is given by the below lemma.
\begin{lemma}[Adversary gradient for total cost MDP]
\label{lem: direct adversary gradient total cost MDP}
For a total cost MDP $\cM = (d_0,p,r,\cX,\cU)$ with directly parametrised $p$, we have 
\begin{align}
\partial_{(x,u,x')} J(\theta,\Pr) &= \tilde{d}_{d_0,p}^{\pi_{\theta}}(x) \pi(u \vert x) V(x' \vert p)  \,.
\end{align}
\end{lemma}
\begin{proof}
Using comparable calculations as in Lemma~\ref{lem: performance difference transition kernel},
\begin{align*}
J(\theta,\Pr) &= \sum_{x \in \cX} \tilde{d}_{d_0,p}^{\pi}(x) \sum_{(u,x')} \pi(u \vert x) V(x' \vert p) \Pr(x' \vert x,u) 
\end{align*}
such that
\begin{align*}
\partial_{(x,u,x')} J(\theta,\Pr) = \tilde{d}_{d_0,p}^{\pi}(x) \pi(u \vert x) V(x' \vert p) \,.
\end{align*}
\end{proof}

\subsection{Smoothness derivation}
Below is a proof of the validity of the $L$-smoothness assumption under our general assumptions.
\begin{lemma}[$L$-smoothness of the deterministic policy objective]\label{lem: l smoothness}
Let $J : \Theta \to \Reals^{+}$ be the expected total cost as in \eqref{eq: policy objective} for fixed $\xi \in \Xi$, under the ODE dynamics
\[
\dot x_t = f_\xi(t,x_t,u_t), \qquad u_t = \mu_\theta(t,x_t).
\]
Assume the following:
\begin{enumerate}
    \item $f_\xi$, $r$, and $R$ are $C^2$ in their arguments on a compact invariant set, and their first derivatives are bounded and Lipschitz in $x$ and in $u$;
    \item $\mu_\theta(t,x)$ is $C^2$ in $(\theta,x)$;
    \item $\mu_\theta$, $\nabla_\theta \mu_\theta$, and $\nabla_x \mu_\theta$ are uniformly bounded;
    \item $\nabla_\theta \mu_\theta$ and $\nabla_x \mu_\theta$ are Lipschitz in $\theta$ and in $x$;
    \item the state equation, the sensitivity equation, and the adjoint equation admit unique solutions on $[0,T]$.
\end{enumerate}
Then there exists a constant $L > 0$ such that for all $\theta,\theta' \in \Theta$,
\[
\|\nabla_\theta J(\theta',\xi) - \nabla_\theta J(\theta,\xi)\|
\le L \|\theta' - \theta\|.
\]
Hence, $J(\cdot,\xi)$ is $L$-smooth.
\end{lemma}

\begin{proof}
For fixed $\xi$, recall that the deterministic adjoint-based policy gradient is
\[
\nabla_\theta J(\theta,\xi)
=
\mathbb{E}\left[
\int_0^T
\left(
\nabla_\theta \mu_\theta(t,x_t)
+
\nabla_x \mu_\theta(t,x_t)\,z_t
\right)^\top
\nabla_u H(t,x_t,u_t,p_t)\,dt
\right],
\]
where
\[
H(t,x_t,u_t,p_t)
=
r(t,x_t,u_t)+p_t^\top f_\xi(t,x_t,u_t),
\]
the state sensitivity process $z_t := \nabla_\theta x_t$ satisfies
\[
\dot z_t
=
\nabla_x f_\xi(t,x_t,u_t)\,z_t
+
\nabla_u f_\xi(t,x_t,u_t)\big(\nabla_x \mu_\theta(t,x_t)\,z_t + \nabla_\theta \mu_\theta(t,x_t)\big),
\qquad
z_0 = 0,
\]
and the costate $p_t$ satisfies
\[
-\dot p_t
=
\nabla_x r(t,x_t,u_t)
+
\nabla_x f_\xi(t,x_t,u_t)^\top p_t,
\qquad
p_T = \nabla_x R(x_T).
\]

For convenience, define
\[
a_\theta(t)
:=
\nabla_\theta \mu_\theta(t,x_t)
+
\nabla_x \mu_\theta(t,x_t)\,z_t,
\]
and
\[
h_\theta(t)
:=
\nabla_u H(t,x_t,u_t,p_t)
=
\nabla_u r(t,x_t,u_t) + \nabla_u f_\xi(t,x_t,u_t)^\top p_t.
\]
Then
\[
\nabla_\theta J(\theta,\xi)
=
\mathbb{E}\left[
\int_0^T a_\theta(t)^\top h_\theta(t)\,dt
\right].
\]

We will show that both $a_\theta(t)$ and $h_\theta(t)$ are Lipschitz in $\theta$, uniformly in $t$.

\paragraph{Uniform bounds on $z_t$ and $p_t$.}
Since the dynamics evolve on a compact invariant set and all derivatives are bounded, there exist constants
\[
C_{fx},\quad C_{fu},\quad C_{\mu \theta},\quad C_{\mu x}
\]
such that
\[
\|\nabla_x f_\xi(t,x_t,u_t)\| \le C_{fx},\qquad
\|\nabla_u f_\xi(t,x_t,u_t)\| \le C_{fu},
\]
\[
\|\nabla_\theta \mu_\theta(t,x_t)\| \le C_{\mu \theta},\qquad
\|\nabla_x \mu_\theta(t,x_t)\| \le C_{\mu x}.
\]
Hence
\[
\|\dot z_t\|
\le
\left(C_{fx}+C_{fu}C_{\mu x}\right)\|z_t\|
+
C_{fu}C_{\mu \theta}.
\]
By Gr\"onwall's inequality, there exists a constant $C_z>0$ such that
\[
\sup_{t\in[0,T]}\|z_t\| \le C_z.
\]

Next, write the costate equation as
\[
\dot p_t = -A_t p_t - b_t,\qquad p_T=\nabla_x R(x_T),
\]
where
\[
A_t := \nabla_x f_\xi(t,x_t,u_t)^\top,\qquad
b_t := \nabla_x r(t,x_t,u_t).
\]
By boundedness of $\nabla_x f_\xi$, $\nabla_x r$, and $\nabla_x R$, there exist constants $C_A$, $C_b$, $C_R$ such that
\[
\|A_t\| \le C_A,\qquad \|b_t\| \le C_b,\qquad \|p_T\| \le C_R.
\]
By variation of constants,
\[
p_t = \Phi(t,T)p_T - \int_t^T \Phi(t,s)b_s\,ds,
\]
where $\Phi(t,s)$ is the state transition matrix for $\dot y = -A_ty$. Using the standard bound
\[
\|\Phi(t,s)\| \le e^{C_A(s-t)},
\]
we get
\[
\|p_t\|
\le
C_R e^{C_A(T-t)}
+
\int_t^T e^{C_A(s-t)} C_b\,ds.
\]
Thus, there exists a constant $C_p>0$ such that
\[
\sup_{t\in[0,T]}\|p_t\| \le C_p.
\]

\paragraph{Lipschitz continuity of the state trajectory in $\theta$.}
Let $x_t$ and $x_t'$ denote the trajectories induced by $\theta$ and $\theta'$, respectively, and define
\[
\delta x_t := x_t' - x_t,\qquad \delta u_t := u_t' - u_t.
\]
Then
\[
\dot{\delta x}_t
=
f_\xi(t,x_t',u_t') - f_\xi(t,x_t,u_t).
\]
By Lipschitz continuity of $f_\xi$ in $(x,u)$,
\[
\|\dot{\delta x}_t\|
\le
L_{fx}\|\delta x_t\| + L_{fu}\|\delta u_t\|.
\]
Since
\[
u_t = \mu_\theta(t,x_t),\qquad u_t'=\mu_{\theta'}(t,x_t'),
\]
we have
\[
\|\delta u_t\|
\le
\|\mu_{\theta'}(t,x_t')-\mu_{\theta'}(t,x_t)\|
+
\|\mu_{\theta'}(t,x_t)-\mu_\theta(t,x_t)\|
\le
L_{\mu x}\|\delta x_t\| + L_{\mu\theta}\|\theta'-\theta\|.
\]
Therefore
\[
\|\dot{\delta x}_t\|
\le
\left(L_{fx}+L_{fu}L_{\mu x}\right)\|\delta x_t\|
+
L_{fu}L_{\mu\theta}\|\theta'-\theta\|.
\]
Since $\delta x_0=0$, Gr\"onwall's inequality yields
\[
\sup_{t\in[0,T]}\|x_t'-x_t\|
\le
C_x\|\theta'-\theta\|
\]
for some constant $C_x>0$.

\paragraph{Lipschitz continuity of the sensitivity process in $\theta$.}
Let $z_t$ and $z_t'$ denote the sensitivity processes induced by $\theta$ and $\theta'$, and define
\[
\delta z_t := z_t' - z_t.
\]
Using the sensitivity equation for $z_t$ and subtracting the two equations, we obtain
\[
\dot{\delta z}_t
=
\Big(
\nabla_x f_\xi(t,x_t',u_t') - \nabla_x f_\xi(t,x_t,u_t)
\Big)z_t'
+
\nabla_x f_\xi(t,x_t,u_t)\,\delta z_t
\]
\[
\qquad
+
\Big(
\nabla_u f_\xi(t,x_t',u_t') - \nabla_u f_\xi(t,x_t,u_t)
\Big)
\Big(
\nabla_x \mu_{\theta'}(t,x_t')\,z_t' + \nabla_\theta \mu_{\theta'}(t,x_t')
\Big)
\]
\[
\qquad
+
\nabla_u f_\xi(t,x_t,u_t)
\Big(
\nabla_x \mu_{\theta'}(t,x_t')\,z_t' - \nabla_x \mu_\theta(t,x_t)\,z_t
+
\nabla_\theta \mu_{\theta'}(t,x_t') - \nabla_\theta \mu_\theta(t,x_t)
\Big).
\]
Using boundedness of $z_t$, $z_t'$, boundedness of first derivatives, and Lipschitz continuity of $\nabla_x f_\xi$, $\nabla_u f_\xi$, $\nabla_x \mu_\theta$, and $\nabla_\theta \mu_\theta$, it follows that
\[
\|\dot{\delta z}_t\|
\le
L_z \|\delta z_t\|
+
K_z\big(\|x_t'-x_t\|+\|\theta'-\theta\|\big)
\]
for some constants $L_z,K_z>0$. Using Lipschtz continuity of the state and $\delta z_0=0$, Gr\"onwall's inequality yields
\[
\sup_{t\in[0,T]}\|z_t'-z_t\|
\le
C_z' \|\theta'-\theta\|
\]
for some constant $C_z'>0$.

\paragraph{Lipschitz continuity of the costate in $\theta$.}
Let
\[
\delta p_t := p_t' - p_t.
\]
Subtracting the two costate equations gives
\[
-\dot{\delta p}_t
=
\Big(
\nabla_x r(t,x_t',u_t') - \nabla_x r(t,x_t,u_t)
\Big)
+
\Big(
\nabla_x f_\xi(t,x_t',u_t')^\top p_t'
-
\nabla_x f_\xi(t,x_t,u_t)^\top p_t
\Big).
\]
Rearranging,
\[
\dot{\delta p}_t
=
-A_t \delta p_t - c_t,
\]
where $A_t=\nabla_x f_\xi(t,x_t,u_t)^\top$ and
\[
c_t
=
\Big(
\nabla_x r(t,x_t',u_t') - \nabla_x r(t,x_t,u_t)
\Big)
+
\Big(
\nabla_x f_\xi(t,x_t',u_t') - \nabla_x f_\xi(t,x_t,u_t)
\Big)^\top p_t'.
\]
By Lipschitz continuity of $\nabla_x r$ and $\nabla_x f_\xi$, and boundedness of $p_t'$, there exists $K_p>0$ such that
\[
\|c_t\|
\le
K_p\big(\|x_t'-x_t\|+\|u_t'-u_t\|\big).
\]
Using Lipschitz continuity of the state and the corresponding bound on $\delta u_t$,
\[
\|c_t\|
\le
K_p' \|\theta'-\theta\|
\]
for some $K_p'>0$. Also,
\[
\|\delta p_T\|
=
\|\nabla_x R(x_T') - \nabla_x R(x_T)\|
\le
L_{R,x}\|x_T'-x_T\|
\le
L_{R,x}C_x \|\theta'-\theta\|.
\]
Applying variation of constants and Gr\"onwall's inequality, we conclude that there exists $C_p'>0$ such that
\[
\sup_{t\in[0,T]}\|p_t'-p_t\|
\le
C_p' \|\theta'-\theta\|.
\]

\paragraph{Lipschitz continuity of $a_\theta(t)$.}
Recall
\[
a_\theta(t)
=
\nabla_\theta \mu_\theta(t,x_t)
+
\nabla_x \mu_\theta(t,x_t)\,z_t.
\]
Hence
\[
a_{\theta'}(t)-a_\theta(t)
=
\Big(
\nabla_\theta \mu_{\theta'}(t,x_t') - \nabla_\theta \mu_\theta(t,x_t)
\Big)
+
\Big(
\nabla_x \mu_{\theta'}(t,x_t')\,z_t' - \nabla_x \mu_\theta(t,x_t)\,z_t
\Big).
\]
Therefore,
\[
\|a_{\theta'}(t)-a_\theta(t)\|
\le
\|\nabla_\theta \mu_{\theta'}(t,x_t') - \nabla_\theta \mu_\theta(t,x_t)\|
+
\|\nabla_x \mu_{\theta'}(t,x_t')\,z_t' - \nabla_x \mu_\theta(t,x_t)\,z_t\|.
\]
Split the second term:
\[
\|\nabla_x \mu_{\theta'}(t,x_t')\,z_t' - \nabla_x \mu_\theta(t,x_t)\,z_t\|
\le
\|\nabla_x \mu_{\theta'}(t,x_t')\|\,\|z_t'-z_t\|
+
\|\nabla_x \mu_{\theta'}(t,x_t') - \nabla_x \mu_\theta(t,x_t)\|\,\|z_t\|.
\]
Using boundedness of $\nabla_x \mu_\theta$, Lipschitz continuity of $\nabla_x \mu_\theta$, $\nabla_\theta \mu_\theta$, the state and sensitivity process, there exists $C_a>0$ such that
\[
\sup_{t\in[0,T]}\|a_{\theta'}(t)-a_\theta(t)\|
\le
C_a \|\theta'-\theta\|.
\]
Also, by boundedness,
\[
\sup_{t\in[0,T]}\|a_\theta(t)\|
\le
C_{\mu \theta} + C_{\mu x} C_z
=: \bar C_a.
\]

\paragraph{Lipschitz continuity of $h_\theta(t) = \nabla_u H(t,x_t,u_t,p_t)$.}
Recall
\[
h_\theta(t)
=
\nabla_u r(t,x_t,u_t) + \nabla_u f_\xi(t,x_t,u_t)^\top p_t.
\]
Thus
\[
h_{\theta'}(t)-h_\theta(t)
=
\Big(
\nabla_u r(t,x_t',u_t') - \nabla_u r(t,x_t,u_t)
\Big)
+
\Big(
\nabla_u f_\xi(t,x_t',u_t')^\top p_t' - \nabla_u f_\xi(t,x_t,u_t)^\top p_t
\Big).
\]
Split the second term:
\[
\|\nabla_u f_\xi(t,x_t',u_t')^\top p_t' - \nabla_u f_\xi(t,x_t,u_t)^\top p_t\|
\le
\|\nabla_u f_\xi(t,x_t',u_t')\|\,\|p_t'-p_t\|
\]
\[
\qquad
+
\|\nabla_u f_\xi(t,x_t',u_t') - \nabla_u f_\xi(t,x_t,u_t)\|\,\|p_t\|.
\]
Using boundedness of $\nabla_u f_\xi$, Lipschitz continuity of $\nabla_u r$ and $\nabla_u f_\xi$, boundedness of $p_t$, and Lipschitz continuity of the state and costate, there exists $C_h>0$ such that
\[
\sup_{t\in[0,T]}\|h_{\theta'}(t)-h_\theta(t)\|
\le
C_h \|\theta'-\theta\|.
\]
Also,
\[
\sup_{t\in[0,T]}\|h_\theta(t)\|
\le
C_{\nabla_u r} + C_{fu} C_p
=: \bar C_h.
\]

\paragraph{Conclusion.}
We now estimate the difference of the gradients:
\[
\nabla_\theta J(\theta',\xi)-\nabla_\theta J(\theta,\xi)
=
\mathbb{E}\left[
\int_0^T
\Big(
a_{\theta'}(t)^\top h_{\theta'}(t) - a_\theta(t)^\top h_\theta(t)
\Big)\,dt
\right].
\]
Insert and subtract $a_{\theta'}(t)^\top h_\theta(t)$:
\[
a_{\theta'}^\top h_{\theta'} - a_\theta^\top h_\theta
=
a_{\theta'}^\top (h_{\theta'}-h_\theta) + (a_{\theta'}-a_\theta)^\top h_\theta.
\]
Therefore,
\[
\|\nabla_\theta J(\theta',\xi)-\nabla_\theta J(\theta,\xi)\|
\le
\mathbb{E}\left[
\int_0^T
\left(
\|a_{\theta'}(t)\|\,\|h_{\theta'}(t)-h_\theta(t)\|
+
\|a_{\theta'}(t)-a_\theta(t)\|\,\|h_\theta(t)\|
\right)\,dt
\right].
\]
Using the uniform bounds established above,
\[
\|\nabla_\theta J(\theta',\xi)-\nabla_\theta J(\theta,\xi)\|
\le
T\left(\bar C_a C_h + C_a \bar C_h\right)\|\theta'-\theta\|.
\]
Hence $J(\cdot,\xi)$ is $L$-smooth with
\[
L := T\left(\bar C_a C_h + C_a \bar C_h\right).
\]
\end{proof}

\subsection{Proof of Lemma~\ref{lem: one-step state discretisation}}
\label{app: one-step state discretisation}
Recall that the exact trajectory be denoted by \(\bar x(t) = f_\xi(t, \bar x(t), \bar u_n)\) for \(t \in [t_n, t_{n+1})\), \(\bar x (t_0) = x_0\) where \(\bar u_n \sim \pi_\theta (\cdot | t_n, \bar x (t_n))\). That is, the exact step is given by 
\begin{equation}\label{eq:exact_onestep}
    \bar x(t_{n+1}) = \bar x(t_n) + \int_{t_n}^{t_{n+1}} f_\xi(s, \bar x(s), \bar u_n) \dif s
\end{equation}

The above corresponds to a sample-and-hold continuous-time model (i.e. we sample an action and hold it constant on the time interval \([t_n, t_{n+1})\). Essentially, this is a random ODE with piecewise-constant control. The Euler scheme is denoted by
\begin{equation}\label{eq:euler_onestep}
    x_{n+1} = x_n + h f_\xi(t_n, x_n, u_n), \quad u_n \sim \pi_\theta(\cdot | t_n, x_n).
\end{equation}

We will bound the grid error \(e_n \: \bar x(t_n) - x(t_n)\) in expectation. 

\paragraph{One-step error}
The error is given by accumulating the difference between \Cref{eq:exact_onestep} and \Cref{eq:euler_onestep} according to
\begin{equation*}
    e_{n+1} = e_n + \int_{t_n}^{t_{n+1}} f_\xi(s, \bar x(s), \bar u_n) - f_\xi(t_n, x_n, u_n) \dif s \,.
\end{equation*}
Splitting the integrand gives the expression
\begin{align*}
    f_\xi(s, \bar x(s), \bar u_n) - f_\xi(t_n, x_n, u_n) &= 
\underbrace{f_\xi(s, \bar x(s), \bar u_n) - f_\xi(t_n, \bar x(t_n), \bar u_n)}_{\text{(A) local truncation}}\\
    &+ \underbrace{f_\xi(t_n, \bar x(t_n), \bar u_n) - f_\xi(t_n, x_n, \bar u_n)}_{\text{(B) state error}} \\
    &+ \underbrace{f_\xi(t_n, x_n, \bar u_n) - f_\xi(t_n, x_n, u_n)}_{\text{(C) action mismatch}} \,.
\end{align*}
\paragraph{(B) State error}
By the stated assumptions, \(\norm{\text{(B)}} \leq L_x \norm{e_n}\).
\paragraph{(C) Action mismatch}
Likewise, \(\norm{\text{(C)}} \leq L_u\norm{\bar u_n - u_n}\). By a coupling argument (using an optimal \(W_1\) coupling), we can choose a coupling and obtain
\begin{equation}
\label{eq: coupling}
    \Ex \norm{\bar u_n - u_n} \leq W_1(\pi_\theta(\cdot | t_n, \bar x(t_n)), \pi_\theta (\cdot | t_n, x_n)) \leq L_\pi \Ex \norm{e_n}.
\end{equation}
by the aforementioned assumption of the policy being Lipschitz in the state.

\paragraph{(A) Local truncation}
We have \(\norm{f_\xi(s, \bar x(s), \bar u_n) - f_\xi(t_n, \bar x(t_n), \bar u_n)} \leq L_x \norm{\bar x(s) - \bar x(t_n)} + L_t (s-t_n)\). By the boundedness of \(f_{\xi}\), we also have \(\norm{\bar x(s) - \bar x(t_n)} \leq M(s-t_n)\), implying 
\begin{equation*}
    \norm{f_\xi(s, \bar x(s), \bar u_n) - f_\xi(t_n, \bar x(t_n), \bar u_n)} \leq (L_xM + L_t) (s-t_n)
\end{equation*}
Integrating \(s \in [t_n, t_{n+1})\) yields
\begin{equation*}
    \int_{t_n}^{t_{n+1}} \text{(A)}\dif s \leq \frac{1}{2}(L_x M + L_t) h^2 \,.
\end{equation*}
We now perform recursion for the error \(e_n\) and obtain an explicit global bound. From the three bounds above, 
\begin{align}
    \norm{e_{n+1}} &\leq \norm{e_n} + \int_{t_n}^{t_{n+1}} L_x \norm{e_n} + L_u \norm{\bar u_n - u_n} \dif s + \frac{1}{2}(L_x M + L_t) h^2\\
    &\leq (1 + L_x h) \norm{e_n} + h L_u\norm{\bar u_n - u_n} + \frac{1}{2}(L_x M + L_t) h^2
\end{align}
Taking the expectation, and bounding the action mismatch as in part (C),
\begin{align}
    \Ex \norm{e_{n+1}} \leq (1 + h\underbrace{(L_x +  L_uL_\pi)}_{L_{{\rm cl}}})\Ex \norm{e_n} + \frac{1}{2}(L_x M + L_t) h^2 \,.
\end{align}
With \(e_0 = 0\), by unrolling the linear recursion with a geometric-series sum,
\begin{equation*}
    \Ex \norm{e_n} \leq \frac{1}{2}(L_xM + L_t)h^2 \sum_{k=0}^{n-1} (1-L_{\rm cl}h)^k = \frac{(1 + L_{\rm cl})^n - 1}{2 L_{\rm cl}}(L_x M + L_t)h.
\end{equation*}
At the horizon \(T = t_0 + Nh\) (where \(t_0 = 0\) w.l.o.g.),  
\begin{align}
\label{eq: state error}
    \max_{0 \leq n \leq N} \Ex\norm{\bar x(t_n) - x_n} \leq \underbrace{\frac{e^{L_{\rm cl}(T-t_0)}-1}{2 L_{\rm cl}}(L_xM + L_t)}_{C_x(T)} h = \cO(h)
\end{align}
This is the first-order global one-step Euler discretisation error under a general stochastic policy. 

\subsection{Proof of Lemma~\ref{lem: rectangularity}}
\label{app: rectangularity}
Define $D(x,u,\xi) :=  \int \vert \Pr_{\xi}(x' \vert x,u) - \Pr_{\mathbf{0}}(x' \vert x,u) \vert \dif x'$. Denoting $\cN(x \vert y)$ as the density of the normal distribution with mean $y$ and covariance $\Delta t \sigma(x,u) \sigma(x,u)^{\intercal}$, note that 
\begin{align*}
D(x,u,\xi) \leq \max_{\xi \in \Xi} D(x,u,\xi) &= \max_{\xi \in \Xi} \int_{\cX} \vert \Pr_{\xi}(x' \vert x,u) - \Pr_{\mathbf{0}}(x' \vert x,u) \vert \dif x'  \\
                                        &=  \max_{\xi \in \Xi} \int_{\cX} \vert \cN(x' \vert \Delta t b_{\xi}(x,u)) - \cN( x' \vert \mathbf{0}) \vert \dif x' \,.
\end{align*}
The two distributions' densities will intersect at at $M(x,u; \xi) = \frac{1}{2} \Delta t b_{\xi}(x,u)$. Denoting $\Phi(x\vert y)$ for the cumulative density given mean $y$, again with covariance $\Delta t \sigma(x,u) \sigma(x,u)^{\intercal}$, the above leads to an upper bound on the $L_1$ budget for $(x,u)$:
\begin{align*}
\max_{\xi \in \Xi} D(x,u,\xi) &= \max_{\xi \in \Xi} \int_{-\infty}^{M(x,u; \xi)} \left(\cN( x' \vert \mathbf{0}) - \cN(x' \vert \Delta t b_{\xi}(x,u)) \right)\dif x'\\
&\hspace{1em}+ \int_{M(x,u; \xi)}^{\infty} \left( \cN(x' \vert \Delta t b_{\xi}(x,u)) - \cN( x' \vert \mathbf{0}) \right) \dif x' \\
&=\max_{\xi \in \Xi} \Phi(M(x,u; \xi) \vert \mathbf{0}) - \Phi(M(x,u) \vert \Delta t b_{\xi}(x,u)) \\
&\hspace{1em}+ \left( (1 - \Phi(M(x,u; \xi) \vert \Delta t b_{\xi}(x,u))) - (1 - \Phi(M(x,u; \xi) \vert\mathbf{0})) \right) \\
&=\max_{\xi \in \Xi} 2 \left(\Phi(M(x,u; \xi) \vert \mathbf{0}) - \Phi(M(x,u; \xi) \vert \Delta t b_{\xi}(x,u)))\right) =: \psi(x,u) \,.
\end{align*}
Note then that $\cP_{\Xi} = \bigotimes_{(x,u) \in \cZ} \{ \Pr_{\xi} \in \bP[\cX]^{\cX \times \cU}: \xi \in \Xi)$. Since for any $(x,u) \in \cZ$, $\phi(x,u)$ is linearly independent from all other $v \in \cB$, a single parameter $\xi \in \Xi$ can be chosen to define the drifts, and thereby the $\ell_1$-maximizing probability distributions, independently for all $z \in \cZ$. Consequently, $\cP_{\Xi}$ is a well-defined rectangular uncertainty set with respect to $Z$. 

Now further let the drifts be contained in $B_{x,u}$ for all $(x,u) \in \cZ$, and let $(x,u) \in \cZ$. Note that the Gaussian likelihood of the $\mathbf{0}$ model will decrease monotonically (exponentially in particular) with increasing $\ell_2$ distance. Thus $D(x,u,\xi)$ will increase monotonically with $\norm{\Delta t b_{\xi}(x,u)}$. Due to monotonicity, we can define $\cP_{\Xi}(x,u)$ with a particular budget for the drift norm, $\rho(x,u)$, such that $\int_{\cX} \vert \cN(x' \vert \Delta t b_{\xi}(x,u)) - \cN( x' \vert \mathbf{0}) \vert \dif x' \leq \psi(x,u)$. Since $\cP_{\Xi}(x,u)$ densely covers the $L_1$ ball, it follows that $\cP_{\Xi}(x,u) = \cP = \bigotimes_{(x,u) \in \cZ} \{ \Pr \in \bP[\cX]^{\cX \times \cU}: D(\Pr(\cdot \vert x,u),\Pr_{\mathbf{0}}(\cdot \vert x,u)) \leq \psi(x,u) \}$.

\subsection{Proof of Lemma~\ref{lem: PL condition adversary}}
\label{app: PL condition adversary}
Making use of the transient performance difference lemma for transition kernels (Lemma~\ref{lem: performance difference transition kernel}) on the restricted transient visitation measure, the setting $M = \max_{\Pr,\Pr' \in \cP} \norm{\frac{d_{d_0,\Pr}^{\pi}}{d_{d_0,\Pr'}^{\pi}}}_{\infty}$, and Lemma~\ref{lem: direct adversary gradient total cost MDP}, we have that
\begin{align*}
[J(\theta,\Pr^*) - J(\theta,\Pr)]^{1/2} &= \left[ \sum_{x \in \cX} d_{d_0,\Pr^*}^{\pi}(x) \sum_{(u,x')} \pi(u \vert x) V_h(x' \vert \Pr) \left(\Pr^*(x' \vert x,u) - \Pr(x' \vert x,u) \right) \right]^{1/2}  \tag{performance difference lemma} \\
                                 &= \left[\sum_{x \in \cX} d_{d_0,\Pr}^{\pi_{\theta}}(x) \frac{d_{d_0,\Pr^*}^{\pi_{\theta}}(x)}{d_{d_0,\Pr}^{\pi_{\theta}}(x)} \sum_{u \in \cU} \pi(u \vert x)  \langle V_h(x \cdot \vert \Pr),  \Pr^*(\cdot \vert x,u) - \Pr(\cdot \vert x,u)  \rangle \right]^{1/2}  \\
                                 &\leq M^{1/2} \left[\sum_{x \in \cX} d_{d_0,\Pr}^{\pi_{\theta}}(x) \sum_{u \in \cU} \pi(u \vert x)  \langle V_h(\cdot \vert \Pr),  \Pr^*(\cdot \vert x,u) - \Pr(\cdot \vert x,u)  \rangle \right]^{1/2}  \\
                                 &\leq M^{1/2} \norm{\sum_{x \in \cX} d_{d_0,\Pr}^{\pi_{\theta}}(x) \sum_{u \in \cU} \pi(u \vert x)  \langle V_h(\cdot \vert \Pr),  \Pr^*(\cdot \vert x,u) - \Pr(\cdot \vert x,u)  \rangle}^{1/2}  \\
                                 &\leq M^{1/2} \norm{\sum_{x \in \cX} d_{d_0,\Pr}^{\pi_{\theta}}(x) \sum_{u \in \cU} \pi(u \vert x)  \langle V_h(\cdot \vert \Pr),  \Pr^*(\cdot \vert x,u)  \rangle}^{1/2}  \tag{non-negative costs} \\
\end{align*}
Due to the bounds of the uncertainty set,  $\Pr^*(x' \vert x,u) \leq (1+\nu) \Pr(x' \vert x,u)$ such that $\langle \Pr^*(\cdot \vert x,u), V_h(\cdot \vert \Pr) \rangle \leq (1+\nu) \langle  \Pr(\cdot \vert x,u), V_h(\cdot \vert \Pr) \rangle$. Continuing derivations, we obtain
\begin{align*}
[J(\theta,\Pr^*) - J(\theta,\Pr)]^{1/2} &\leq (M(1+\nu))^{1/2} \norm{\sum_{x \in \cX} d_{d_0,\Pr}^{\pi_{\theta}}(x) \sum_{u \in \cU} \pi(u \vert x)  \langle V_h(\cdot \vert \Pr),  \Pr(\cdot \vert x,u)  \rangle}^{1/2}   \\
                                &\leq \sqrt{M(1+\nu)}    \norm{\nabla_{\Pr} J(\theta,\Pr)}^{1/2} \tag{ Lemma~\ref{lem: direct adversary gradient total cost MDP}} \\
                                 &\leq \sqrt{\frac{M(1+\nu)}{ B_{\text{min}}(\Xi)}}  \norm{\nabla_{\Pr} J(\theta,\Pr)}  \\
                                 &\leq \sqrt{\frac{M(1+\nu)}{ B_{\text{min}}(\Xi)}} \beta^{-1} \norm{h_{\eta}(\Pr)}  \\
                                 &\leq \frac{1}{c} \norm{h_{\eta}(\Pr)} \,.
\end{align*}
Setting $\mu = c^2/2$ and noting that $\xi = \Pr$, the stated result follows.

\subsection{Proof of Lemma~\ref{lem: PL condition policy}}
\label{app: PL condition policy}
 The proof follows formalisms within the total MDP framework as in Lemma~\ref{lem: performance difference}, considering the MDP obtained by the discretised ODE or SDE dynamics parametrised by $\xi$. Using the notation $Q_h(x,u \vert \pi) = Q(x,u \vert \pi) \Delta_t$ and $d_{d_0,\Pr_{\xi}}^{\pi_{\theta}}$ to denote the transient visitation measure from starting distribution $d_0$ and the transition kernel $\Pr_{\xi}$ induced by the discretised dynamics. 
\paragraph{a) direct parametrisation.} 
Making use of the transient performance difference lemma (Lemma~\ref{lem: performance difference}) on the restricted transient visitation measure, the setting $M = \max_{\pi,\pi^* \in \Pi} \norm{\frac{d_{d_0,\Pr_{\xi}}^{\pi}}{d_{d_0,\Pr_{\xi}}^{\pi'}}}_{\infty}$, and \Cref{eq: direct gradient}, we have that
\begin{align*}
[J(\theta,\xi) - J(\theta^*,\xi)]^{1/2} &= \left[\sum_{x \in \cX} d_{d_0,\Pr_{\xi}}^{\pi_{\theta^*}}(x) \langle Q_h(x,\cdot \vert \pi_{\theta}),  \pi_{\theta}(\cdot \vert x) - \pi_{\theta^*}(\cdot \vert x)\rangle \right]^{1/2} \tag{performance difference lemma} \\
                                 &= \left[\sum_{x \in \cX} d_{d_0,\Pr_{\xi}}^{\pi_{\theta}}(x) \frac{d_{d_0,\Pr_{\xi}}^{\pi_{\theta^*}}(x)}{d_{d_0,\Pr_{\xi}}^{\pi_{\theta}}(x)} \langle Q_h(x,\cdot \vert \pi_{\theta}), \pi_{\theta}(\cdot \vert x) - \pi_{\theta^*}(\cdot \vert x) \rangle \right]^{1/2}  \\
                                 &\leq \left[M  \sum_{x \in \cX} d_{d_0,\Pr_{\xi}}^{\pi_{\theta}}(x) \langle Q_h(x,\cdot \vert \pi_{\theta}), \pi_{\theta}(\cdot \vert x) - \pi_{\theta^*}(\cdot \vert x)  \rangle \right]^{1/2} \\       
                                 &\leq \left[ M \sum_{x \in \cX} d_{d_0,\Pr_{\xi}}^{\pi_{\theta}}(x) \sum_{u \in \cU} Q_h(x,u \vert \pi_{\theta})  \pi_{\theta}(u \vert x) \right]^{1/2}  \tag{non-negative costs} \\
                                 &= \sqrt{M} \left(\sum_{x \in \cX} d_{d_0,\Pr_{\xi}}^{\pi_{\theta}}(x) \sum_{u \in \cU} Q_h(x,u \vert \pi_{\theta})  \pi_{\theta}(u \vert x)\right)^{-1/2}  \left(\sum_{x \in \cX} d_{d_0,\Pr_{\xi}}^{\pi_{\theta}}(x) \sum_{u \in \cU} Q_h(x,u \vert \pi_{\theta})  \pi_{\theta}(u \vert x) \right) \\
                                 &= \sqrt{M}  \norm{\nabla_{\theta} J(\theta,\xi)}_1^{-1/2}   \norm{\nabla_{\theta} J(\theta,\xi)}_1 \tag{\Cref{eq: direct gradient}} \\
                                 &\leq \sqrt{M}  \norm{\nabla_{\theta} J(\theta,\xi)}^{-1/2}  \norm{\nabla_{\theta} J(\theta,\xi)} \\
                                 &\leq \sqrt{M}   B_{\text{min}}(\Theta)^{-1/2}   \norm{\nabla_{\theta} J(\theta,\xi)} \\
                                 &\leq \frac{1}{\beta} \left(\frac{M}{B_{\text{min}}(\Theta)}\right)^{1/2} \norm{h_{\eta}(\theta)} \\
                                 &\leq \frac{1}{c} \norm{h_{\eta}(\theta)} \,.
\end{align*}
It follows that $\frac{\norm{h_{\eta}(\theta)}^2}{2} \geq \frac{c^2}{2} \left( J(\theta,\xi) - J(\theta^*,\xi) \right) = \mu \left( J(\theta,\xi) - J(\theta^*,\xi) \right)$.
\paragraph{b) indirect parametrisation.} Define $c = \min_{x \in \cX}  \sqrt{d_{d_0,\Pr_{\xi}}^{\pi_{\theta}}(x)}  \min_{x' \in \cX, a' \in \cU} \pi_{\theta}(u' \vert x') \min_{x' \in \cX, u' \in \cU} \norm{\nabla_{\theta}   \log(\pi_{\theta}(u \vert x))}_{\infty}$ and set $\mu = c^2 / 2$. Then
\begin{align*}
\left[ J(\theta,\xi) - J(\theta^*,\xi) \right]^{1/2} &= \left[\sum_{x \in \cX} d_{d_0,\Pr_{\xi}}^{\pi_{\theta^*}}(x) \langle Q_h(x,\cdot \vert \pi_{\theta}),  \pi_{\theta}(\cdot \vert x) - \pi_{\theta^*}(\cdot \vert x)\rangle \right]^{1/2} \tag{performance difference lemma} \\
                    &=  \left[\sum_{x \in \cX} d_{d_0,\Pr_{\xi}}^{\pi_{\theta}}(x) \frac{d_{d_0,\Pr_{\xi}}^{\pi_{\theta^*}}(x)}{d_{d_0,\Pr_{\xi}}^{\pi_{\theta}}(x)} \langle Q_h(x,\cdot \vert \pi_{\theta}), \pi_{\theta}(\cdot \vert x) - \pi_{\theta^*}(\cdot \vert x) \rangle \right]^{1/2}  \\ \\
                    &\leq \left[M  \sum_{x \in \cX} d_{d_0,\Pr_{\xi}}^{\pi_{\theta}}(x) \langle Q_h(x,\cdot \vert \pi_{\theta}), \pi_{\theta}(\cdot \vert x) - \pi_{\theta^*}(\cdot \vert x)  \rangle \right]^{1/2} \\
                    &\leq \left[M  \sum_{x \in \cX} d_{d_0,\Pr_{\xi}}^{\pi_{\theta}}(x) \sum_{u \in \cU} \pi_{\theta}(\cdot \vert x) Q_h(x,u \vert \pi_{\theta})  \right]^{1/2} \tag{rearranging and non-negative costs} \\
\phantom{\left[ J(\theta,\xi) - J(\theta^*,\xi) \right]^{1/2}}  
                    &= M^{1/2}  \left[\sum_{x \in \cX} d_{d_0,\Pr_{\xi}}^{\pi_{\theta}}(x)   \Ex_{u \sim \pi_{\theta}}[Q_h(x,u \vert \pi_{\theta})] \right]^{1/2}   \\
                    &\leq M^{1/2} \alpha^{-1} \left[\sum_{x \in \cX} d_{d_0,\Pr_{\xi}}^{\pi_{\theta}}(x) \norm{\Ex_{u \sim \pi_{\theta}}[Q_h(x,u)\nabla_{\theta} \log(\pi(u \vert x))} \right]^{1/2} \tag{$\norm{\Ex_{u \sim \pi_{\theta}}[Q_h(x,u)\nabla_{\theta} \log(\pi(u \vert x))]} \geq  \alpha\Ex_{u \sim \pi_{\theta}}[Q_h(x,u)]$} \\
                    &= M^{1/2} \alpha^{-1} \left[\sum_{x \in \text{supp}(d_{d_0,\Pr_{\xi}}^{\pi_{\theta}})} d_{d_0,\Pr_{\xi}}^{\pi_{\theta}}(x) \norm{\Ex_{u \sim \pi_{\theta}}[Q_h(x,u)\nabla_{\theta} \log(\pi(u \vert x))}^2 \right]^{1/2}  \\
                    &\leq M^{1/2} \alpha^{-1} \left[\frac{\max_{x'  \in  \text{supp}(d_{d_0,\Pr_{\xi}}^{\pi_{\theta}})}  d_{d_0,\Pr_{\xi}}^{\pi_{\theta}}(x')}{\min_{x''  \in  \text{supp}(d_{d_0,\Pr_{\xi}}^{\pi_{\theta}})}  d_{d_0,\Pr_{\xi}}^{\pi_{\theta}}(x'')}\norm{\sum_{x \in   \text{supp}(d_{d_0,\Pr_{\xi}}^{\pi_{\theta}})} d_{d_0,\Pr_{\xi}}^{\pi_{\theta}}(x) \Ex_{u \sim \pi_{\theta}}[Q_h(x,u)\nabla_{\theta} \log(\pi(u \vert x))]} \right]^{1/2}  \\
                    &= M^{1/2} \alpha^{-1} \left[\frac{\max_{x'  \in  \text{supp}(d_{d_0,\Pr_{\xi}}^{\pi_{\theta}})}  d_{d_0,\Pr_{\xi}}^{\pi_{\theta}}(x')}{\min_{x''  \in  \text{supp}(d_{d_0,\Pr_{\xi}}^{\pi_{\theta}})}  d_{d_0,\Pr_{\xi}}^{\pi_{\theta}}(x'')} \norm{\sum_{x \in  \cX} d_{d_0,\Pr_{\xi}}^{\pi_{\theta}}(x) \Ex_{u \sim \pi_{\theta}}[Q_h(x,u)\nabla_{\theta} \log(\pi(u \vert x))]} \right]^{1/2}  \\
                    &= M^{1/2} \alpha^{-1} \kappa^{1/2} \norm{\nabla_{\theta} J(\theta,\xi)}^{1/2} \tag{\Cref{eq: general gradient} and definition of $\kappa$} \\
                    &\leq M^{1/2} \alpha^{-1} \kappa^{1/2}  \norm{\nabla_{\theta} J(\theta,\xi)}^{1/2}  \\
                    &\leq M^{1/2} \alpha^{-1} \kappa^{1/2}  B_{\text{min}}(\Theta)^{-1/2}  \norm{\nabla_{\theta} J(\theta,\xi)}  \\
                    &\leq \frac{1}{c} \norm{\nabla_{\theta} J(\theta,\xi)} \,.
\end{align*}
It follows that $\frac{\norm{\nabla_{\theta} J(\theta,\xi)}^2}{2} \geq \frac{c^2}{2} \left[ J(\theta,\xi) - J(\theta^*,\xi) \right] = \mu \left[ J(\theta,\xi) - J(\theta^*,\xi) \right]$. For projected gradient descent, the result then follows for $c' = \beta c$, analogous to part a).

\subsection{Proof of Lemma~\ref{lem: ascent}}
\label{app: ascent}
Define the step $\Delta \xi = \eta_2(k-1,t-1) \nabla_{\xi} J(\theta_{k},\xi_{k-1}^{t-1})$. Then applying Lemma~\ref{lem: standard descent} to $-J$ as a cost function with traditional gradient descent, we have that
\begin{align*}
 - J\left(\theta_{k},\xi_{k-1}^{t-1} + \Delta \xi \right) 
&\leq - J(\theta_{k},\xi_{k-1}^{t-1}) - \langle \nabla_\theta J(\theta_{k},\xi_{k-1}^{t-1}), \Delta \xi \rangle + \frac{L}{2}\norm{\Delta \xi}^2 \\
& \leq - J(\theta_{k},\xi_{k-1}^{t-1}) - \eta_2(k-1,t-1)- (1 - \eta_2(k-1,t-1)\frac{L}{2})  \norm{\nabla_{\xi} J(\theta_{k},\xi_{k-1}^{t-1})}^2 \\
& \leq  - J(\theta_{k},\xi_{k-1}^{t-1}) - \frac{1}{2} \eta_2(k-1,t-1) \norm{\nabla_{\xi} J(\theta_{k},\xi_{k-1}^{t-1})}^2 \\
& \leq - J(\theta_{k},\xi_{k-1}^{t-1}) - \eta_2(k-1,t-1) \mu_{k,t-1}' \left(J(\theta_{k},\xi_{k-1}^{*} - J(\theta_{k},\xi_{k-1}^{t-1}) \right) \,,
\end{align*}    
where the penultimate step follows from the PL condition (Assumption~\ref{ass: PL}). Adding the maximal cost on both sides and switching the iterates, we obtain
\begin{align*}
J(\theta_{k},\xi_{k-1}^{*}) - J(\theta_{k},\xi_{k-1}^{t})& \leq (1 - \eta_2(k-1,t-1) \mu_{k,t-1}) (J(\theta_{k},\xi_{k-1}^{*}) - J(\theta_{k},\xi_{k-1}^{t-1})) \,.
\end{align*}

Now consider the case with projected gradient descent. Define the projected gradient mapping
\begin{align*}
h(k,t) = \frac{1}{\eta_2(k-1,t-1)} \left( \text{proj}_{\Xi}\left( \xi_{k-1}^{t-1} + \eta_2(k-1,t-1) \nabla_\xi J(\theta_{k},\xi_{k-1}^{t-1})   \right)  - \xi_{k-1}^{t-1} \right) \,.
\end{align*}
Based on our Modified PL requirement,
\begin{align*}
\frac{1}{2} \norm{h(k,t)}^2 \geq  \mu_{k,t-1} \left( J(\theta_{k},\xi_{k-1}^{*}) - J(\theta_{k},\xi_{k-1}^{t-1}) \right) \,.  
\end{align*}
Following analogous computations to the above-mentioned gradient ascent derivations, but now replacing $\Delta \xi = \eta_2(k-1,t-1) h(k,t)$, we obtain the desired result.

\subsection{The number of trajectories}
Since the expectation of the policy gradient is unbiased, with a sufficient number of samples it is possible to closely approximate the true gradient, as shown in the lemma below. The proof relies on $\ell_2$ norm bounds on the gradients, which is a reasonable assumption based on our general assumption. For instance, using notations from the proof of Lemma~\ref{lem: l smoothness}, the deterministic policy gradient norm is bounded by
\begin{align}
 \label{eq: gradient norm}
     \norm{\nabla_\theta J(\theta)} &\leq \int_0^T \norm{\partial_\theta a_\theta} \norm{H(t)} \dif t = T \bar{C}_{a}(C_{ru} + C_{fu}C_p) =: B_{\theta}
\end{align}
As soon as any stochasticity is introduced (through the starting state, the policy, adversary, etc.), we compute the policy gradient based on its expectation over trajectories, e.g. $\nabla_{\theta} J(\theta,\xi) = \Ex_{Y} \nabla_{\theta} J(\theta,\xi \vert Y)$, where $Y$ is some random variable. Thus the approximation error can be analysed in terms of how close the sample mean of random variables $X_1,\dots, X_n$, where $X_i = J(\theta,\xi \vert Y_i)$, is to the expectation $\nabla_{\theta} J(\theta,\xi)$.
\begin{lemma}[Number of trajectories]
\label{lem: trajectories}
Let $\eps > 0$ and let $\delta_1,\delta_2 \in (0,1)$ represent the failure probability. If we let $X = \nabla_{\theta} J(\theta,\xi \vert Y)$ denote the trajectory-wise cost gradient w.r.t. $\theta$ based on some random variable $Y$, then with probability $1 - \delta_1$, the number of trajectories to obtain an $\eps$-precise policy gradient estimate is bounded by $n = \tilde{\cO}(\frac{B_{\theta}^2}{\eps^2})$, where $B_{\theta} \geq \norm{X}$. Similarly, if we let $X = \nabla_{\xi} J(\theta,\xi \vert Y)$ denote the trajectory-wise cost gradient w.r.t. $\xi$ based on some random variable $Y$, then with probability $1 - \delta_2$, the number of trajectories to obtain an $\eps$-precise projected adversary gradient estimate is bounded by $n = \tilde{\cO}(\frac{B_{\xi}^2}{\eps^2})$, where $B_{\xi} \geq \norm{X}$. 
\end{lemma}
\begin{proof}
Sample $n$ independent trajectories and compute the policy/adversary gradient expressions as $X_1, \dots, X_n$. The arithmetic mean $\bar{X} = \frac{1}{n} \sum_{i=1}^{n} X_i$ then represents the estimate of the policy/adversary gradient. We will evaluate the four mentioned estimates, with failure probabilities $\delta_i$, $i=1,\dots,4$, and then 
\paragraph{The policy gradient.} 
Via Hoeffding's inequality with $X$ in $[-B_{\theta},B_{\theta}]$, it follows for any dimension $j \in 1,\dots,d$ that
 \begin{align*}
\bP( \vert \bar{X}^j - \Ex[X^j] \vert \geq \eps ) \leq 2 \exp\left(-\frac{2n\eps^2}{2B_{\theta}^2}\right) \,.
 \end{align*}
 Applying union bound, it follows that the failure probability is bounded by
  \begin{align*}
\sum_{j=1}^{d_{\theta}} \bP( \vert \bar{X}^j - \Ex[X^j] \vert \geq \eps ) \leq 2 d_{\theta} \exp\left(-\frac{n\eps^2}{B_{\theta}^2}\right) =: \delta_1 \,,
 \end{align*}
 such that $n = \frac{B_{\theta}^2}{\eps^2} \log\left(\frac{2d_{\theta}}{\delta_1}\right)$. 

\paragraph{The adversary.} Define the projected gradient mappings
\begin{align*}
\hat{h}_{\eta}(\xi) = \frac{1}{\eta} \left(\text{proj}_{\Xi}\left( \xi + \eta \hat{X}   \right) - \xi \right)
\end{align*}
and 
\begin{align*}
h_{\eta}(\xi) = \frac{1}{\eta} \left( \text{proj}_{\Xi}\left( \xi + \eta \Ex[X]   \right)  - \xi \right) \,.
\end{align*}
Their distance is upper bounded by 
\begin{align*}
& \norm{\hat{h}_{\eta}(\xi) - h_{\eta}(\xi)}_{\infty} \\
&= \norm{\frac{1}{\eta} \left( \text{proj}_{\Xi}\left( \xi + \eta \hat{X}   \right) - \xi \right) - \frac{1}{\eta} \left( \text{proj}_{\Xi}\left( \xi + \eta \Ex[X]   \right)  - \xi \right)}_{\infty} \\
&= \frac{1}{\eta} \norm{\text{proj}_{\Xi}\left( \xi + \eta \hat{X}   \right) - \text{proj}_{\Xi}\left( \xi + \eta \Ex[X]   \right)}_{\infty}  \\
&\leq \norm{\hat{X} - \Ex[X]}_{\infty}  \tag{$\norm{\text{proj}_{\Xi}\left(v \right) - \text{proj}_{\Xi}\left(w \right)} \leq \norm{v -w}$} \\
& \leq \eps \,.
\end{align*}

Applying Hoeffding's inequality analogously to the derivations for the policy gradient, it follows that the failure probability is bounded by
  \begin{align*}
\sum_{j=1}^{d} \bP( \vert \hat{h}_{\eta}(\xi)^j - h_{\eta}(\xi)^j \vert \geq \eps ) \leq \sum_{j=1}^{d} \bP( \vert \hat{X}^j - \Ex [X^j] \vert \geq \eps ) \leq 2 d_{\xi} \exp\left(-\frac{n\eps^2}{B_{\xi}^2}\right) =: \delta_2 \,,
 \end{align*}
 such that $n = \frac{B_{\xi}^2}{\eps^2} \log\left(\frac{2d_{\xi}}{\delta_2}\right)$.
\end{proof}

\section{Further experiment details}

\subsection{Double-loop algorithm experiments}
\label{app: double-loop details}
The experiment hyperparameters with the best performance for the different learners is shown in Table~\ref{tab: hyperparameters ODE}; these settings form the basis of Table~\ref{tab: test performance}. The AdamW optimiser \citep{Loshchilov19} is used for the stochastic policy gradient based optimisers (the discrete-time baseline and the stochastic policy gradient algorithm) as this improved the performance compared to SGD. For deterministic policy gradient based algorithms, both optimisers worked well and so the SGD was chosen. The policy learning rate was tuned separately for the different optimisers, resulting in settings of $1\mathrm{e}^{-3}$ and $3\mathrm{e}^{-4}$ for SGD and AdamW, respectively. For the adversary, an aggressive step size was found to work best, setting it to 0.5 for double-loop algorithms (which take 20 steps per macro-iteration) and to 0.1 for the adversary-only experiments (which take 100 steps). A small amount of noise was added to the inner parameters after each update to prevent getting stuck at the boundary of the uncertainty set. In preliminary experiments, adversaries converged quickly to a particular environment such that the policy only experiences a single environment. To mitigate this issue, we run the experiments with the nominal flag, so that the adversaries restart from the nominal at each macro-iteration, a technique previously used in \cite{Wang2024}. The stochastic policies do not seem to be hampered with this issue. For the Stochastic Hamiltonian algorithm, the from nominal flag does not improve performance, so the adversary is continued from the previous iteration. For the discrete-time algorithm, both flags perform poorly, and so we report the result based on the from nominal flag, which has the best robust optimisation score. 

\begin{table}[htbp!]
    \centering
        \caption{Hyperparameter settings for double-loop algorithm experiments.}
    \label{tab: hyperparameters ODE}
    \resizebox{\textwidth}{!}{
    \begin{tabular}{l|p{12cm}}
    \toprule
    \textbf{Hyperparameter} & \textbf{Setting} \\
    \midrule
    Policy learning rate ($\eta_1$)  & $1\mathrm{e}^{-3}$ for deterministic and $3\mathrm{e}^{-4}$ for stochastic \\
    Policy optimiser & SGD for deterministic and AdamW for stochastic \\
    Policy updates per macro-iteration & 4 \\
    Macro-iterations ($K$) & 100 \\
    Policy gradient norm clipping &  10.0 \\
    Number of trajectories ($n$) & 1 for deterministic and 10 for stochastic \\
    Inner learning rate ($\eta_2$) &     $5\mathrm{e}^{-1}$ for double-loop training, $\mathrm{e}^{-1}$ for adversary-only (test and training)  \\
    From nominal &  False for Stochastic Hamiltonian, True otherwise    \\
    Transition kernel updates per macro-iteration & 20 \\
    Inner parameter noise & 0.001 \\
    Inner gradient norm clipping &  1.0 \\
    \bottomrule
    \end{tabular}}
\end{table}

\subsection{MFL-SDA experiments}
\label{app: MFL-SDA details}
The hyperparameters for the mean-field optimisation experiments are given in Table~\ref{tab: hyperparameters ODE MF}. Many settings were taken from the double-loop experiments. Other settings were modified to account for the changed algorithm and increased computational expense. As no tuning was done, all algorithms are evaluated on SGD with the same $\eta_1 = 1\mathrm{e}^{-3}$ and $\eta_2 = 1\mathrm{e}^{-1}$ learning rates, where the adversary learning rate is set to less than that in double-loop experiments since no near-optimality is required at each iteration. The regularisation constant was set to $\tau = 1$, which leads to $\sqrt{2\eta_1}$ and $\sqrt{2\eta_2}$ noise factors. The number of iterations was reduced to 50.
\begin{table}[htbp!]
    \centering
        \caption{Hyperparameter settings for MFL-SDA experiments.}
    \label{tab: hyperparameters ODE MF}
    \resizebox{\textwidth}{!}{
    \begin{tabular}{l|p{12cm}}
    \toprule
    \textbf{Hyperparameter} & \textbf{Setting} \\
    \midrule
    Policy learning rate ($\eta_1$)  & $1\mathrm{e}^{-3}$ \\
    Policy optimiser & SGD  \\
    Iterations ($K$) & 50 \\
    Policy gradient norm clipping &  10.0 \\
    Number of trajectories ($n$) & 1 for deterministic and 10 for stochastic \\
    Inner learning rate ($\eta_2$) &     $\mathrm{e}^{-1}$   \\
    Inner gradient norm clipping &  1.0 \\
    Regularisation constant ($\tau$) & 1.0 \\
    Number of particles & 25 \\
    \bottomrule
    \end{tabular}}
\end{table}

\subsection{Duffing Oscillator experiments}
The parameter settings for the oscillator experiments follow that of the deterministic policy gradient algorithms of the LQR domain, and are listed below.
\begin{table}[htbp!]
    \centering
        \caption{Hyperparameter settings for the Duffing Oscillator domain.}
    \label{tab: hyperparameters oscillator}
    \resizebox{\textwidth}{!}{
    \begin{tabular}{l|p{12cm}}
    \toprule
    \textbf{Hyperparameter} & \textbf{Setting} \\
    \midrule
    Policy learning rate ($\eta_1$)  & $1\mathrm{e}^{-3}$  \\
    Policy optimiser & SGD \\
    Policy updates per macro-iteration & 4 \\
    Macro-iterations ($K$) & 100 \\
    Policy gradient norm clipping &  10.0 \\
    Inner learning rate ($\eta_2$) &     $5\mathrm{e}^{-1}$ for double-loop training, $\mathrm{e}^{-1}$ for adversary-only (test and training)  \\
    From nominal & True \\
    Transition kernel updates per macro-iteration & 20 \\
    Inner parameter noise & 0.001 \\
    Inner gradient norm clipping &  1.0 \\
    \bottomrule
    \end{tabular}}
\end{table}

\end{document}